\documentclass[12pt]{article}
\usepackage[left=2.5cm,
right=2.5cm,
top=2.3cm,
bottom=2.3cm,
headheight=20pt,
headsep=10pt,
footskip=25pt,
letterpaper]{geometry}
\usepackage[utf8]{inputenc}
\usepackage[T1]{fontenc}
\usepackage[english]{babel}
\usepackage{url}
\usepackage{amsfonts}
\usepackage{nicefrac}
\usepackage{microtype}
\usepackage{amsmath}
\usepackage{amssymb}
\usepackage{mathtools}
\usepackage{subcaption}
\usepackage{caption}
\usepackage{ragged2e}
\usepackage{stackengine}
\usepackage{etoolbox}
\usepackage{xspace}
\usepackage{xpatch}
\usepackage{cuted}
\usepackage{enumerate}
\usepackage{xstring}
\usepackage{natbib}
\usepackage{setspace}
\usepackage{makecell}
\usepackage{graphicx}
\usepackage{changepage}
\usepackage{enumitem}
\usepackage{eqparbox}
\usepackage{wrapfig}
\usepackage[hang,flushmargin,bottom,symbol]{footmisc}
\usepackage[useregional=numeric]{datetime2}
\usepackage{pifont}
\usepackage{listings}
\usepackage{environ}
\usepackage{tabularray}
\usepackage{titlesec}
\usepackage{titletoc}
\usepackage{afterpage}
\usepackage{fancyhdr}
\usepackage{trimclip}
\usepackage[colorlinks]{hyperref}
\usepackage{xurl}
\usepackage[capitalise,noabbrev,nameinlink]{cleveref}
\usepackage{lineno}
\usepackage{float}
\usepackage{booktabs}
\usepackage{multirow}
\usepackage{algorithm}
\usepackage{algorithmic}
\usepackage[table,xcdraw]{xcolor}
\usepackage{rotating}

\usepackage{amsthm}
\newtheorem{theorem}{Theorem}
\newtheorem{assumption}{Assumption}
\newtheorem{lemma}{Lemma}

\newtheorem{definition}{Definition}
\newtheorem{remark}{Remark}

\fancypagestyle{first}{\fancyfoot[R]{\small\thepage}}

\setlist[itemize]{leftmargin=1em,itemsep=0ex,topsep=0ex}
\titlespacing*{\paragraph}{0pt}{0ex plus .1ex}{1ex}
\titlespacing*{\section}{0ex}{2.3ex plus .3ex minus .0ex}{.6ex plus .3ex minus .2ex}
\titlespacing*{\subsection}{0ex}{1.5ex plus .3ex minus .5ex}{.4ex plus .2ex minus .1ex}
\titlespacing*{\subsubsection}{0ex}{1.2ex plus .3ex minus .3ex}{.3ex plus .2ex minus .2ex}

\xapptocmd\normalsize{%
\abovedisplayskip=.8em plus .2em minus .2em
\belowdisplayskip=.6em plus .1em minus .1em
\abovedisplayshortskip=.8em plus .2em minus .2em
\belowdisplayshortskip=.6em plus .1em minus .1em
}{}{}

\setcitestyle{numbers}
\renewcommand{\cite}[1]{\citep{#1}}

\creflabelformat{equation}{#2\textup{#1}#3}
\crefname{algocf}{Algorithm}{Algorithms}
\Crefname{algocf}{Algorithm}{Algorithms}

\usepackage{bibunits}
\defaultbibliographystyle{unsrtnat}
\defaultbibliography{references}

\definecolor{mydarkblue}{rgb}{0.0,0.15,0.7}
\hypersetup{%
colorlinks=true,
linkcolor=mydarkblue,
citecolor=mydarkblue,
filecolor=mydarkblue,
urlcolor=mydarkblue}

\usepackage{bm}
\newcommand{\bSigma}{\mathbf{\Sigma}}
\newcommand{\bLambda}{\mathbf{\Lambda}}
\newcommand{\btheta}{\bm{\theta}}
\newcommand{\bomega}{\bm{\omega}}

\newcommand{\bphi}{\bm{\phi}}
\newcommand{\bx}{\mathbf{x}}
\newcommand{\bb}{\mathbf{b}}
\newcommand{\bg}{\mathbf{g}}
\newcommand{\bh}{\mathbf{h}}

\newcommand{\bs}{\mathbf{s}}

\newcommand{\bz}{\mathbf{z}}
\newcommand{\bo}{\mathbf{0}}
\newcommand{\bA}{\mathbf{A}}

\newcommand{\bW}{\mathbf{W}}
\newcommand{\bH}{\mathbf{H}}

\newcommand{\bI}{\mathbf{I}}

\usepackage{tcolorbox}
\usepackage{fvextra}
\usepackage{xcolor}
\definecolor{darkerlogocolor}{RGB}{151, 133, 245}  
\newtcolorbox{apxtcolorbox}[1][]{colframe=darkerlogocolor, colback=darkerlogocolor!4!white, title=#1}

\usepackage{etoc}

\title{\vspace*{-2.5ex}\bfseries CASCADE: Case-Based Continual Adaptation for Large Language Models During Deployment}

\date{}

\author{
Siyuan Guo\textsuperscript{a,b,c,\#}\rlap{,}\enskip
Yali Du\textsuperscript{d,e,$\dagger$}\rlap{,}\enskip
Hechang Chen\textsuperscript{a,b}\rlap{,}\enskip
Yi Chang\textsuperscript{a,b,c,$\dagger$}\rlap{,}\enskip
Jun Wang\textsuperscript{f,$\dagger$}}

\begin{document}
\pagestyle{fancy}

\etocdepthtag.toc{main}

\vspace{-4ex}
\maketitle
\thispagestyle{first}

\enlargethispage{1.5\baselineskip}
{\renewcommand\thefootnote{}\footnote{
\textsuperscript{\#}This work is done during Siyuan Guo's visit to Yali  Du and Jun Wang at King's College London and UCL. \\
\textsuperscript{$\dagger$}Corresponding Authors. \\
\textsuperscript{a}School of Artificial Intelligence, Jilin University.\enskip
\textsuperscript{b}Engineering Research Center of Knowledge-Driven Human-Machine Intelligence, Jilin University.\enskip
\textsuperscript{c}International Center of Future
Science, Jilin University.\enskip
\textsuperscript{d}Department of Informatics, King's College London.\enskip
\textsuperscript{e}The Alan Turing Institute.\enskip
\textsuperscript{f}AI Centre, Department of Computer Science, UCL. \\
Email: \texttt{guosyjlu@gmail.com}; \texttt{yali.du@kcl.ac.uk}; \texttt{chenhc@jlu.edu.cn}; \texttt{yichang@jlu.edu.cn}; \texttt{jun.wang@cs.ucl.ac.uk}
}}
\addtocounter{footnote}{-1}

%\linenumbers

\vspace{-6ex}
\begin{center}\bfseries
\vspace{-.5ex}
Abstract
\vspace{-.2ex}
\end{center}
\begin{adjustwidth}{0.95cm}{0.95cm}
\begin{hyphenrules}{nohyphenation}
\ignorespaces
Large language models (LLMs) have become a central foundation of modern artificial intelligence, yet their lifecycle remains constrained by a rigid separation between training and deployment, after which learning effectively ceases. This limitation contrasts with natural intelligence, which continually adapts through interaction with its environment. In this paper, we formalise deployment-time learning (DTL) as the third stage in the LLM lifecycle that enables LLM agents to improve from experience during deployment without modifying model parameters. We present CASCADE (CASe-based Continual Adaptation during DEployment), a general and principled framework that equips LLM agents with an explicit, evolving episodic memory. CASCADE formulates experience reuse as a contextual bandit problem, enabling principled exploration-exploitation trade-offs and establishing no-regret guarantees over long-term interactions. This design allows agents to accumulate, select, and refine task-relevant cases, transforming past experience into actionable knowledge. Across 16 diverse tasks spanning medical diagnosis, legal analysis, code generation, web search, tool use, and embodied interaction, CASCADE improves macro-averaged success rate by 20.9\% over zero-shot prompting while consistently outperforming gradient-based and memory-based baselines. By reframing deployment as an adaptive learning process, this work establishes a foundation for continually improving AI systems.

\end{hyphenrules}
\end{adjustwidth}
\vspace{2ex}

\begin{bibunit}

\begin{figure}[th]
    \centering
    \includegraphics{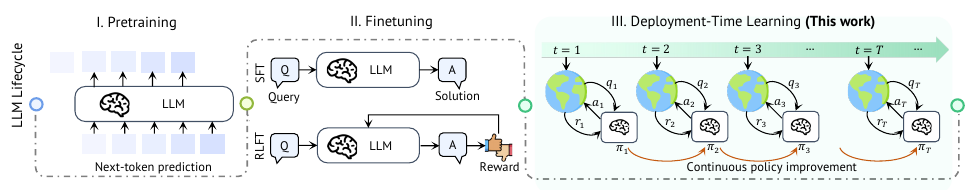}
    \caption{\textbf{The LLM Lifecycle.} In the first stage, LLMs are pre-trained with next-token prediction tasks on a large scale of corpus. Then, LLMs are further finetuned using supervised finetuning (SFT) and reinforcement learning finetuning (RLFT) for alignment and enhancing reasoning capabilities. We consider deployment-time learning as the third stage, where LLMs learn from experience during deployment, enabling continuous policy improvement over online interactions without updating the underlying LLM parameters.
    }
    \label{fig:motivation}
\end{figure}

\section{Introduction}
Large language models (LLMs) mark a transformation in artificial intelligence (AI), shifting the field from training task-specific models toward building more general-purpose AI systems. They demonstrate remarkable versatility, from accelerating scientific and algorithmic discoveries \cite{alpha-evolve} to achieving human-level data science performance in Kaggle competitions \cite{agentk}. The prevailing learning paradigm for LLMs follows a two-stage pipeline: large-scale pretraining on static corpora, followed by a finetuning phase aimed at enhancing alignment and reasoning capabilities \cite{deepseek-r1}. Despite its proven effectiveness, this paradigm suffers from a fundamental limitation: once deployed, learning essentially stops. This sharp separation between training and deployment stands in contrast to natural intelligence, where adaptation is continuous, grounded in interaction, and driven by the accumulation and selective reuse of experience \cite{natural-intelligence-1, natural-intelligence-2}. As LLMs are increasingly deployed as autonomous agents \cite{llm-agent-survey} that interact with dynamic environments and make decisions, the inability to learn from deployment-time experience emerges as a critical bottleneck, limiting adaptability, robustness, and long-term performance. Although gradient-based techniques such as reinforcement learning (RL) \cite{rl} provide principled frameworks for experiential learning \cite{era}, they require backpropagation across model parameters, incurring prohibitive cost at LLM scale. More fundamentally, many deployed LLMs are accessed as black-box application-programming-interface (API) services, making gradient-based adaptation even methodologically infeasible.

Motivated by this gap, we consider deployment-time learning (DTL) as a third, complementary stage in the LLM lifecycle (Fig.~\ref{fig:motivation}). Unlike pretraining and finetuning, DTL breaks the long-standing separation between training and testing, and enables learning during deployment by allowing LLMs to adapt from experience as they interact with the environment. Crucially, DTL shifts the locus of learning away from the foundation model itself and toward the agentic components that surround it, such as prompts, memory, tools, and decision-making mechanisms. We further formalise DTL as agentic online learning, where LLM agents observe a stream of tasks, generate solutions, receive scalar feedback indicating success or failure, and adapt their behaviour over time. This perspective shifts the objective from reducing individual errors to optimising long-term performance. By reframing deployment as an ongoing learning process, DTL transforms LLMs from static artifacts into continually improving systems.

Here, we present CASCADE (CASe-based Continual Adaptation during DEployment), a general algorithm that enables LLM agents to achieve continuous online improvement from deployment-time experience without finetuning the underlying LLM (Fig.~\ref{fig:method}a). CASCADE builds on the classic paradigm of case-based reasoning (CBR) \cite{cbr-1,cbr-2,cbr-3}, where new problems are solved by retrieving and reusing past successful solutions, allowing experience to accumulate explicitly as an episodic memory. With the LLM fixed and its response behaviour effectively stationary, adaptation during deployment hinges entirely on which past cases to retrieve. This naturally gives rise to an exploration–exploitation trade-off: agents must leverage high-utility cases while selectively exploring uncertain ones to improve future performance. CASCADE overcomes this challenge through a contextual bandit formulation \cite{bandit-book}, thereby establishing, to our knowledge, the first principled DTL algorithm for LLM agents with provable no-regret guarantees (Fig.~\ref{fig:method}b).

Through extensive experiments, we empirically demonstrate that deployment-time learning enables LLM agents to achieve continuous performance improvement from interaction experience, even when the underlying models remain fixed and are accessed as black-box APIs. Within this paradigm, we demonstrate the power of CASCADE across a diverse set of single-turn and multi-turn tasks, spanning medical diagnosis, legal analysis, operational reasoning, code generation, embodied interaction, web-based information seeking, and complex tabular reasoning on electronic health records. These improvements are observed across a wide range of model scales, from 4B models suitable for edge deployment to 32B models used in industrial applications. Together, these results establish deployment-time learning as a viable and general framework for adaptive AI systems, and position CASCADE as a principled and scalable instantiation of this framework.

\begin{figure}[t]
    \centering \includegraphics[width=\linewidth]{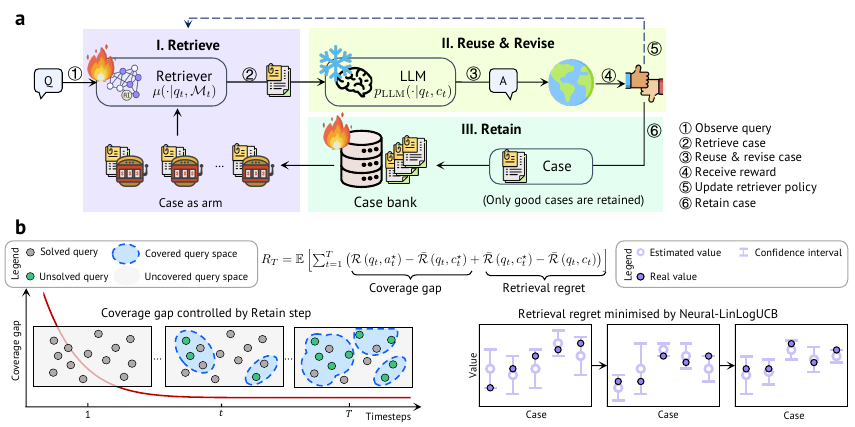}
    \caption{\textbf{Overview of CASCADE.} \textbf{a}, Given a query, CASCADE retrieves the case via the contextual bandit algorithm, reuses and revises it to generate the solution, and receives the reward. The retriever policy is updated accordingly, and successful cases are retained in the case bank. \textbf{b}, CASCADE exhibits the no-regret learning property: the coverage gap is controlled by the Retain step, while the retrieval regret is minimised by the proposed contextual bandit algorithm.
    }
    \label{fig:method}
\end{figure}

\section{Case-Based Deployment-Time Learning}
Deployment-time learning is defined by a set of constraints that fundamentally reshape how adaptation can occur. First, queries are presented as a stream, and the agent must act online without access to future tasks or outcomes. Rather than solving each query independently, the agent must extract reusable knowledge from prior interactions and apply it to new, unseen queries. Second, learning is driven by experience rather than supervision. The agent interacts with the environment, accumulates experience in textual form, and receives only scalar feedback indicating success or failure. In this work, we focus on a particularly general and practically relevant setting in which feedback is binary, reflecting the minimal signals available in many deployed systems. Third, the foundation model is fixed: once deployed, the parameters of the LLM remain unchanged. This distinguishes deployment-time learning from classical online and continual learning, particularly reinforcement learning \cite{agentic-rl}, where adaptation is typically achieved through gradient-based updates to model parameters. For LLMs, however, such updates are often impractical at deployment and impossible in black-box API settings. As a result, the locus of adaptation shifts from model parameters to agentic components operating around the fixed model.

DTL is related to, but clearly distinct from, existing test-time adaptation methods. One line of work focuses on improving performance for a single query through iterative search, reflection, or textual feedback during inference, as exemplified by Reflexion \cite{reflexion} and TextGrad \cite{textgrad}. However, they neither accumulate experience nor generalise improvements across tasks. Another line follows a conventional training–testing paradigm, optimising agentic components on a fixed training set and then deploying a static system without further adaptation, as in DSPy \cite{dspy} and GEPA \cite{gepa}. In contrast, DTL explicitly targets long-term policy improvement across a stream of tasks by learning from interaction feedback during deployment.

Under these constraints, case-based reasoning (CBR) \cite{cbr-1,cbr-2,cbr-3} provides a natural framework for DTL, where new problems are solved by retrieving relevant past cases, reusing and revising their solutions, and retaining successful new cases to the case bank. Rather than encoding knowledge implicitly in model parameters, CBR externalises experience as an explicit episodic memory that grows over time, enabling adaptation without updating the underlying model. Because adaptation is realised through memory and retrieval, this memory-centric learning mechanism empowers the agent with interpretability, flexibility, and computational efficiency, properties that are particularly well aligned with the constraints of deployment-time learning.

Within this framework, CASCADE realises deployment-time learning as a case-based continual adaptation process. For each incoming query, CASCADE first retrieves a past case from an evolving case bank based on the contextual bandit algorithm, and conditions the frozen LLM on both the current query and the retrieved case to generate a solution. The observed reward is then used to update the retrieval policy, while successful interactions are retained as new cases. In this way, the memory progressively expands to cover a broader portion of the query space, and the retrieval policy becomes increasingly effective at selecting useful prior experience. As such, adaptation arises from the cumulative growth of episodic memory and the refinement of experience selection, rather than from updates of LLM parameters.

\begin{figure}[t]
    \centering
    \includegraphics[width=\linewidth]{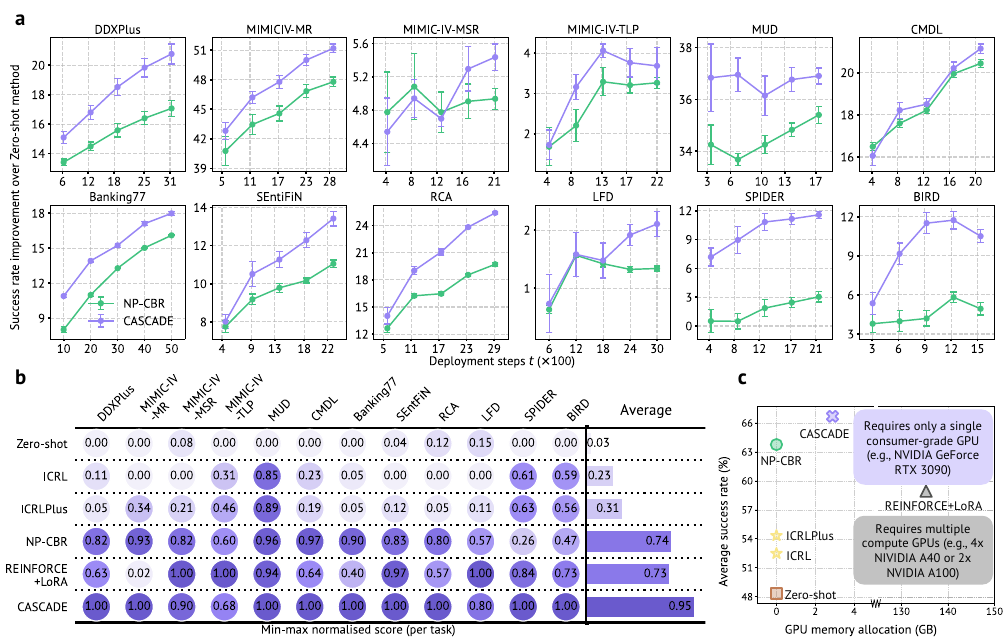}
    \caption{\textbf{Main results on 12 single-turn tasks.} All results are obtained using \texttt{Qwen3-32B} and are reported based on five different random seeds. \textbf{a}, Success rate improvement over Zero-shot method during the deployment steps across different tasks. Solid lines represent mean values and the error bars are standard deviations. \textbf{b}, Table displaying the normalised scores (0-1 range) of all the methods across different tasks, with score of 1 indicating the maximum performance per task. \textbf{c}, Performance-resource trade-off across different methods.
    Compared with the gradient-based learning method, deployment-time learning aims to achieve a better balance between resource efficiency and performance. The proposed CASCADE method outperforms REINFORCE+LoRA on 12 single-turn tasks. Moreover, CASCADE can be deployed on a single consumer-grade GPU, whereas REINFORCE+LoRA requires multiple high-end GPUs.
    }
    \label{fig:main}
\end{figure}

\section{Results}
In this section, we present the empirical results for deployment-time learning in LLM agents, where agents must improve over binary feedback from the online sequence of tasks. We mainly compare CASCADE against three learning mechanisms: (i) non-learning methods, exemplified by Zero-shot prompting method; (ii) memory-based learning methods, including ICRL\cite{icrl}, ICRLPlus\cite{icrl}, and NP-CBR, an ablation variant of CASCADE without adaptive retrieval; (iii) gradient-based learning methods, REINFORCE+LoRA \cite{cursor-online-RL,lora,lora-no-regret}, which combines on-policy RL with parameter-efficient finetuning. To evaluate the long-term performance, we utilise success rate over deployment steps as the evaluation metric, which directly reflects average regret in the online learning from binary feedback setting. Across a series of single-turn tasks, multi-turn tasks and two complex real-world tasks, CASCADE demonstrates consistent online improvement over deployment steps without updating the parameters of the underlying LLMs.

\subsection{Results on Single-Turn Tasks}
To evaluate the effectiveness of deployment-time learning in single-turn settings, we consider 12 challenging tasks spanning three representative categories: (i) decision support, including medical diagnosis, medication recommendation, legal charge prediction, penalty legal prediction, and financial query routing; (ii) operational reasoning in artificial intelligence for information technology operations (AIOps); and (iii) code generation for text-to-SQL generation. We provide detailed descriptions of each task in Supplementary Notes.

\noindent\textbf{Learning process.} We first analyse the online policy improvement during deployment by examining how success rates evolve compared to the non-learning baseline Zero-shot. Fig.~\ref{fig:main}a reports the improvement in success rate over Zero-shot for CASCADE and NP-CBR, the strongest DTL baseline. The results show that NP-CBR consistently improves the performance of \texttt{Qwen3-32B} across all benchmarks, which demonstrates the effectiveness of CBR framework for deployment-time learning. Building on this, CASCADE further improves upon NP-CBR, increasing the average success rate from 63.76\% to 66.68\% across all benchmarks. This gain highlights the importance of learning the adaptive retriever policy from the task feedback to achieve an effective trade-off between exploration and exploitation during case retrieval.

We further summarise the normalised success rates of all the baselines in Fig.~\ref{fig:main}b. Among all memory-based learning methods, CASCADE consistently achieves the best performance across all benchmarks. Notably, CASCADE outperforms the gradient-based learning method REINFORCE+LoRA on 9 out of 12 tasks and achieves comparable performance on the remaining ones. These results validate the feasibility of achieving continuous policy improvement during deployment without updating the parameters of the underlying LLM. Moreover, CASCADE can be naturally extended to retrieve and reuse multiple cases by adopting the combinatorial neural contextual bandit framework \cite{cnb}, which selects the top-$k$ cases based on upper confidence bound scores. As shown in Extended Data Fig.~\ref{fig:multicase}a, increasing the number of retrieved cases to four enables CASCADE to surpass REINFORCE+LoRA on the remaining three tasks as well. This finding underscores the potential of memory-based learning mechanisms to outperform gradient-based ones through appropriate context engineering.

In terms of resource efficiency, Fig.~\ref{fig:main}c shows that CASCADE achieves the highest average success rate while requiring less than 4~GB of GPU memory, corresponding to a single consumer-grade GPU. No existing method achieves comparable performance under an equal or smaller memory budget, placing CASCADE on the Pareto frontier of success rate and resource efficiency. In contrast, REINFORCE+LoRA requires multiple high-end GPUs during learning process, highlighting the necessity of shifting the learning locus from model parameters to agentic components.

\begin{figure}[thbp]
    \centering
    \includegraphics{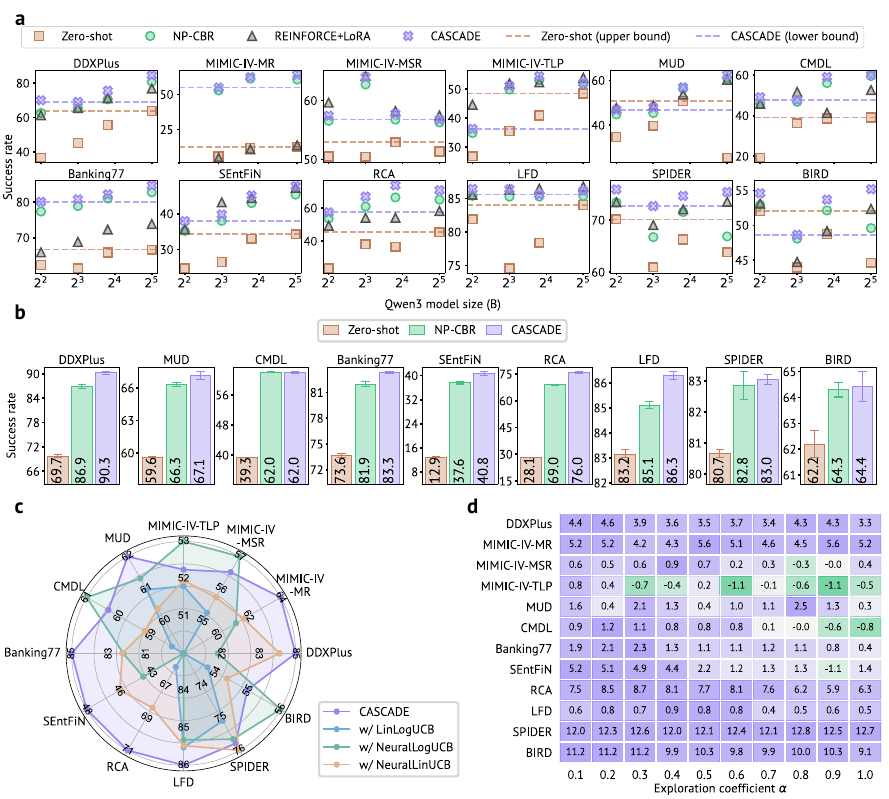}
    \caption{\textbf{In-depth analyses on 12 single-turn tasks.} All results are reported based on five different random seeds. \textbf{a}, Performance comparison with Qwen3 models of varying sizes as the underlying LLM. Dotted lines are used to highlight the upper bound of Zero-shot and the lower bound of CASCADE. \textbf{b}, Performance of three non-parametric methods with the black-box LLM \texttt{gemini-2.0-flash}. \textbf{c,d}, Radar plot of overall performance comparing CASCADE with three ablation variants (\textbf{c}); heatmap of CASCADE’s relative improvement over the best non-parametric baseline NP-CBR across varying values of the exploration coefficient $\alpha$, where higher $\alpha$ indicates stronger exploration (\textbf{d}). The results are obtained using \texttt{Qwen3-32B}.
    }
    \label{fig:ablation}
\end{figure}

\noindent\textbf{Generality across different size of LLMs.} To examine the generality of CASCADE across backbone LLMs of varying scales, we evaluate all methods on multiple model sizes from the Qwen3 series \cite{qwen3}. Specifically, we conduct experiments on the 4B, 8B, 14B, and 32B variants, where the 4B model is suited for edge-device deployment and the 32B model targets industrial scenarios. We present the results of Zero-shot, NP-CBR, REINFORCE+LoRA, and CASCADE in Fig.~\ref{fig:ablation}b. Overall, CASCADE consistently achieves the best performance in most settings, demonstrating strong generality and robustness for deployment-time learning. A notable exception occurs in the challenging medication recommendation task (MIMIC-IV-MR) with \texttt{Qwen3-4B}, where all methods fail. This observation suggests that effective deployment-time learning relies on a minimum level of foundational capability in the backbone LLM. When a zero-shot prompted LLM fails to obtain any successful interactions with the environment, online policy improvement becomes difficult to guarantee. Importantly, the lower-bound performance of CASCADE surpasses the upper-bound performance of zero-shot baselines in 9 out of 12 tasks. This result indicates that CASCADE equipped with small-scale LLMs can outperform larger-scale models, underscoring the importance of introducing deployment-time learning as a third stage in the LLM lifecycle.

\noindent\textbf{Applicability to black-box LLMs.} Beyond open-sourced LLMs, the memory-based learning mechanism also enables CASCADE to extend to LLMs accessed solely through black-box APIs. To validate this applicability, we utilise the commercial black-box LLM \texttt{gemini-2.0-flash} and compare CASCADE with both Zero-shot and the strongest DTL baseline NP-CBR across nine tasks, as shown in Fig.~\ref{fig:main}c. We exclude MIMIC-IV-MR, MIMIC-IV-MSR, and MIMIC-IV-TLP from this evaluation due to dataset licensing restrictions. Experimental results demonstrate that both NP-CBR and CASCADE consistently yield online policy improvements over the Zero-shot baseline across all evaluated datasets. Moreover, CASCADE further benefits from its adaptive retriever policy, achieving an average relative improvement of 3\% over NP-CBR. In contrast, gradient-based learning methods such as REINFORCE+LoRA are not applicable in the black-box setting, as they require gradient backpropagation through model parameters.

\noindent\textbf{Ablation study and hyper-parameter analysis.} To evaluate the effectiveness of the proposed neural contextual bandit algorithm, we conduct ablation studies and replace it with several state-of-the-art bandit baselines, including LinLogUCB \cite{linlogucb}, NeuralLogUCB \cite{neural-logucb}, and NeuralLinUCB \cite{neural-linucb}. The success rates of all ablation variants are summarised in Fig.~\ref{fig:ablation}c. The results show that the performance of different contextual bandit algorithms varies substantially across tasks, suggesting that different tasks may favour different assumptions about the underlying reward model. In contrast, the proposed Neural-LinLogUCB consistently achieves the highest or at least comparable success rates across all tasks. This demonstrates that modelling CBR as a contextual bandit problem with binary feedback, while decoupling representation learning and uncertainty estimation, provides an effective solution to regret minimisation in case retrieval.

We further analyse the impact of the exploration coefficient $\alpha$ on CASCADE’s performance. This coefficient controls the exploration strength, with larger values encouraging CASCADE to retrieve and reuse more novel cases. Fig.~\ref{fig:ablation}d reports the relative performance gains of CASCADE over the strongest non-parametric baseline, NP-CBR, as $\alpha$ varies from 0.1 to 1.0. Across this range, CASCADE consistently achieves positive improvements over NP-CBR in most settings, indicating strong robustness to the choice of $\alpha$. Notably, different tasks exhibit distinct optimal values of $\alpha$. Consequently, we recommend performing lightweight hyper-parameter tuning before deployment; alternatively, setting $\alpha$ to a small default value (e.g., 0.1) provides a reliable and effective choice.

\begin{figure}[tbhp]
    \centering
    \includegraphics[width=\linewidth]{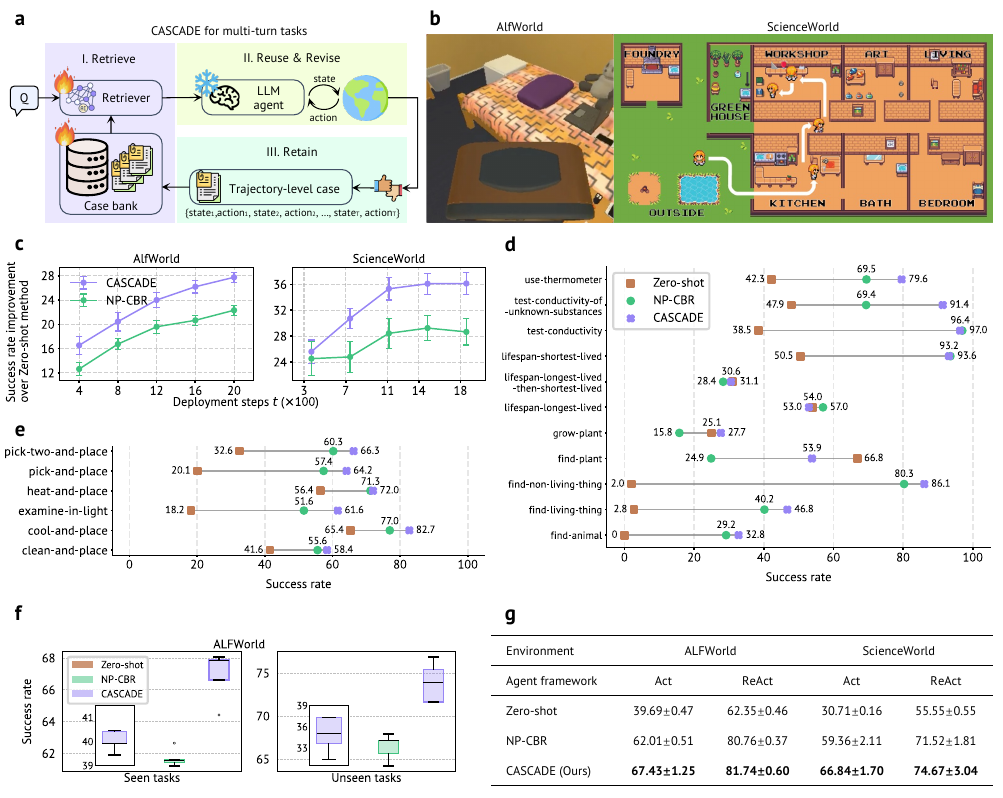}
    \caption{\textbf{Results on embodied sequential decision-making tasks.} All the results are obtained using \texttt{Qwen3-32B} and reported based on five different random seeds. \textbf{a}, CASCADE for multi-turn tasks, which retrieves, reuses, revises, and retains trajectory-level cases during learning. \textbf{b}, The visualisation of two embodied sequential decision-making tasks, ALFWorld and ScienceWorld. Both are taken from existing works \cite{alfworld, scienceworld}. \textbf{c}, Success rate improvement over Zero-shot method during the deployment steps. Solid lines represent mean values and the error bars are standard deviations. \textbf{d,e}, Topic-wise performance by task type is presented for ScienceWorld (\textbf{d}) and ALFWorld (\textbf{e}). \textbf{f}, Success rates across different task distributions in ALFWorld. The data are represented as median values (the central line of each box), 25th and 75th percentiles (the bottom and top edges of each box), minima and maxima (the whiskers attached to each box), and outliers (outside the box and whiskers). \textbf{g}, Performance comparison among three methods under two different agent frameworks.}
    \label{fig:sequential}
\end{figure}

\subsection{Results on Multi-Turn Tasks}
Beyond single-turn tasks, CASCADE naturally extends to multi-turn settings through trajectory-level case-based reasoning (Fig.~\ref{fig:sequential}a). In this subsection, we first evaluate CASCADE on two challenging embodied sequential decision-making benchmarks, ALFWorld \cite{alfworld} and ScienceWorld \cite{scienceworld}. We then present detailed case studies demonstrating the effectiveness of CASCADE in two complex real-world application scenarios: web-based deep search and tabular reasoning on electronic health records (EHR). We primarily compare CASCADE against two baselines, Zero-shot and NP-CBR, and exclude REINFORCE+LoRA from our evaluation due to its prohibitively high computational cost in multi-turn settings. Unless otherwise specified, all results are obtained using \texttt{Qwen3-32B}. 

\noindent\textbf{Embodied sequential decision-making.} To evaluate the effectiveness of CASCADE in multi-turn tasks, we conduct experiments on two challenging simulated sequential decision-making environments, ALFWorld and ScienceWorld (Fig.~\ref{fig:sequential}b). ALFWorld \cite{alfworld} is a popular decision-making benchmark where agents must navigate environments and interact with objects using natural language instructions to complete household tasks. In contrast, ScienceWorld \cite{scienceworld} is a more challenging text-based embodied benchmark, featuring a larger action space tailored to conducting elementary-level scientific experiments.

Fig.~\ref{fig:sequential}c illustrates the success rate improvement over the Zero-shot method for both CASCADE and NP-CBR across deployment steps in the two environments. The results demonstrate that both methods consistently improve performance over the Zero-shot method during deployment. Notably, CASCADE further enhances the performance of NP-CBR, increasing success rates in ALFWorld from 62.01\% to 67.43\% and in ScienceWorld from 59.36\% to 66.84\%. We further analyse topic-wise performance by task type in ALFWorld (Fig.~\ref{fig:sequential}e) and ScienceWorld (Fig.~\ref{fig:sequential}d). In ALFWorld, CASCADE consistently achieves the best results across all task categories, delivering improvements ranging from 0.7\% to 10.0\% over NP-CBR. In ScienceWorld, tasks such as \textit{find-animal}, \textit{find-living-thing}, and \textit{find-non-living-thing} are particularly challenging for backbone LLMs with Zero-shot prompting method, which achieve near-zero success rates. In contrast, both NP-CBR and CASCADE substantially improve performance, with gains exceeding 20\%, 40\%, and 80\%, respectively. These results highlight the importance of deployment-time learning, enabling LLM agents to achieve continuous policy improvement during deployment.

Since ALFWorld additionally provides 134 unseen tasks that differ from the training set, we append these tasks to the end of the online task sequence to further investigate the impact of task distribution shift on CASCADE. As shown in Fig.~\ref{fig:sequential}f, CASCADE attains success rates nearly twice those of the Zero-shot method on unseen tasks, demonstrating its ability to rapidly adapt to novel tasks via the memory-based learning mechanism. Moreover, CASCADE consistently outperforms all other methods across both seen and unseen task distributions, highlighting its strong generalisation capability in both in-distribution and out-of-distribution settings.

In practical settings, more advanced agent frameworks can be leveraged to empower LLMs with complex planning and action capabilities. We demonstrate that CASCADE can function as a plug-in component for such frameworks, owing to its simple yet effective memory-based learning mechanism. Specifically, we integrate CASCADE into the ReAct framework \cite{react}, which autonomously performs reasoning before taking actions, and evaluate its effectiveness. The overall success rates are summarised in Fig.~\ref{fig:sequential}g. ReAct equipped with CASCADE achieves success rates of 81.74\% and 74.67\% on the two tasks, outperforming the strongest baseline by 0.98\% and 3.15\%, respectively. In contrast, ReAct with Zero-shot performs substantially worse, exhibiting an approximately 20\% performance gap compared to CASCADE. These results highlight the strong potential of CASCADE to further enhance existing LLM agent frameworks.

\begin{figure}[tbhp]
    \centering
    \includegraphics[width=\linewidth]{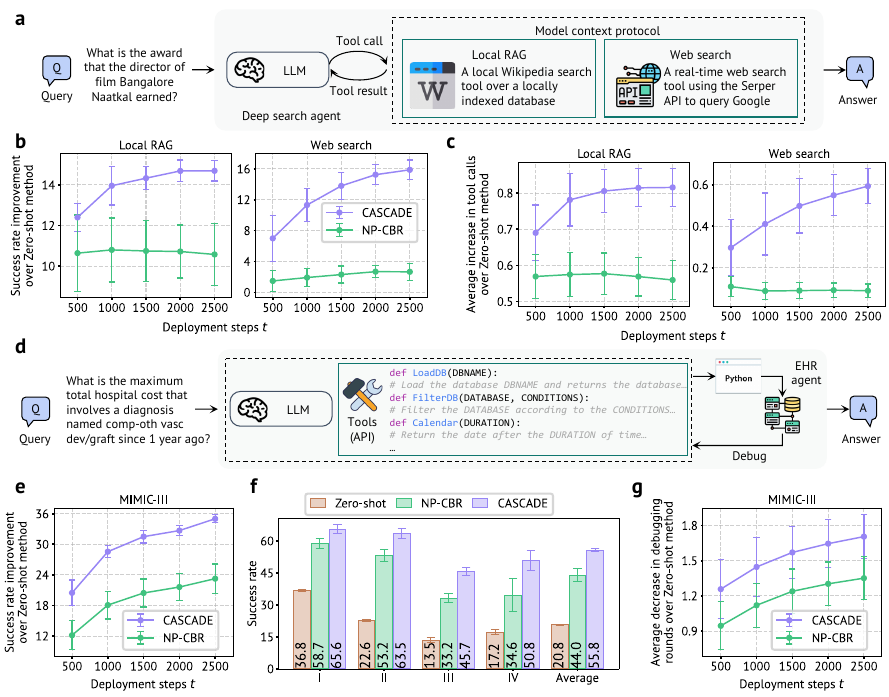}
    \caption{\textbf{Results on two real-world tasks: web-based deep search and complex tabular reasoning over electronic health records (EHR).} All results are reported based on five different random seeds, and all error bars represent standard deviations. \textbf{a}, Web-based deep search tasks that require multiple rounds of tool calling. We employ the model context protocol to provide two types of tools: local retrieval-augmented generation (RAG) and real-time web search. \textbf{b}, Success rate improvement over Zero-shot method during the deployment steps on 2Wiki under both tool types. \textbf{c}, Average increase in tool calls over Zero-shot method duering the deployment steps on 2Wiki under both tool types. \textbf{d}, Complex tabular reasoning tasks that require using appropriate APIs to query the EHR database. Real-time execution results are returned to the LLM when debugging is required. \textbf{e}, Success rate improvement over Zero-shot method during the deployment steps on MIMIC-III. \textbf{f}, Success rate across different difficulties on MIMIC-III. \textbf{g}, Average decrease in debugging rounds during the deployment step on MIMIC-III. 
    }
    \label{fig:real-world}
\end{figure}

\noindent\textbf{Case study of web-based deep search.} Deep search \cite{search-r1, deepresearcher} is a critical capability for modern web-based question answering because it empowers LLMs to retrieve relevant information from external sources, such as knowledge bases and the open web, apply multi-step reasoning, and synthesise accurate, well-grounded answers to complex user queries. Despite recent progress, LLMs still face significant challenges in this multi-turn tool-calling setting, where they must decide when to retrieve information, how to formulate retrieval queries, and how to reason over the retrieved evidence to derive reliable answers. In this subsection, we present a case study demonstrating how LLM agents can be integrated with CASCADE to enable deployment-time learning for web-based deep search tasks (Fig.~\ref{fig:real-world}a). Specifically, we conduct the experiments with two types of tools: (i) Local retrieval-augmented generation (RAG), where we deploy a search engine with 2018 Wikipedia dump \cite{wiki-2018} as the knowledge base and \texttt{E5} \cite{e5} as the retriever model; (ii) Web search, where we leverage the Serper API to perform real-time queries on Google. Both tools are implemented based on the model context protocol (MCP). We use the 2Wiki dataset \cite{2wiki} to benchmark web-based deep search, which requires multi-hop knowledge-intensive reasoning over the queries. We use \texttt{Qwen3-30B-A3B-Instruct-2507} as the backbone LLM due to its advanced tool-calling capabilities.

In 2Wiki, directly answering questions without tool calls achieves a success rate of 20.29\%. Enabling tool calls in a zero-shot setting improves performance to 35.36\% with local RAG and further to 50.63\% with web search. As shown in Fig.~\ref{fig:real-world}b, deployment-time learning consistently derives policy improvement over the Zero-shot method, while CASCADE further improves NP-CBR from 45.94\% to 50.06\% in local RAG, and from 53.29\% to 66.51\% in web search. Moreover, Fig.~\ref{fig:real-world}c shows that deployment-time learning increases the number of tool calls per query, resulting in the retrieval of more evidence relevant to the target question, which may contribute to improved performance.

\noindent\textbf{Case study of complex tabular reasoning on EHR.} Electronic health record (EHR) is a digital version of a patient’s comprehensive medical history, created and maintained by healthcare providers to document and monitor the patient’s health status over time \cite{ehr, ehragent, ehr-r1}. In routine clinical practice, clinicians typically rely on rule-based systems or support from data engineers to access and retrieve patient data from EHRs. This workflow often leads to inefficiencies and delays, which can negatively impact the quality and timeliness of patient care. In this case study, we investigate LLM agents for complex tabular reasoning on EHRs and their integration with CASCADE to enable deployment-time learning. Following EHRAgent \cite{ehragent}, we integrate multiple external tools and expose them to LLMs through API functions. Then, LLM agents generate code scripts to interact with EHR databases and iteratively refine their solutions based on real-time execution feedback, framing the task as a multi-turn code generation problem (Fig.~\ref{fig:real-world}d). We use the MIMIC-III dataset \cite{mimic-iii, ehrsql} to benchmark complex tabular reasoning on EHR. The dataset comprises real-world clinical data collected from patients at Beth Israel Deaconess Medical Center between 2001 and 2012.

In MIMIC-III, LLM agents that interact with the EHR database in a zero-shot setting achieve only a 20.75\% success rate, mainly due to difficulties in accurately invoking API functions within the allowed number of debugging iterations. In contrast, NP-CBR and CASCADE substantially improve success rates through deployment-time learning, reaching 44.02\% and 55.76\%, respectively (Fig.~\ref{fig:real-world}e). Notably, NP-CBR represents the state-of-the-art solution for EHR agents \cite{ehragent}. Fig.~\ref{fig:real-world}f summarises success rates across varying task difficulties, where CASCADE consistently outperforms all baselines. Furthermore, Fig.~\ref{fig:real-world}g shows that deployment-time learning reduces the number of debugging rounds per query, indicating that LLM agents progressively improve their ability to correctly invoke API functions during deployment.
\section{Discussion}
This work challenges the conventional view that deployment marks the end of learning in the lifecycle of LLMs. We instead argue that deployment should be regarded as a distinct and essential learning stage, where LLM agents can continue to improve through interaction with the environment and feedback. This perspective is motivated by a growing gap between how LLMs are trained, offline on static corpora, and how they are increasingly used, as autonomous agents operating in dynamic open-ended settings. Within this reframing, we introduce deployment-time learning as a general paradigm for enabling adaptation after deployment, and demonstrate its feasibility through CASCADE, a principled framework that achieves no-regret online improvement without updating model parameters.

A central insight of this work is that effective learning in deployed LLM systems does not require modifying the foundation model itself. Instead, learning can be realised by shifting the locus of adaptation to agentic components that shape behaviour, such as prompt, memory, and retrieval. This design choice is not merely a practical compromise, but a principled response to the constraints of real-world deployment, where gradient access is often unavailable, computational budgets are limited, and models are accessed as black-box APIs. Our results across single-turn tasks, multi-turn decision-making environments, and real-world case studies demonstrate that such parameter-free learning is sufficient to drive continuous policy improvement.

There are several limitations that motivate future work to further realise the potential of case-based deployment-time learning. First, while CASCADE performs robustly across a wide range of tasks, its behaviour over substantially longer horizons remains to be explored, which may involve memory growth, case management, and forgetting mechanism. Second, the current framework primarily learns from successful experiences, discarding failed attempts that may contain informative signals. Given the self-reflective capabilities of modern LLMs \cite{reflexion, memento2}, developing principled methods for learning from failure is a promising direction for enhancing deployment-time learning. Finally, as with other experiential learning methods, deployment-time learning requires a minimum level of foundational capability in the underlying LLM; when its zero-shot performance is near zero, learning from experience becomes substantially more challenging.

More broadly, as AI systems transition from static predictors to autonomous agents operating in open-ended environments, experience becomes a central driver of continual improvement. Realising experiential learning at scale requires methods that are computationally efficient and compatible with the practical constraints of deployment. By modelling case-based reasoning as a contextual bandit problem, CASCADE illustrates how these requirements can be met without updating model parameters. We believe deployment-time learning represents a complementary axis for progress in artificial intelligence, orthogonal to scale, data, and architecture, by transforming deployment from an endpoint into a continual source of learning.

\section{Methods}
\subsection{Problem Formulation}
We formalise deployment-time learning as an online learning problem with bandit feedback, where LLM agents interact with the environment over a deployment horizon of $T$ timesteps. In each timestep $t\in[T]$, the agent observes a query $q_t \in \mathcal{Q}$, and then generates a solution $a_t \in \mathcal{A}$. Here, both the query and the solution are in the natural language form, represented as a sequence of tokens from a pre-defined vocabulary $\mathcal{V}$. Finally, it receives a binary reward $r_t=\mathcal{R}(q_t,a_t)$ from the environment, indicating whether the generated solution is successful. The objective of the agent is to maximise its expected accumulated rewards over these $T$ timesteps, which is equivalent to the minimisation of the \textit{pseudo regret} (or \textit{regret} for short) as:
\begin{equation}
\label{eq:regret}
    R_T=\mathbb{E}\left[\sum_{t=1}^T \left[\mathcal{R}\left(q_t,a_t^\star\right) - \mathcal{R}\left(q_t,a_t\right)\right]\right],
\end{equation}
where $a_t^\star=\arg\max_{a\in\mathcal{A}}\mathcal{R}(q_t,a)$ refers to the optimal solution in the solution space. A good deployment-time learning agent should have sub-linear regret, and can thus achieve \textit{no-regret learning}, i.e., $\lim_{T\rightarrow\infty}R_T/T=0$. This implies that the agent’s average performance asymptotically approaches that of an oracle that always generates the optimal solution.

\subsection{CASCADE: Case-Based Deployment-Time Learning}
To instantiate deployment-time learning in LLM agents, we seek a learning mechanism that enables continual adaptation from experience without modifying the underlying LLM. A key inspiration comes from human cognition: rather than relying solely on synaptic plasticity, the brain exhibits powerful learning through episodic memory, where past problem–solution experiences are explicitly stored and reused to guide future decisions. This conceptually aligns with the principles of case-based reasoning (CBR) \cite{cbr-1,cbr-2,cbr-3}, a classic problem-solving paradigm that solves new problems by retrieving and reusing past successful solutions. Built on top of CBR, at the core of CASCADE is an episodic memory (or case bank) $\mathcal{M}_t$, where each case $c$ is defined as a tuple $(q,a,r)$. In each timestep, given a new query $q_t$, CASCADE first retrieves a case $c_t$ from the case bank via a retriever model $\mu$, i.e., $c_t \sim \mu(\cdot|q_t,\mathcal{M}_t)$, where $\mu$ can be either stochastic or deterministic. Then, it reuses and revises the retrieved case to derive the solution for the current query with the LLM, i.e., $a_t \sim p_{\text{LLM}}(\cdot|q_t,c_t)$. The environment returns a binary reward $r_t$. If the generated solution is effective (i.e., $r_t=1$), CASCADE retains the new case $(q_t,a_t,r_t)$ in the case bank, i.e., $\mathcal{M}_{t+1}=\mathcal{M}_t \cup \{(q_t,a_t,r_t)\}$; otherwise, the case bank remains unchanged, i.e., $\mathcal{M}_{t+1}=\mathcal{M}_t$. This Retrieve-Reuse-Revise-Retain process follows the principle of the classic 4R cycle \cite{cbr-3,re4,4R} in CBR. 

As such, CASCADE can be formulated as a composite policy:
\begin{equation}
    \pi_{\text{CASCADE}}(a_t|q_t,\mathcal{M}_t)=\sum_{c_t \in \mathcal{M}_t} \mu(c_t|q_t,\mathcal{M}_t) p_{\text{LLM}}(a_t|q_t,c_t).
\end{equation}
This formulation enables us to decouple the decision-making process between case retrieval and LLM generation. By fixing the parameters of the LLM, we treat its stationary response behaviour as part of the environment dynamics. Therefore, CASCADE can be optimised by merely learning the lightweight retriever policy without finetuning the LLM, and the optimisation objective is to maximise the \textit{expected utility} of the retrieved case $c$ for a given problem $q$, i.e.,
\begin{equation}
    \bar{\mathcal{R}}(q,c)=\mathbb{E}_{a \sim p_{\text{LLM}}(\cdot|q,c)} [\mathcal{R}(q,a)].
\end{equation}
Then, we can decompose the regret of CASCADE, based on the definition in Eq.~(\ref{eq:regret}), into two components:
\begin{equation}
\label{eq:regret-decomposition}
    R_T = \mathbb{E} \left[
    \sum_{t=1}^T
    \left(
    \underbrace{\mathcal{R}\left(q_t,a_t^\star\right) - \bar{\mathcal{R}}\left(q_t,c_t^\star\right)}_{\text{coverage gap }\Delta_t} +
    \underbrace{\bar{\mathcal{R}}\left(q_t,c_t^\star\right) - \bar{\mathcal{R}}\left(q_t,c_t\right)}_{\text{retrieval regret }\rho_t}
    \right)
    \right],
\end{equation}
where $c_t^\star=\arg\max_{c\in\mathcal{M}_t}\bar{\mathcal{R}}(q_t,c)$ denotes the optimal case that achieves the highest expected utility at timestep $t$. The \textit{coverage gap} $\Delta_t$ reflects the inherent limitation of the case bank, capturing the loss due to the absence of a perfectly matching case. In contrast, the \textit{retrieval regret} $\rho_t$ quantifies the error introduced by the retriever $\mu$, arising when it fails to select the most useful case available in $\mathcal{M}_t$. While the continuously growing case bank of CASCADE can eventually close the coverage gap, it still suffers from retrieval regret, as it fails to fully leverage feedback to adapt and improve the retriever policy.

As the case bank grows over time, efficient and accurate retrieval becomes critical. To address this, we adopt a two-stage retrieval pipeline widely used in large-scale information retrieval systems. In the first stage, a pretrained embedding model is utilised to recall a small set of relevant candidate cases from the full case bank. Then, a cross-encoder based reranker scores these candidates to select the best case for reuse. In our design, the embedding model is kept fixed to ensure the semantic similarity during recall, while the reranker is trained online to maximise the expected utility for the final retrieval decision.

To enable adaptive retriever learning from binary feedback, we further model the retrieval process as a contextual bandit problem, where retrieving a case is analogous to pulling an arm. At each timestep $t \in [T]$, the retriever observes a context (i.e., the query) $q_t \in \mathcal{Q}$, and recalls $K$ cases from the case bank to form the candidate pool $\mathcal{C}_t \subseteq \mathcal{M}_t$. For each candidate case $c\in\mathcal{C}_t$, we construct the contextual information as the concatenation of the query and the case, denoted by $x_{t,c}=[q_t;c]$. Then, the retriever policy adaptively selects one case $c_t \in \mathcal{C}_t$ based on this contextual information. After selection, the agent receives a stochastic binary feedback $r_t \in \{0,1\}$, which is exactly the reward of the sampled solution by prompting the LLM with the selected case. We assume that the binary feedback follows a Bernoulli distribution, where the probability of $r_t=1$ for contextual information $x_{t,c_t}$ is modelled as:
\begin{equation}
    \mathbb{P}\{r_t=1\mid x_{t,c_t}\}=\bar{\mathcal{R}}(q_t,c_t)=\sigma\left( h\left(x_{t,c_t}\right) \right),
\end{equation}
where $h:\mathcal{V}^\star\rightarrow\mathbb{R}$ is an unknown latent reward function, and $\sigma:\mathbb{R}\rightarrow [0,1]$ is the sigmoid function, i.e., $\sigma(x)=1/(1+e^{-x})$. The goal of the retriever policy is thus to maximise the expected utility of the selected case, or equivalently, minimise the retrieval regret.

While standard logistic contextual bandit algorithms \cite{linlogucb, neural-logucb, nlb} can be applied here, they rely on a fixed feature extraction model, which typically struggles to capture the complex interaction between query pairs. To this end, we follow previous works \cite{neural-lin1, neural-lin2, neural-linucb} to decouple representation learning and uncertainty estimation, and propose the Neural Linear Logistic Upper Confidence Bound (Neural-LinLogUCB) algorithm. In particular, we model the latent reward function with a deep Transformer-based encoder model \cite{transformer, bert} $f(\cdot;\bomega): \mathcal{V}^\star \rightarrow \mathbb{R}^d$ and a shallow linear head $\btheta \in \mathbb{R}^d$, i.e.,
\begin{equation}
    h(x_{t,c}) = \langle \btheta^\star, f(x_{t,c};\bomega^\star) \rangle,
\end{equation}
where $\btheta^\star$ and $\bomega^\star$ denote the unknown parameters for linear head and encoder model, respectively.

Now, we describe the learning algorithm of Neural-LinLogUCB. In timestep $t$, the retriever policy observes a set of contextual information $\mathcal{X}_t=\{x_{t,c_1},...,x_{t,c_K}\}$. Then, it selects the case that maximises the following upper confidence bound:
\begin{equation}
\label{eq:ucb}
    c_t=\arg\max_{c\in \mathcal{C}_t} \left\{
    \langle \btheta_{t-1}, f(x_{t,c};\bomega_{t-1}) \rangle
    +
    \alpha \lVert f(x_{t,c};\bomega_{t-1}) \rVert_{\bA_{t-1}^{-1}}
    \right\},
\end{equation}
where $\alpha$ is a hyper-parameter that controls the exploration, and the matrix $\bA_t$ is defined as
\begin{equation}
    \bA_t = \lambda \bI + \sum_{s=1}^t f(x_{s,c_{s}};\bomega_{t})f(x_{s,c_s};\bomega_{t})^\top,
\end{equation}
where $\lambda$ denotes the hyper-parameter of the regularisation strength in linear head estimation. After selecting the case, the retriever policy will receive a stochastic binary reward $r_t$, which is then utilised to update the parameters for deep encoder model and the linear head.

At timestep $t$, a natural way for the estimation of linear head based on history data $\{(x_{s,c_s}, r_s)\}_{s=1}^{t}$ derives from the maximum-likelihood principle. This is equivalent to finding the optimal parameter $\btheta_{t+1}$ by solving the following logistic regression problem with L2 regularisation:
\begin{equation}
    \min_{\btheta \in \mathbb{R}^d}\sum_{s=1}^{t} \left[
    -r_s \log\sigma\left(\btheta^\top f(x_{s,c_s};\bomega_{t})\right) - (1-r_s) \log\left(1-\sigma\left(\btheta^\top f(x_{s,c_s};\bomega_{t})\right)\right)
    \right] + \frac{\lambda}{2} \lVert \btheta \rVert^2_2.
\end{equation}
Considering the linearly increasing number of samples in each timestep, we perform single-step gradient descent in practice to estimate the linear head with the following update rule:
\begin{equation}
\label{eq:linear-grad}
    \btheta_{t+1} \leftarrow \btheta_{t} - \eta \left[
    \left(\sigma\left(\btheta_t^\top f(x_{t,c_{t}};\bomega_{t})\right)-r_{t}\right)f(x_{t,c_{t}};\bomega_{t}) + \lambda \btheta_{t}
    \right],
\end{equation}
where $\eta$ denotes the learning rate.

For the deep encoder network, we update its parameters once every $H$ timesteps to reduce computational overhead. At epoch $m$ (corresponding to timestep $t=mH$), we utilise the collected data during the current epoch, i.e., $\{(x_{s,c_s}, r_s)\}_{s=(m-1)H+1}^{mH}$, to perform batch gradient descent. The update is initiated from the parameters at the previous epoch, $(\bomega_{m-1}, \btheta_{m-1})$, by minimising the following loss function:
\begin{equation}
\label{eq:loss}
    \mathcal{L}_m(\bomega) = \frac{1}{H} \sum_{s=(m-1)H+1}^{mH} \left[
    -r_s \log\sigma\left(\btheta_{m-1}^\top f(x_{s,c_s};\bomega)\right) - (1-r_s) \log\left(1-\sigma\left(\btheta_{m-1}^\top f(x_{s,c_s};\bomega)\right)\right)
    \right].
\end{equation}
As such, we can achieve principled deployment-time learning without performing back-propagation on the large amount of parameters of LLMs. Instead, we only need to learn a lightweight deep encoder model and a linear head to fully exploit the binary feedback for the minimisation of the retrieval regret $\rho$, thereby achieving no-regret learning. We summarise the pseudo-code of CASCADE in Algorithm~\ref{alg:CASCADE}.

\begin{algorithm*}[tbhp]
\caption{Case-based deployment-time learning framework.}
\begin{algorithmic}[1]
\label{alg:CASCADE}
\STATE \textbf{Input:} Number of rounds $T$, regularisation parameter $\lambda$, exploration coefficient $\alpha$, update interval of deep encoder $H$, learning rate $\eta$, pretrained embedding model $\mathbf{E}$, LLM $p_{\text{LLM}}$.
\STATE \textbf{Initialisation:} Initialise the parameters of encoder network $\bomega_0$ and the linear head $\btheta_0$ from the pretrained reranker model; initialise the case bank $\mathcal{M}_1=\emptyset$.
\FOR{$t = 1,\dots,T$}
\STATE Observe the problem $q_t$ from the environment
\STATE Recall top-$K$ cases from the case bank to form the candidate pool $\mathcal{C}_t=\arg\text{top}K_{c\in\mathcal{M}_t}\langle \mathbf{E}(q_t), \mathbf{E}(c)\rangle$

    \FOR{each case $c \in \mathcal{C}_t$}
    \STATE Compute the exploitation term: $\hat{r}(q_t,c) = \langle \btheta_{t-1}, f([q_t,c];\bomega_{t-1})\rangle$
    \STATE Compute the exploration term: $\mathcal{U}(q_t,c)=\lVert f([q_t,c];\bomega_{t-1})\rVert_{\bA_{t-1}^{-1}}$
    \STATE Compute the UCB score:
        $\text{UCB}(q_t,c) = \hat{r}(q_t,c) + \alpha\mathcal{U}(q_t,c)$
    \ENDFOR

    \STATE Select the retrieved case $c_t =\arg\max_{c \in \mathcal{C}_t}\text{UCB}(q_t,c)$.
    \STATE Sample the solution $a_t\sim p_{\text{LLM}}(\cdot|q_t,c_t)$ 
    \STATE Observe the reward $r_t=\mathcal{R}(q_t,a_t)$

    \IF{$t \% H = 0$}
    \STATE Update the deep encoder following Eq.~(\ref{eq:loss}): $\bomega_t\leftarrow\bomega_{t-1} - \eta \nabla_{\bomega} \mathcal{L}_m(\bomega)$
    \STATE Load the linear head $\btheta_t$ from the trained encoder model
    \ELSE 
    \STATE Fix the deep encoder: $\bomega_t\leftarrow\bomega_{t-1}$
    \STATE Update the linear head $\btheta_t$ via Eq.~(\ref{eq:linear-grad})
    \ENDIF

    \IF{$r_{t}=1$}
    \STATE Retain the new case into the case bank, i.e., $\mathcal{M}_{t+1}=\mathcal{M}_t \cup \{(q_t,a_t,r_t)\}$
    \ELSE
    \STATE Keep the case bank unchanged, i.e., $\mathcal{M}_{t+1}=\mathcal{M}_t$
    \ENDIF
    \ENDFOR
\end{algorithmic}
\end{algorithm*}

\subsection{No-Regret Learning}
This section analyses the regret of the proposed CASCADE algorithm. Overall, we follow Eq.~(\ref{eq:regret-decomposition}) to decompose the overall regret into coverage gap and retrieval regret and bound them individually: (i) a non-parametric coverage gap that vanishes as the case bank grows, and (ii) a bandit-style retrieval regret controlled by Neural-LinLogUCB. This yields an overall sub-linear regret guarantee, establishing no-regret deployment-time learning without updating the parameters of the underlying LLM. Since the complete analysis is technically involved, we defer the full technical analysis to the Supplementary Notes for interested readers and present only the final result as below.

We first present the regret analysis for the coverage gap $R_T^\Delta=\mathbb{E}[\sum_{t=1}^T\Delta_t]$, which captures the limitation of the case bank: even with an optimal retriever, performance is constrained if no relevant past case exists. Intuitively, this gap decreases as the agent accumulates successful experiences during deployment, enabling the case bank to increasingly cover the query space. Under mild regularity conditions, namely that (i) similar queries share similar solutions, (ii) the backbone LLM maintains a non-zero probability of success, and (iii) similar queries recur with sufficient frequency, the case bank progressively covers the query space. As a result, the accumulated coverage gap grows sub-linearly with deployment time. The full derivation is presented in Supplementary Note D.

\begin{theorem}[Informal]
Under the above assumptions, with high probability, the accumulated coverage gap satisfies
\[
R_T^\Delta =
\begin{cases}
\tilde{\mathcal{O}}\!\left(T^{\frac{d_0-1}{d_0}}\right), & d_0 > 1, \\
\tilde{\mathcal{O}}(1), & 0 < d_0 \le 1,
\end{cases}
\]
where $d_0$ denotes the intrinsic dimension of the query distribution.
\end{theorem}

For the retrieval regret bound $R_T^\rho=\mathbb{E}[\sum_{t=1}^T\rho_t]$, we adopt the standard assumptions used in prior neural contextual bandit analyses \cite{neural-ucb, neural-linucb, neural-logucb, nlb}. The full derivation is presented in Supplementary Note E.

\begin{theorem}[Informal]
Under standard assumptions for neural contextual bandits, for a sufficiently wide encoder network, with high probability, the retrieval regret of CASCADE satisfies
\[
R_T^\rho = \tilde{\mathcal{O}}(\frac{1}{\kappa_\sigma}B\sqrt{T}),
\]
where $B$ measures the approximation error between the true latent reward and the learned latent reward, and $\kappa_\sigma$ is the strong monotonicity constant of the sigmoid function.
\end{theorem}

Combining the bounds on coverage gap and retrieval regret, we obtain a sub-linear bound on the total regret of CASCADE. Consequently, CASCADE achieves no-regret learning during deployment, despite the backbone LLM remaining frozen.

\begin{theorem}[Informal]
Under the above assumptions, with high probability, the overall regret of CASCADE satisfies
\[
R_T=
\begin{cases}
\tilde{\mathcal{O}}\left(\max\left\{T^{\frac{d_0-1}{d_0}},\dfrac{B}{\kappa_\sigma}\sqrt{T}
\right\}\right), & d_0 > 1, \\
\tilde{\mathcal{O}}\left(\dfrac{B}{\kappa_\sigma}\sqrt{T}\right), & 0 < d_0 \le 1.
\end{cases}
\]
\end{theorem}
\paragraph{Data availability.} We organise all the processed dataset in this paper as a new benchmark named DTLBench, and we open-source a portion of it under a mixed open-source license at \url{https://huggingface.co/datasets/guosy/DTLBench}. The remaining portion involving MIMIC-IV falls under the PhysioNet Credentialed Health Data License 1.5.0. In accordance with PhysioNet policy, dataset projects may only be considered after publication; therefore, this component will be open-sourced following publication.
\paragraph{Code availability.} Source code for CASCADE is available under an open-source license at \url{https://github.com/guosyjlu/CASCADE}.
\paragraph{Acknowledgement.} This work is supported by National Key R\&D Program of China under Grant No. 2023YFF0905400 (Yi Chang), National Natural Science Foundation of China under Grant No. 624B2059 (Siyuan Guo), U2341229 (Yi Chang), 62476110 (Hechang Chen), Key R\&D Project of Jilin Province under Grant No. 20240304200SF (Hechang Chen), China Scholarship Council under Grant No. 202406170133 (Siyuan Guo), the New Cornerstone Science Foundation through the XPLORER PRIZE (Yi Chang).

\clearpage
\begin{hyphenrules}{nohyphenation}
\setlength{\bibsep}{.5ex plus .8ex}
\putbib
\end{hyphenrules}

\end{bibunit}

\begin{bibunit}
\renewcommand{\refname}{References for Supplementary Notes}
\clearpage
\setcounter{figure}{0}
\setcounter{table}{0}
\renewcommand{\thefigure}{E\arabic{figure}}
\renewcommand{\thetable}{E\arabic{table}}

\begin{figure}[tbhp]
    \centering
    \includegraphics[width=\linewidth]{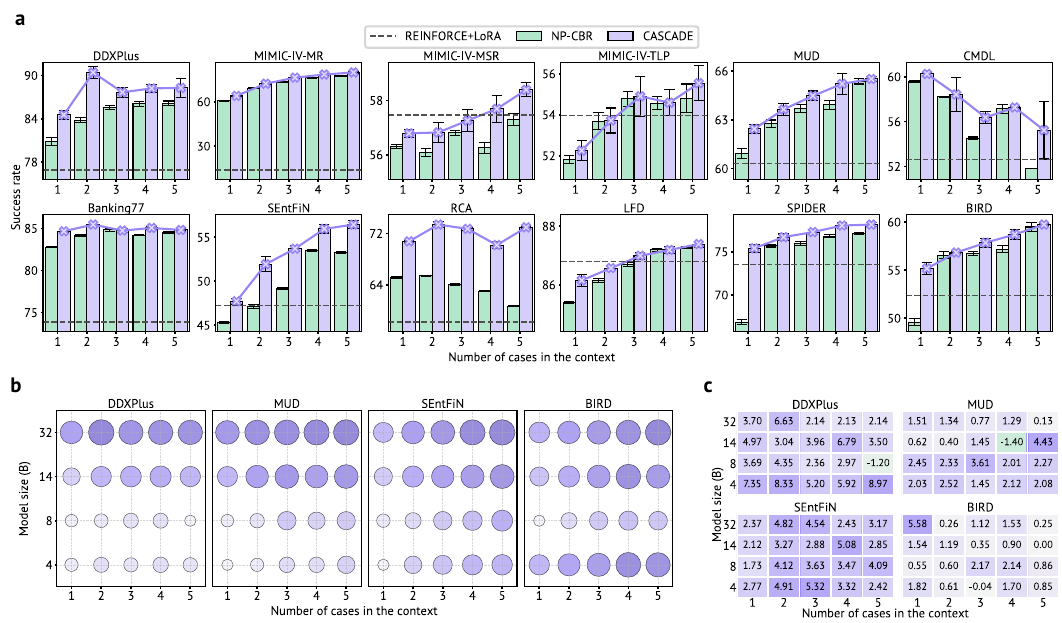}
    \caption{\textbf{Results on the extension of CASCADE to multiple cases.} All results are reported based on five different random seeds. \textbf{a}, Performance comparison under different number of cases with \texttt{Qwen3-32B} as the backbone LLM across 12 single-turn tasks. The centre represents the mean value and the error bar represents the standard deviation. In general, larger number of retrieved cases predictably increase the performance of CASCADE. \textbf{b}, CASCADE performance across model sizes and numbers of retrieved cases. Circle colour intensity and size reflect success rate, with darker and larger circles indicating better performance. \textbf{c}, Heatmap of CASCADE’s absolute improvement over the best non-parametric baseline NP-CBR, across different LLM model sizes and numbers of cases in the context. }
    \label{fig:multicase}
\end{figure}
\clearpage
\appendix

\etocdepthtag.toc{appendix}
\begingroup
  \etocsettagdepth{main}{none}
  \etocsettagdepth{appendix}{subsection}
  \renewcommand{\contentsname}{Contents of Supplementary Notes}
  \tableofcontents
\endgroup

% Reset figure, table, and equation numbering
\renewcommand{\thefigure}{S\arabic{figure}}
\renewcommand{\thetable}{S\arabic{table}}
\renewcommand{\theequation}{S\arabic{equation}}
\setcounter{figure}{0}
\setcounter{table}{0}
\setcounter{equation}{0}
\clearpage
\begin{table}[tbp]
\caption{\textbf{Summary statistics of the DTLBench.} The maximum steps refer to the maximum number of interaction steps that the environment allows per task.}
\resizebox{\linewidth}{!}{
\begin{tabular}{@{}clllcc@{}}
\toprule
\textbf{Property} & \textbf{Domain} & \textbf{Task} & \textbf{Dataset} & \textbf{\begin{tabular}[c]{@{}c@{}}Maximum\\ Steps\end{tabular}} & \textbf{\begin{tabular}[c]{@{}c@{}}Number of \\ Samples\end{tabular}} \\ \midrule
\multirow{12}{*}{Single-turn} & \multirow{4}{*}{Medical} & Medical Diagnosis & DDXPlus & 1 & 3136 \\
 &  & Medication Recommendation & MIMIC-IV-MR & 1 & 2881 \\
 &  & Medical Specialty Referral & MIMIC-IV-MSR & 1 & 2115 \\
 &  & Triage Level Prediction & MIMIC-IV-TLP & 1 & 2200 \\ \cmidrule{2-6}
 & \multirow{2}{*}{Legal} & Multi-Defendant Legal Charge Prediction & MUD & 1 & 1740 \\
 &  & Penalty Legal Prediction & CMDL & 1 & 2080 \\ \cmidrule{2-6}
 & \multirow{2}{*}{Financial} & Financial Costumer Intent Routing & Banking77 & 1 & 5000 \\
 &  & Entity-Aware Financial Sentiment Analysis & SEntFiN & 1 & 2299 \\ \cmidrule{2-6}
 & \multirow{2}{*}{AIOps} & AIOps Root Cause Analysis & RCA & 1 & 2925 \\
 &  & AIOps Log Fault Diagnosis & LFD & 1 & 3000 \\ \cmidrule{2-6}
 & \multirow{2}{*}{Coding} & Text-to-SQL & SPIDER & 1 & 2147 \\
 &  & Knowledge-Augmented Text-to-SQL & BIRD & 1 & 1534 \\ \midrule
\multirow{2}{*}{\begin{tabular}[c]{@{}c@{}}Multi-turn, \\ Simulated\end{tabular}} & \multirow{2}{*}{Embodied} & Houshold Embodied Decision Making & ALFWorld & 30 & 2000 \\
 &  & Scientific Embodied Decision Making & ScienceWorld & 10-30 & 1857 \\ \midrule
\multirow{2}{*}{\begin{tabular}[c]{@{}c@{}}Multi-turn,\\ Real-world\end{tabular}} & Information Seeking & Web-based Deep Search & 2Wiki & 5 & 2500 \\ \cmidrule{2-6}
 & Medical & Complex Tabular Reasoning on Electronic Health Records & MIMIC-III & 5 & 2500 \\ \bottomrule
\end{tabular}
}
\label{tab:benchmark-stats}
\end{table}
\section{Comparison with Existing Works}
\noindent\textbf{LLM Agents with Experiential Learning.} Recently, researchers have begun exploring experiential learning for LLM agents \cite{era}, emphasising policy improvement through experiences collected through interactions with the environment, instead of relying on human data. One line of this work is agentic reinforcement learning \cite{agentic-rl}, which leverages RL algorithms to finetune LLM agents based on a given offline task set. Among these methods, proximal policy optimisation (PPO) \cite{ppo}, group relative policy optimisation (GRPO) \cite{grpo}, and their variants have been widely adopted across diverse agentic applications \cite{gigpo, early-experience, gem}. However, these approaches rely on backpropagation across the LLM’s parameters, making them unsuitable for online settings where agents must continually adapt to a stream of tasks. Moreover, they assume that each task supports repeated rollouts for policy optimisation, whereas in our setting only a single reward observation is available per task.

Another line of work investigates experiential learning via context engineering \cite{context-engineering}, which can be broadly divided into prompt-based and memory-based approaches. For prompt-based methods, TextGrad \cite{textgrad} and Feedback Descent \cite{feedback-descent} iteratively update prompts using feedback obtained through repeated trials on a single query. In contrast, DSPy \cite{dspy} and GEPA \cite{gepa} focus on automatic prompt or system optimisation to learn from a fixed offline training set for policy learning. Closely related to our work, ICRL \cite{icrl} evolves prompts across a stream of tasks via in-context reinforcement learning with LLMs; however, its ability to support continual adaptation is constrained by its reliance on short-term memory mechanisms. For memory-based methods, such as Agent KB \cite{agentkb}, Memento \cite{memento}, and FLEX \cite{flex}, they typically maintain structured external memory to store and reuse experiences for policy improvement. However, they still follow the conventional training-testing paradigm and remain static during the deployment stage. Different from the aforementioned methods, deployment-time learning emphasises experiential learning in an online and gradient-free manner, where agents should continuously improve their policies over a stream of tasks without modifying the parameters of the underlying LLM. 

The most closely related work is the online setting introduced in DC \cite{dc} and ACE \cite{ace}, both of which develop agentic workflows to achieve incremental context adaptation for deployment-time learning. However, as noted in ACE \cite{ace}, these approaches depend on strong underlying LLMs to achieve stable improvements, potentially limiting their robustness across different model scales and deployment scenarios. Moreover, despite their early progress, they lack a principled learning mechanism, an important gap that this work aims to address. From this perspective, DC and ACE can be regarded as important precursors that highlight both the promise and the challenges of this setting, rather than as complete formulations of the broader deployment-time learning paradigm.

\noindent\textbf{Case-Based Reasoning with LLMs.} Case-based reasoning \cite{cbr-1, cbr-2, cbr-3} is a classic AI paradigm that emphasises the retrieval and adaptation of the past relevant cases while retaining new cases to the case bank. Recently, it has been integrated into LLM agents and has demonstrated strong performance across a variety of application domains \cite{cbr-llm-survey}. Conceptually, CBR is closely related to two widely adopted techniques in LLM systems: in-context learning (ICL) \cite{few-shot-learner, icl-survey} and retrieval-augmented generation (RAG) \cite{rag-1, rag-2}. ICL can be viewed as a static form of case reuse, where a small set of demonstrations is manually curated and fixed at inference time. In contrast, CBR maintains an explicit and evolving episodic memory that is continuously updated through interaction. This distinction allows CBR to accumulate experience over time, rather than relying on a fixed prompt design. Importantly, ICL establishes a natural lower bound for CBR: even when the retrieved cases are less relevant, CBR generally maintains performance similar to standard ICL and typically does not fall below zero-shot LLM performance. Compared with RAG, which conditions generation on unstructured text retrieved from static corpora, CBR retrieves, reuses and revises structured cases within a continuously evolving memory. 

Among prior works on CBR with LLMs, DS-Agent \cite{ds-agent} is a pioneering example, showcasing how empowering LLMs with CBR enables effective automated data science workflows. Re4 \cite{re4} introduces offline alignment to optimise the retrieval and reuse steps within a CBR-based automated software testing system. CBR-DDI \cite{ddi} explores the hybrid retrieval mechanism and knowledge-based prompting strategy for drug–drug interaction prediction with the CBR framework. Memento \cite{memento,memento2,memento-s} further applies CBR to deep research tasks, achieving new state-of-the-art performance on the GAIA leaderboard \cite{gaia}. By formalising case retrieval as a contextual bandit problem, CASCADE provides a principled mechanism for balancing exploration and exploitation, enabling no-regret learning and distinguishing it from prior empirical CBR works. Furthermore, different from these domain-specific methods, our work develops a general and scalable framework and demonstrates its effectiveness across diverse agentic tasks.

\noindent\textbf{Contextual Bandits for LLMs.}
Contextual bandit algorithms offer a principled framework for online decision-making by balancing exploration and exploitation conditioned on the observed context, and they have been successfully applied across a range of domains, particularly in information retrieval \cite{lin-ucb, scalable-bandit-recsys, neural-bandit-recsys}. Recently, several works have investigated the use of contextual bandits for LLMs, aiming to bridge the gap between practical applications and theoretically grounded online learning. INSTINCT \cite{instinct} formulates prompt optimisation for LLMs as a contextual bandit problem, and applies NeuralUCB \cite{neural-ucb} to find the optimal prompt with a given validation set. EASE \cite{ease} extends this line of work by leveraging contextual bandits to find an ordering-aware set of examples that improves the in-context learning performance of LLMs. PILOT \cite{llm-routing-bandit} utilises LinUCB \cite{lin-ucb} to model user preferences for LLM routing, while Poon et al. \cite{llm-selection-bandit} extend it to the multi-step LLM selection setting. In contrast to these works that rely on a fixed feature extractor, our method decouples representation learning from uncertainty estimation, enabling end-to-end optimisation of reward learning. Furthermore, we investigate a distinct setting, i.e., deployment-time learning under the CBR framework with LLMs.
\section{Benchmark Details}
\label{sup:benchmark}
In this work, we benchmark the deployment-time learning capabilities of LLM agents with 16 diverse tasks, including 12 single-turn tasks, 2 multi-turn simulated tasks and 2 multi-turn real-world tasks. For each task, we randomly shuffle all the samples into a fixed online query sequence to ensure fair and reproducible comparison, except for LFD where we order the samples based on the logging timestamp. We refer to this benchmark as DTLBench and summarise the statistics in Table \ref{tab:benchmark-stats}. We present the details of each task as below.

\subsection{Medical Diagnosis: DDXPlus}
Automated medical diagnosis systems have the potential to significantly improve healthcare accessibility, reduce diagnostic delays, and support clinicians with timely, evidence-based decision making. We utilise DDXPlus \cite{ddxplus}, a large-scale dataset for medical diagnosis, as the dataset, where the agent should predict the correct diagnosis from a set of 49 pathologies based on a given patient profile. Following StreamBench \cite{streambench}, we merge its preprocessed validation set and test set as the final dataset, consisting of 3136 samples in total. While DDXPlus is synthetic, it remains the largest public benchmark for medical diagnosis, and is widely used in research due to its high-quality data. In this work, we incorporate DDXPlus as a proof of concept to illustrate how LLM agents can benefit from deployment-time learning in medical diagnostic tasks.

\noindent\textbf{Reward Function.} If the prediction is correct, a reward of 1 is assigned; otherwise, 0.

\begin{apxtcolorbox}[Example Query for DDXPlus]
\footnotesize
Sex: Male, Age: 53

- I have recently had a viral infection.

- I have had a pericarditis.

- I have pain somewhere related to my reason for consulting.

- I am experiencing shortness of breath or difficulty breathing in a significant way.

- My symptoms are worse when lying down and alleviated while sitting up.

- On a scale of 0-10, the pain intensity is 7

- On a scale of 0-10, the pain's location precision is 4

- On a scale of 0-10, the pace at which the pain appear is 6

- The pain is:

* a knife stroke

* sharp

- The pain locations are:

* breast(R)

* breast(L)

- The pain radiates to these locations:

* thoracic spine

* posterior chest wall(L)
\end{apxtcolorbox}

\subsection{Medication Recommendation: MIMIC-IV-MR}
Medication recommendation systems can assist clinicians in selecting appropriate treatments, reducing prescription errors, and improving patient safety. MIMIC-IV \cite{mimiciv} is a large-scale real-world clinical database sourced from the electronic health record of the Beth Israel Deaconess Medical Center. We follow the script provided by LEADER \cite{leader} to extract the single-visit medication recommendation dataset, which consists of 2881 samples. For each sample, the agent should recommend a set of medications based on the provided diagnosis and procedure. There are 122 available mediations in total.

\noindent\textbf{Reward Function.} As this is a recommendation task, we compare the recommended mediations with the ground-truth prescription. We follow the original work \cite{leader} to utilise the Jaccard similarity score as the evaluation metric. Given the recommended mediation set $y$ and the ground-truth mediation set $\hat{y}$, the Jaccard similarity score is calculated as: 
\begin{equation}
\nonumber
    \text{Jaccard}(y,\hat{y})=\frac{|y \cap\hat{y}|}{|y \cup \hat{y}|}.
\end{equation}
If the Jaccard similarity score is greater than 0.2, a reward of 1 is assigned; otherwise, 0. Note that we reveal the ground-truth mediation set after the answer is submitted to simulate the real-world scenarios for this task.

Below is an example query from MIMIC-IV-MR. As the dataset is distributed under a restricted license, detailed patient information has been omitted.

\begin{apxtcolorbox}[Example Query for MIMIC-IV-MR]
\footnotesize
In this visit, he has diagnosis: <DIAGNOSIS>

procedures: <PROCEDURES>
\end{apxtcolorbox}

\subsection{Medical Specialty Referral: MIMIC-IV-MSR}
Medical specialty referral systems help guide patients to the most appropriate clinical department. When referrals are incorrect, they can cause delayed diagnoses, unnecessary transfers, and extended patient discomfort. We follow MedLLMBench \cite{medllmbench} to extract medical specialty referral tasks from the large-scale real-world clinical database MIMIC-IV \cite{mimiciv}. After preprocessing, we obtain a total of 2115 samples, where the agent is required to select the most appropriate medical specialty from 19 possible options based on the given patient profile.

\noindent\textbf{Reward Function.} If the referral prediction exactly matches the ground truth, a reward of 1 is assigned; otherwise, 0.

Below is an example query from MIMIC-IV-MSR. As the dataset is distributed under a restricted license, detailed patient information has been omitted.

\begin{apxtcolorbox}[Example Query for MIMIC-IV-MSR]
\footnotesize
Gender: <GENDER>, Race: <RACE>, Age: <AGE>. <SYMPTOM>
\end{apxtcolorbox}

\subsection{Triage Level Prediction: MIMIC-IV-TLP}
Accurate triage decisions are critical for prioritising limited emergency care resources and ensuring that patients in life-threatening conditions receive timely attention. We follow MedLLMBench \cite{medllmbench} to extract triage level prediction tasks from the large-scale real-world clinical database MIMIC-IV \cite{mimiciv}. After preprocessing, we obtain a total of 2200 samples, where the agent is required to predict the triage level defined by the Emergency Severity Index (ESI) \cite{esi} based on the given patient profile.

\noindent\textbf{Reward Function.} If the triage level prediction exactly matches the ground truth, a reward of 1 is assigned; otherwise, 0.

Below is an example query from MIMIC-IV-TLP. As the dataset is distributed under a restricted license, detailed patient information has been omitted.

\begin{apxtcolorbox}[Example Query for MIMIC-IV-TLP]
\footnotesize
Gender: <GENDER>, Race: <RACE>, Age: <AGE>. <SYMPTOM>
\end{apxtcolorbox}

\subsection{Multi-Defendant Legal Charge Prediction: MUD}
Legal charge prediction systems can improve judicial efficiency by assisting prosecutors and legal professionals in identifying appropriate charges based on case facts. MUD \cite{mud} is a real-world multi-defendant legal charge prediction dataset sourced from the Chinese government website China Judgment Online (CJO)\footnote{\label{fn:cjo}\url{https://wenshu.court.gov.cn/}}. We utilise the training split of the original dataset as the online dataset, consisting of 1740 samples and 22 charges. For each sample, the agent is provided with the fact description, and then required to recommend charges for each defendant. To make our dataset more accessible to non-Chinese speakers, we translate the whole dataset into English by prompting \texttt{gpt-4o}.

\noindent\textbf{Reward Function.} A reward of 1 is assigned if all defendants in the fact description are recommended with correct charges; otherwise, the reward is 0.

\begin{apxtcolorbox}[Example Query for MUD]
\footnotesize
Upon trial, it was found that from February to September 2017, the defendant <Defendant A>, at locations including No. 64 Qishan West Road, Huanggu District, Shenyang City, fabricated his ability to arrange jobs to victims Gao Mou1 and Gao Mou2. After gaining the trust of victims Gao Mou1 and Gao Mou2, <Defendant A> instructed the defendant <Defendant B> to impersonate Human Resources Section Chief Chi Mou twice within the Shen Mou Courtyard. <Defendant B> handed over job transfer procedures to the victims. The victims Gao Mou1 and Gao Mou2 paid a total of RMB 25,000 to <Defendant A> via bank transfer as service fees. The illicit money has been squandered.
\end{apxtcolorbox}

\subsection{Penalty Legal Prediction: CMDL}
Penalty recommendation systems can support judges and legal practitioners by providing consistent, data-driven insights during sentencing. CMDL \cite{cmdl} is a large-scale real-world multi-defendant legal judgement prediction dataset sourced from the Chinese government website China Judgment Online (CJO). We employ the test split of CMDL-big as the seed dataset and perform basic cleaning and filtering to construct the final dataset, which consists of 2080 samples. For each sample, the agent is provided with a fact description and tasked with recommending appropriate penalties for each defendant. The possible penalty types include Surveillance, Detention, Imprisonment, Death Penalty, Life Imprisonment, and Fine. Notably, there may be multiple penalties for each defendant, which increases the complexity of the task. To make our dataset more accessible to non-Chinese speakers, we translate the whole dataset into English by prompting \texttt{gpt-4o}.

\noindent\textbf{Reward Function.} A reward of 1 is assigned if all defendants in the fact description are recommended with correct penalties; otherwise, the reward is 0.

\begin{apxtcolorbox}[Example Query for CMDL]
\footnotesize
Upon trial, it was found:

1. On the night of July 8, 2013, the defendants Chen, Liu, and Chu conspired in advance to snatch a gold necklace in the room of the Renhe Hotel at Chicheng Street, Tiantai County. After coming to an agreement, the three took a moped driven by Chu to Shiliang Bay Park near the bridge on Jiuchang Road, Chicheng Street, Tiantai County. When the victim, Ding, got on an electric bike, Chen snatched a gold necklace from around Ding's neck. After succeeding, they fled the scene on Chu's moped. Subsequently, Chu and Liu sold the stolen gold necklace to Seagull Goldsmiths for 3,600 yuan. The appraisal determined that the stolen gold necklace was valued at 5,450 yuan. The gold from the necklace has been returned to the victim after being remelted.

2. In the afternoon of July 9, 2013, defendants Chen and Liu agreed to snatch another gold necklace. Liu drove the moped carrying Chen to a road outside Tiantai Vocational and Technical Secondary School on Tiantaishan Middle Road in Tiantai County, approached Yang, who was riding an electric bike, and Chen snatched the gold necklace from around Yang’s neck. The appraisal determined that the stolen gold necklace was valued at 4,860 yuan, and the necklace has been returned to the victim. After the incident, defendant Chu assisted the police in capturing one suspect.

In summary, defendants Chen and Liu committed two instances of snatching, gaining property valued at a total of 10,310 yuan. Defendant Chu participated in one instance, gaining property valued at 5,450 yuan. The aforementioned facts were not disputed by defendants Chen, Liu, and Chu during the court hearing. Their confessions are consistent with the statements of the victims, Ding and Yang, and the testimony of witness Ying. The evidence includes the Appraisal Conclusion Document, Identification Record, Scene Investigation Record, Appraisal Opinion, Criminal Judgment, Video Material, Seizure and Return Lists, Incident-related Mobile Inventory, Proof, Tiantai Seagull Goldsmiths Register, Household Registration Certificate, Account of Arrest, and Description of Previous Offenses, all of which are sufficient to establish the facts.
\end{apxtcolorbox}

\subsection{Financial Costumer Intent Routing: Banking77}
Accurate intent routing plays a crucial role in modern financial customer service systems, where users expect fast and precise support for a wide range of banking needs. Banking77 \cite{banking77} is a financial user intent routing dataset with 77 fine-grained intents. Given the large size of the original dataset, full-scale evaluation is computationally expensive. Therefore, we randomly downsample 5000 samples to construct our final dataset. For each sample, the agent is provided with a customer service query and tasked with routing it to the corresponding intent category.

\noindent\textbf{Reward Function.} If the prediction is correct, a reward of 1 is assigned; otherwise, 0.

\begin{apxtcolorbox}[Example Query for Banking77]
\footnotesize
I have checked the account information several times to be sure that it is correct, but the in country transfer I did a few days ago still has not appeared! What is the hold up?
\end{apxtcolorbox}

\subsection{Entity-Aware Financial Sentiment Analysis: SEntFiN}
Financial news often conveys subtle sentiment targeted at specific entities such as companies, commodities, or markets. Accurately identifying sentiment at an entity level can support analysts, traders, and automated financial systems in understanding market signals more precisely. SEntFiN \cite{sentifin} is an entity-aware sentiment analysis dataset for financial news. We utilise the training split of the original dataset to construct the final dataset. To increase task difficulty, we retain only samples containing multiple entities, resulting in 2299 instances. For each sample, the agent is presented with a news headline and required to identify all mentioned entities along with their associated sentiment (positive or negative).

\noindent\textbf{Reward Function.} A reward of 1 is assigned if all entities in a headline are correctly extracted and their corresponding sentiments are accurately predicted; otherwise, the reward is 0.

\begin{apxtcolorbox}[Example Query for SEntFiN]
\footnotesize
ITC, Godrej Consumer, United Spirits \& Jubilant Foods top bets in consumer sector for 2015: CLSA
\end{apxtcolorbox}

\subsection{AIOps Root Cause Analysis: RCA}
Efficient root cause analysis is essential in large-scale information technology operations, where system reliability directly influences business continuity and user experience. RCA \cite{logreasoner} is a real-world root cause analysis dataset for AIOps collected from Huawei's public forums. We merge the router and switch subset from the original paper to construct our final dataset, resulting in 2925 samples in total. For each task, the agent is provided a segment of the log, and required to indicate the type of the alert.

\noindent\textbf{Reward Function.} If the prediction is correct, a reward of 1 is assigned; otherwise, 0.

\begin{apxtcolorbox}[Example Query for RCA]
\footnotesize
WLAN/4/hwApVapStaFullTrap:OID [oid] VAP has the max number of stations notify.(APMAC=[OPAQUE], APName=[STRING], RADIOID=[INTEGER], WLANID=[INTEGER], FailCause=[INTEGER], PermitNum=[INTEGER], APID=[INTEGER])
\end{apxtcolorbox}

\subsection{AIOps Log Fault Diagnosis: LFD}
Fault diagnosis in cloud environments is critical for ensuring service availability, performance stability, and operational efficiency. We collect LFD dataset sourced from a past public AIOps competition\footnote{\url{https://tianchi.aliyun.com/competition/entrance/531947}}, which is publicly available now \footnote{\url{https://tianchi.aliyun.com/dataset/121954}}. The competition dataset is derived from real-world log data collected on Alibaba Cloud. We follow an open-sourced notebook\footnote{\url{https://tianchi.aliyun.com/notebook/357437}} to preprocess the dataset and select the first 3000 samples based on chronological order to construct the final dataset. Each sample consists of a segment of logs, and the agent is tasked with diagnosing the corresponding fault type-central processing unit (CPU), memory, or hardware.

\noindent\textbf{Reward Function.} If the prediction is correct, a reward of 1 is assigned; otherwise, 0.

\begin{apxtcolorbox}[Example Query for LFD]
\footnotesize
memory \#0xe2 | correctable ecc | asserted. memory \#0xe2 | correctable ecc | asserted
\end{apxtcolorbox}

\subsection{Text-to-SQL: SPIDER}
Text-to-SQL systems can significantly lower the barrier for accessing structured data, enabling users without database expertise to query complex information through natural language. SPIDER \cite{spider} is a large and high-quality text-to-SQL dataset. We utilise the test split of the original dataset as the final dataset, consisting of 2147 samples. For each sample, the agent is provided the database schema and the query problem, and is required to generated SQL code script to solve the problem.

\noindent\textbf{Reward Function.} We utilise the execution accuracy as the evaluation metric. If the generated SQL script can be executed successfully and return the ground-truth answer of the problem, a reward of 1 is assigned; otherwise, 0.

\begin{apxtcolorbox}[Example Query for SPIDER]
\footnotesize
Here is the database schema:

<DATABASE\_SCHEMA>

Question: How many authors do we have?
\end{apxtcolorbox}

\subsection{Knowledge-Augmented Text-to-SQL: BIRD}
Different from SPIDER, BIRD \cite{bird} is another text-to-SQL dataset that requires domain-specific knowledge to solve the given query problem. We utilise the development split of the original dataset as the final dataset, consisting of 1534 samples. For each sample, the agent is provided the database schema, the query problem, and the related domain-specific knowledge, and then is required to generated SQL code script to solve the problem.

\noindent\textbf{Reward Function.} Similar to SPIDER, we utilise the execution accuracy as the evaluation metric. If the generated SQL script can be executed successfully and return the ground-truth answer of the problem, a reward of 1 is assigned; otherwise, 0.

\begin{apxtcolorbox}[Example Query for BIRD]
\footnotesize
Here is the database schema:

<DATABASE\_SCHEMA>

Question: How many elders obtained the ``Supporter'' badge? 

Tips: ``Supporter'' is the Name of badge;  elders refers to Age > 65
\end{apxtcolorbox}

\subsection{Household Embodied Decision Making: ALFWorld}
Household embodied decision-making tasks represent a crucial step toward deploying intelligent agents in real physical environments. ALFWorld \cite{alfworld} is a household embodied interactive environment with six different task types. We utilise the train and unseen split of the original paper as the final dataset. Specifically, we randomly downsample 1866 tasks from the train split, and combine them with all 134 tasks from the unseen split to form the final online dataset. In each task, the agent receives a natural-language task description and then sequentially interacts with the environment to achieve the goal. The maximum number of interaction steps is set to 30.

\noindent\textbf{Reward Function.} We adopt the original ALFWorld reward function: the agent receives a reward of 1 upon successful task completion within 30 steps, and 0 otherwise.

\begin{apxtcolorbox}[Example Initial Query for ALFWorld]
\footnotesize
You are in the middle of a room. Looking quickly around you, you see a cabinet 13, a cabinet 12, a cabinet 11, a cabinet 10, a cabinet 9, a cabinet 8, a cabinet 7, a cabinet 6, a cabinet 5, a cabinet 4, a cabinet 3, a cabinet 2, a cabinet 1, a coffeemachine 1, a countertop 1, a drawer 5, a drawer 4, a drawer 3, a drawer 2, a drawer 1, a fridge 1, a garbagecan 1, a microwave 1, a sinkbasin 1, a stoveburner 4, a stoveburner 3, a stoveburner 2, a stoveburner 1, and a toaster 1.

Your task is to: heat some cup and put it in cabinet.
\end{apxtcolorbox}

\subsection{Scientific Embodied Decision Making: ScienceWorld}
 Scientific embodied reasoning tasks help develop agents that can perform experiments, test hypotheses, and reason about cause–effect relationships in dynamic environments. ScienceWorld \cite{scienceworld} is a scientific embodied interactive environment at the level of a standard elementary school science curriculum. We select 11 different types of tasks, including \textit{use-thermometer}, \textit{test-conductivity}, \textit{test-conductivity-of-unknown-substances}, \textit{find-animal}, \textit{find-living-thing}, \textit{find-non-living-thing}, \textit{find-plant}, \textit{grow-plant}, \textit{lifespan-longest-lived}, \textit{lifespan-longest-lived-then-shortest-lived}, and \textit{lifespan-shortest-lived}. For each task, we randomly sample approximately 200 variations to form the final dataset, with 1857 tasks in total. In each task, the agent receives a natural-language task description and then sequentially interacts with the environment to achieve the goal. The maximum number of interaction steps is set to the recommended number, ranging from 10 to 30.

\noindent\textbf{Reward Function.} We extend the original ScienceWorld reward function to the sparse reward setting: the agent receives a reward of 1 if the environment score achieves 100, and 0 otherwise.

\begin{apxtcolorbox}[Example Initial Query for ScienceWorld]
\footnotesize
Your task is to find a(n) plant. First, focus on the thing. Then, move it to the orange box in the bathroom.

This room is called the living room. In it, you see: the agent, a substance called air, a chair (On the chair is: nothing.), a couch (On the couch is: a white pillow.), a finger painting, a table (On the table is: nothing.).

You also see: A door to the hallway (that is open).

In your inventory, you see: an orange.
\end{apxtcolorbox}

\subsection{Web-based Deep Search: 2Wiki}
Deep search is a critical capability for modern web-based question answering. 2Wiki \cite{2wiki} is a multi-hop question-answering dataset. We adopt the development split as the seed dataset for preprocessing. Specifically, we remove relatively simple instances (e.g., yes–no and multiple-choice questions) and then randomly sample 2500 tasks to construct the final dataset. For each task, the agent is given a user query and then iteratively invokes appropriate tools to acquire evidence and progress toward an answer. The interaction horizon is limited to a maximum of five steps.

\noindent\textbf{Reward Function.} We follow previous works \cite{memento, deepresearcher} to utilise LLM-as-judge to judge the correctness of the answer with \texttt{gpt-4o-mini}. The agent receives a reward of 1 if the LLM judges that the answer matches the ground truth, and 0 otherwise.

\begin{apxtcolorbox}[Example Query for 2Wiki]
\footnotesize
Which country Zubaidah Bint Ja'Far's husband is from?
\end{apxtcolorbox}

\subsection{Complex Tabular Reasoning on Electronic Health Records: MIMIC-III}
Accessing and retrieving patient data from electronic health records is significant for routine clinical practice. We utilise MIMIC-III \cite{mimic-iii}, a real-world clinical dataset to benchmark the complex tabular reasoning capabilities of LLM agents. Specifically, we utilise the processed dataset from EHRSQL \cite{ehrsql} and follow EHRAgent \cite{ehragent} to filter unverifiable queries. After that, we randomly sample 2500 tasks to construct the final dataset. For each task, the agent is given a user query and then iteratively interacts with the workspace via Python script to derive the final answer. The interaction horizon is limited to a maximum of five steps.

\noindent\textbf{Reward Function.} We utilise the execution accuracy as the evaluation metric. If the generated Python
script can be executed successfully and return the ground-truth answer of the problem, a reward of 1
is assigned; otherwise, 0.

\begin{apxtcolorbox}[Example Query for MIMIC-III]
\footnotesize
count the number of times that patient 73713 had received a vancomycin laboratory test during this hospital visit.
\end{apxtcolorbox}
\section{Supplementary Results}
In this section, we describe the experimental setup and provide additional empirical results.

\subsection{Implementation Details}

\noindent\textbf{Model Configuration.} We serve the Qwen3-series LLMs \cite{qwen3} via the vLLM framework \cite{vllm} with bfloat16 precision for efficient inference. Across all settings, we set the sampling temperature to 0.1, the top-p value to 0.8, the top-k value to 20, and the presence penalty to 1.5. For faster responses, we use the non-thinking mode during inference. The specific model versions for each model size are listed below.
\begin{itemize}
    \item \textbf{Qwen3-32B}: \texttt{Qwen3-32B} (\url{https://huggingface.co/Qwen/Qwen3-32B}).
    \item \textbf{Qwen3-14B}: \texttt{Qwen3-14B} (\url{https://huggingface.co/Qwen/Qwen3-14B}).
    \item \textbf{Qwen3-8B}: \texttt{Qwen3-8B} (\url{https://huggingface.co/Qwen/Qwen3-8B}).
    \item \textbf{Qwen3-4B}: \texttt{Qwen3-4B-Instruct-2507} (\url{https://huggingface.co/Qwen/Qwen3-4B-Instruct-2507}).
    \item \textbf{Qwen3-30B-A3B-Instruct-2507}: \texttt{Qwen3-30B-A3B-Instruct-2507} (\url{https://huggingface.co/Qwen/Qwen3-30B-A3B-Instruct-2507}).
\end{itemize}
We exclude the original \texttt{Qwen3-4B} model because its mixed-thinking-mode training pipeline results in inferior performance. In its place, we employ the newly open-sourced instruct variant. We only utilise \texttt{Qwen3-30B-A3B-Instruct-2507} for deep search tasks due to its advanced tool-calling capabilities.

For the close-sourced LLM, we utilise \texttt{gemini-2.0-flash} via the API, and set the sampling temperature to 0.05, and the top-p value to 0.8. Note that we do not use the most advanced Gemini-2.5-series LLMs \cite{gemini-2.5} because its overly restrictive safety policy results in frequent response refusals on many DTLBench tasks.

For retrieval models, we utilise \texttt{gte-modernbert-base} as the pretrained embedding model for first-stage retrieval, and \texttt{gte-reranker-modernbert-base} as the base reranker model for reranking \cite{modern-bert, modern-bert-2}. Both are built on top of ModernBERT \cite{ModernBERT}, a modernised bidirectional encoder-only Transformer model with the context length of 8,192. The specific model versions are listed as below.
\begin{itemize}
    \item \textbf{Embedding model}: \texttt{gte-modernbert-base}  (\url{https://huggingface.co/Alibaba-NLP/gte-modernbert-base}).
    \item \textbf{Reranking model}: \texttt{gte-reranker-modernbert-base} (\url{https://huggingface.co/Alibaba-NLP/gte-reranker-modernbert-base}).
\end{itemize}

\begin{table}[tbp]
\caption{\textbf{Hyper-parameters for REINFORCE+LoRA.}}
\centering
\begin{tabular}{@{}lc@{}}
\toprule
\textbf{Name}                            & \textbf{Value}                  \\ \midrule
Matrix rank in LoRA                      & $8$                             \\
Scaling factor in LoRA                   & $16$                            \\
Target modules in LoRA                   & Q projections and V projections \\
Dropout in LoRA                          & $0.0$                           \\
Bias learning in LoRA                    & No                              \\
Batch size                               & $32$                            \\
KL coefficient                           & $0.1$                           \\
Optimiser                                & AdamW                           \\
Learning rate                            & $\{3e-6, 1e-5, 3e-5, 1e-4\}$    \\
Weight decay                             & $1e-5$                          \\ \bottomrule
\end{tabular}
\label{tab:hyper-parameter-reinforce-lora}
\end{table}

\noindent\textbf{Baselines.} As there are few related works that investigate the deployment-time learning setting, we include the following methods as our baselines. Note that two relevant baselines, DC \cite{dc} and ACE \cite{ace}, are proposed under the online adaptation setting in their respective papers. However, both methods rely on significantly stronger backbone LLMs than \texttt{Qwen3-32B} to effectively support the reasoning required by their context engineering strategies. Results reported in ACE \cite{ace} further show that these approaches can even lead to performance degradation when applied to weaker open-source backbone LLMs. Therefore, we do not include them as baselines.
\begin{itemize}
    \item \textbf{Zero-shot}. This baseline relies solely on the foundational capabilities of the underlying LLM and does not utilise any reward feedback from the environment, thus serving as a reference point for performance without deployment-time learning.
    \item \textbf{ICRL} \cite{icrl}. This baseline utilises the in-context learning and reflection capabilities of LLMs to implement in-context reinforcement learning. Specifically, it provides the most recent $K$ trials directly within the prompt, with each trial consisting of the query, the generated solution, as well as the reward feedback.
    \item \textbf{ICRLPlus} \cite{icrl}. This baseline extends ICRL by including only the most recent $K$ correct trials in the prompt, serving as a reference point for evaluating pure in-context learning performance.
    \item \textbf{NP-CBR}. This baseline represents the non-parametric variant of CBR, which has been widely adopted by previous works \cite{ds-agent, memento}. It relies solely on the pretrained reranker model’s capabilities, without any online adaptation based on reward feedback. Thus, it also serves as an ablation of our CASCADE framework, isolating the impact of the proposed bandit-based retrieval policy.
    \item \textbf{REINFORCE+LoRA} \cite{cursor-online-RL, lora-no-regret}. This baseline represents the most advanced parametric baseline that integrates the industrial success of on-policy RL with LoRA finetuning techniques. We summarise the hyper-parameter setting of REINFORCE+LoRA in Table~\ref{tab:hyper-parameter-reinforce-lora}, where we perform the hyper-parameter tuning on learning rate for each task.
\end{itemize}

\begin{table}[tbp]
\caption{\textbf{Hyper-parameters for CASCADE.}}
\centering
\begin{tabular}{@{}lcc@{}}
\toprule
\textbf{Name}                            & \textbf{Symbol} & \textbf{Value}               \\ \midrule
Number of cases for reranking            & $K$             & $32$                         \\
Exploration coefficient                  & $\alpha$        & $\{0.1,0.2,\ldots,1.0\}$          \\
Regularisation strength                  & $\lambda$        & $0.1$          \\ 
Batch size                               &                 & $32$                         \\
Training interval                        & $H$             & $32$                         \\
Optimiser                                &                 & AdamW                        \\
Learning rate                            & $\eta$          & $\{1e-6, 3e-6, 1e-5, 2e-5\}$ \\
Weight decay                             &                 & $1e-5$                       \\ \bottomrule
\end{tabular}

\label{tab:hyper-parameter-CASCADE}
\end{table}
\noindent\textbf{Implementation details of CASCADE.} We summarise the hyper-parameter setting of CASCADE in Table~\ref{tab:hyper-parameter-CASCADE}. We first tune the learning rate for each task while fixing the exploration coefficient to 0. Then, using the best learning rate, we further tune the exploration coefficient. As with REINFORCE+LoRA, the learning rate is the most critical hyper-parameter for CASCADE. In contrast, we have shown in the main paper that CASCADE is robust to the exploration coefficient. Therefore, the overall tuning effort is comparable to that of the REINFORCE+LoRA baseline. In practice, we suggest the users tuning hyper-parameters with a small-scale online experiments before deployment, which is also the best practice for online learning algorithms \cite{neural-ucb, cf-bandit, graph-bandit, instinct, ease}.

\subsection{Detailed Overall Results}
For clarity and ease of future reference, we summarise all main results in Table~\ref{tab:overall}, including the mean values and standard deviations. For \texttt{Qwen3-4B}, we do not report results on MIMIC-IV-MR due to its near-zero performance on this task. For \texttt{gemini-2.0-flash}, we do not report results on MIMIC-IV-MR, MIMIC-IV-MSR, MIMIC-IV-TLP due to the restricted dataset license.

\begin{sidewaystable}[htbp]
\caption{\textbf{Detailed overall comparison across 12 single-turn tasks in DTLBench.} We report the mean and standard deviation over five random seeds. For each task and model, the best performance is highlighted in bold. `N/A' indicates missing results: \texttt{Qwen3-4B} fails to run in the MIMIC-IV-MR dataset, and \texttt{gemini-2.0-flash} is not permitted due to dataset licensing restrictions.}
\resizebox{\linewidth}{!}{
\begin{tabular}{@{}lccccccccccccc@{}}
\toprule
\multicolumn{1}{c}{\textbf{}} &
  \textbf{DDXPlus} &
  \textbf{\begin{tabular}[c]{@{}c@{}}MIMIC-IV\\ -MR\end{tabular}} &
  \textbf{\begin{tabular}[c]{@{}c@{}}MIMIC-IV\\ -MSR\end{tabular}} &
  \textbf{\begin{tabular}[c]{@{}c@{}}MIMIC-IV\\ -TLP\end{tabular}} &
  \textbf{MUD} &
  \textbf{CMDL} &
  \textbf{Banking77} &
  \textbf{SEntFiN} &
  \textbf{RCA} &
  \textbf{LFD} &
  \textbf{SPIDER} &
  \textbf{BIRD} &
  \textbf{Average} \\ \midrule
\multicolumn{14}{l}{\cellcolor[HTML]{EFEFEF}\texttt{Qwen3-4B}} \\ \midrule
\textbf{Zero-shot} &
  36.58{\scriptsize$\pm$0.16} &
  N/A &
  50.53{\scriptsize$\pm$0.12} &
  26.85{\scriptsize$\pm$0.61} &
  34.72{\scriptsize$\pm$0.04} &
  19.24{\scriptsize$\pm$0.12} &
  62.33{\scriptsize$\pm$0.04} &
  24.72{\scriptsize$\pm$0.05} &
  23.52{\scriptsize$\pm$0.06} &
  81.92{\scriptsize$\pm$0.05} &
  70.15{\scriptsize$\pm$0.19} &
  52.07{\scriptsize$\pm$0.16} &
  43.88 \\
\textbf{ICRL} &
  44.52{\scriptsize$\pm$0.08} &
  N/A &
  51.01{\scriptsize$\pm$0.09} &
  28.43{\scriptsize$\pm$0.15} &
  35.90{\scriptsize$\pm$0.07} &
  21.60{\scriptsize$\pm$0.20} &
  62.92{\scriptsize$\pm$0.06} &
  27.15{\scriptsize$\pm$0.08} &
  21.72{\scriptsize$\pm$0.09} &
  83.31{\scriptsize$\pm$0.05} &
  71.69{\scriptsize$\pm$0.17} &
  51.71{\scriptsize$\pm$0.22} &
  45.45 \\
\textbf{ICRLPlus} &
  44.11{\scriptsize$\pm$0.06} &
  N/A &
  50.10{\scriptsize$\pm$0.16} &
  31.69{\scriptsize$\pm$0.19} &
  38.89{\scriptsize$\pm$0.09} &
  27.88{\scriptsize$\pm$0.27} &
  63.04{\scriptsize$\pm$0.03} &
  27.99{\scriptsize$\pm$0.09} &
  26.05{\scriptsize$\pm$0.04} &
  83.37{\scriptsize$\pm$0.04} &
  71.67{\scriptsize$\pm$0.15} &
  52.86{\scriptsize$\pm$0.25} &
  47.06 \\
\textbf{NP-CBR} &
  62.69{\scriptsize$\pm$0.23} &
  N/A &
  56.56{\scriptsize$\pm$0.05} &
  34.93{\scriptsize$\pm$0.37} &
  44.84{\scriptsize$\pm$0.06} &
  45.92{\scriptsize$\pm$0.24} &
  77.36{\scriptsize$\pm$0.06} &
  35.19{\scriptsize$\pm$0.27} &
  53.70{\scriptsize$\pm$0.19} &
  85.42{\scriptsize$\pm$0.05} &
  73.43{\scriptsize$\pm$0.18} &
  52.93{\scriptsize$\pm$0.11} &
  56.63 \\
\textbf{REINFORCE+LoRA} &
  61.39{\scriptsize$\pm$0.53} &
  N/A &
  \textbf{59.65{\scriptsize$\pm$0.64}} &
  \textbf{44.70{\scriptsize$\pm$1.60}} &
  \textbf{47.69{\scriptsize$\pm$0.82}} &
  45.89{\scriptsize$\pm$2.12} &
  65.89{\scriptsize$\pm$0.48} &
  35.66{\scriptsize$\pm$1.64} &
  49.31{\scriptsize$\pm$0.94} &
  85.57{\scriptsize$\pm$0.44} &
  73.42{\scriptsize$\pm$0.55} &
  53.12{\scriptsize$\pm$0.30} &
  56.57 \\
\textbf{CASCADE (Ours)} &
  \textbf{70.04{\scriptsize$\pm$1.07}} &
  N/A &
  57.43{\scriptsize$\pm$0.33} &
  36.33{\scriptsize$\pm$0.60} &
  46.87{\scriptsize$\pm$0.36} &
  \textbf{49.12{\scriptsize$\pm$0.80}} &
  \textbf{80.10{\scriptsize$\pm$0.30}} &
  \textbf{37.96{\scriptsize$\pm$0.25}} &
  \textbf{57.45{\scriptsize$\pm$0.82}} &
  \textbf{86.45{\scriptsize$\pm$0.18}} &
  \textbf{75.90{\scriptsize$\pm$0.88}} &
  \textbf{54.59{\scriptsize$\pm$0.50}} &
  \textbf{59.29} \\ \midrule
\multicolumn{14}{l}{\cellcolor[HTML]{EFEFEF}\texttt{Qwen3-8B}} \\ \midrule
\textbf{Zero-shot} &
  45.20{\scriptsize$\pm$0.11} &
  5.94{\scriptsize$\pm$0.25} &
  50.49{\scriptsize$\pm$0.24} &
  35.60{\scriptsize$\pm$0.49} &
  39.76{\scriptsize$\pm$0.06} &
  36.22{\scriptsize$\pm$0.15} &
  61.37{\scriptsize$\pm$0.04} &
  26.49{\scriptsize$\pm$0.12} &
  38.08{\scriptsize$\pm$0.03} &
  74.61{\scriptsize$\pm$0.07} &
  60.96{\scriptsize$\pm$0.28} &
  43.96{\scriptsize$\pm$0.15} &
  43.22 \\
\textbf{ICRL} &
  49.50{\scriptsize$\pm$0.09} &
  3.19{\scriptsize$\pm$0.21} &
  53.40{\scriptsize$\pm$0.18} &
  46.76{\scriptsize$\pm$0.13} &
  40.68{\scriptsize$\pm$0.13} &
  29.40{\scriptsize$\pm$0.24} &
  60.43{\scriptsize$\pm$0.06} &
  26.12{\scriptsize$\pm$0.15} &
  21.11{\scriptsize$\pm$0.12} &
  81.53{\scriptsize$\pm$0.08} &
  67.66{\scriptsize$\pm$0.12} &
  43.25{\scriptsize$\pm$0.31} &
  43.59 \\
\textbf{ICRLPlus} &
  48.57{\scriptsize$\pm$0.13} &
  23.36{\scriptsize$\pm$0.40} &
  55.23{\scriptsize$\pm$0.24} &
  48.75{\scriptsize$\pm$0.13} &
  42.48{\scriptsize$\pm$0.08} &
  31.69{\scriptsize$\pm$0.13} &
  60.45{\scriptsize$\pm$0.02} &
  27.97{\scriptsize$\pm$0.05} &
  28.55{\scriptsize$\pm$0.26} &
  80.98{\scriptsize$\pm$0.05} &
  68.19{\scriptsize$\pm$0.09} &
  44.86{\scriptsize$\pm$0.18} &
  46.76 \\
\textbf{NP-CBR} &
  65.44{\scriptsize$\pm$0.16} &
  52.85{\scriptsize$\pm$0.48} &
  62.73{\scriptsize$\pm$0.17} &
  49.91{\scriptsize$\pm$0.25} &
  45.45{\scriptsize$\pm$0.12} &
  46.86{\scriptsize$\pm$0.15} &
  78.92{\scriptsize$\pm$0.22} &
  38.04{\scriptsize$\pm$0.06} &
  60.93{\scriptsize$\pm$0.06} &
  85.30{\scriptsize$\pm$0.08} &
  66.74{\scriptsize$\pm$0.37} &
  48.11{\scriptsize$\pm$0.21} &
  58.44 \\
\textbf{REINFORCE+LoRA} &
  65.55{\scriptsize$\pm$0.48} &
  4.94{\scriptsize$\pm$0.17} &
  \textbf{64.16{\scriptsize$\pm$0.16}} &
  \textbf{52.04{\scriptsize$\pm$2.76}} &
  \textbf{48.85{\scriptsize$\pm$0.56}} &
  \textbf{51.67{\scriptsize$\pm$1.78}} &
  68.84{\scriptsize$\pm$0.45} &
  \textbf{43.19{\scriptsize$\pm$0.88}} &
  54.17{\scriptsize$\pm$0.25} &
  86.37{\scriptsize$\pm$0.07} &
  69.08{\scriptsize$\pm$0.43} &
  44.77{\scriptsize$\pm$0.33} &
  54.47 \\
\textbf{CASCADE (Ours)} &
  \textbf{69.13{\scriptsize$\pm$0.37}} &
  \textbf{54.88{\scriptsize$\pm$0.46}} &
  63.97{\scriptsize$\pm$0.48} &
  51.20{\scriptsize$\pm$0.24} &
  47.90{\scriptsize$\pm$0.48} &
  47.75{\scriptsize$\pm$1.54} &
  \textbf{80.78{\scriptsize$\pm$0.38}} &
  39.77{\scriptsize$\pm$0.67} &
  \textbf{66.91{\scriptsize$\pm$0.35}} &
  \textbf{86.40{\scriptsize$\pm$0.43}} &
  \textbf{72.70{\scriptsize$\pm$0.69}} &
  \textbf{48.66{\scriptsize$\pm$0.46}} &
  \textbf{60.84} \\ \midrule
\multicolumn{14}{l}{\cellcolor[HTML]{EFEFEF}\texttt{Qwen3-14B}} \\ \midrule
\textbf{Zero-shot} &
  55.75{\scriptsize$\pm$0.10} &
  11.91{\scriptsize$\pm$0.26} &
  52.92{\scriptsize$\pm$0.09} &
  40.92{\scriptsize$\pm$0.13} &
  50.84{\scriptsize$\pm$0.11} &
  38.65{\scriptsize$\pm$0.15} &
  65.86{\scriptsize$\pm$0.04} &
  33.03{\scriptsize$\pm$0.08} &
  36.66{\scriptsize$\pm$0.13} &
  78.41{\scriptsize$\pm$0.06} &
  66.29{\scriptsize$\pm$0.15} &
  48.75{\scriptsize$\pm$0.10} &
  48.33 \\
\textbf{ICRL} &
  56.50{\scriptsize$\pm$0.07} &
  1.18{\scriptsize$\pm$0.90} &
  54.65{\scriptsize$\pm$0.15} &
  48.47{\scriptsize$\pm$0.09} &
  49.39{\scriptsize$\pm$0.16} &
  40.52{\scriptsize$\pm$0.12} &
  66.99{\scriptsize$\pm$0.05} &
  32.44{\scriptsize$\pm$0.14} &
  32.03{\scriptsize$\pm$0.11} &
  82.12{\scriptsize$\pm$0.03} &
  69.29{\scriptsize$\pm$0.11} &
  48.59{\scriptsize$\pm$0.19} &
  48.51 \\
\textbf{ICRLPlus} &
  56.16{\scriptsize$\pm$0.15} &
  35.39{\scriptsize$\pm$0.41} &
  55.39{\scriptsize$\pm$0.11} &
  49.70{\scriptsize$\pm$0.05} &
  50.09{\scriptsize$\pm$0.09} &
  41.91{\scriptsize$\pm$0.06} &
  67.14{\scriptsize$\pm$0.03} &
  34.71{\scriptsize$\pm$0.12} &
  38.47{\scriptsize$\pm$0.07} &
  82.52{\scriptsize$\pm$0.03} &
  69.39{\scriptsize$\pm$0.06} &
  48.33{\scriptsize$\pm$0.18} &
  52.43 \\
\textbf{NP-CBR} &
  70.69{\scriptsize$\pm$0.08} &
  61.28{\scriptsize$\pm$0.42} &
  56.77{\scriptsize$\pm$0.06} &
  53.15{\scriptsize$\pm$0.11} &
  56.33{\scriptsize$\pm$0.51} &
  56.11{\scriptsize$\pm$0.03} &
  81.02{\scriptsize$\pm$0.08} &
  42.98{\scriptsize$\pm$0.10} &
  66.41{\scriptsize$\pm$0.13} &
  85.31{\scriptsize$\pm$0.02} &
  71.57{\scriptsize$\pm$0.45} &
  52.14{\scriptsize$\pm$0.29} &
  62.81 \\
\textbf{REINFORCE+LoRA} &
  71.30{\scriptsize$\pm$0.80} &
  11.07{\scriptsize$\pm$0.45} &
  \textbf{58.21{\scriptsize$\pm$0.73}} &
  52.39{\scriptsize$\pm$2.41} &
  53.72{\scriptsize$\pm$0.30} &
  41.38{\scriptsize$\pm$0.14} &
  72.30{\scriptsize$\pm$0.33} &
  44.42{\scriptsize$\pm$1.02} &
  54.11{\scriptsize$\pm$0.82} &
  \textbf{86.47{\scriptsize$\pm$0.06}} &
  72.07{\scriptsize$\pm$0.62} &
  49.17{\scriptsize$\pm$0.46} &
  55.55 \\
\textbf{CASCADE (Ours)} &
  \textbf{75.66{\scriptsize$\pm$0.55}} &
  \textbf{62.35{\scriptsize$\pm$0.40}} &
  57.76{\scriptsize$\pm$0.24} &
  \textbf{54.55{\scriptsize$\pm$0.17}} &
  \textbf{56.95{\scriptsize$\pm$0.27}} &
  \textbf{59.06{\scriptsize$\pm$0.27}} &
  \textbf{82.16{\scriptsize$\pm$0.08}} &
  \textbf{45.10{\scriptsize$\pm$0.27}} &
  \textbf{73.68{\scriptsize$\pm$0.29}} &
  85.61{\scriptsize$\pm$0.24} &
  \textbf{74.66{\scriptsize$\pm$0.61}} &
  \textbf{53.68{\scriptsize$\pm$0.32}} &
  \textbf{65.10} \\ \midrule
\multicolumn{14}{l}{\cellcolor[HTML]{EFEFEF}{\texttt{Qwen3-32B}}} \\ \midrule
\textbf{Zero-shot} &
  63.77{\scriptsize$\pm$0.11} &
  12.79{\scriptsize$\pm$0.33} &
  51.36{\scriptsize$\pm$0.08} &
  48.56{\scriptsize$\pm$0.14} &
  25.53{\scriptsize$\pm$0.15} &
  39.12{\scriptsize$\pm$0.16} &
  66.66{\scriptsize$\pm$0.03} &
  34.30{\scriptsize$\pm$0.09} &
  45.36{\scriptsize$\pm$0.05} &
  84.05{\scriptsize$\pm$0.05} &
  63.84{\scriptsize$\pm$0.39} &
  44.65{\scriptsize$\pm$0.42} &
  48.33 \\
\textbf{ICRL} &
  66.11{\scriptsize$\pm$0.10} &
  12.95{\scriptsize$\pm$2.10} &
  50.82{\scriptsize$\pm$0.13} &
  50.23{\scriptsize$\pm$0.11} &
  57.03{\scriptsize$\pm$0.24} &
  44.06{\scriptsize$\pm$0.26} &
  67.59{\scriptsize$\pm$0.03} &
  33.77{\scriptsize$\pm$0.02} &
  41.83{\scriptsize$\pm$0.05} &
  83.55{\scriptsize$\pm$0.14} &
  70.95{\scriptsize$\pm$0.13} &
  50.86{\scriptsize$\pm$0.23} &
  52.48 \\
\textbf{ICRLPlus} &
  64.91{\scriptsize$\pm$0.13} &
  30.29{\scriptsize$\pm$0.83} &
  52.24{\scriptsize$\pm$0.09} &
  51.05{\scriptsize$\pm$0.22} &
  58.33{\scriptsize$\pm$0.22} &
  43.22{\scriptsize$\pm$0.21} &
  67.61{\scriptsize$\pm$0.06} &
  35.39{\scriptsize$\pm$0.28} &
  43.34{\scriptsize$\pm$0.09} &
  83.91{\scriptsize$\pm$0.07} &
  71.11{\scriptsize$\pm$0.17} &
  50.52{\scriptsize$\pm$0.21} &
  54.33 \\
\textbf{NP-CBR} &
  80.84{\scriptsize$\pm$0.50} &
  60.58{\scriptsize$\pm$0.20} &
  56.29{\scriptsize$\pm$0.07} &
  51.83{\scriptsize$\pm$0.20} &
  60.93{\scriptsize$\pm$0.28} &
  59.56{\scriptsize$\pm$0.07} &
  82.77{\scriptsize$\pm$0.06} &
  45.35{\scriptsize$\pm$0.11} &
  65.08{\scriptsize$\pm$0.16} &
  85.39{\scriptsize$\pm$0.03} &
  66.87{\scriptsize$\pm$0.25} &
  49.60{\scriptsize$\pm$0.36} &
  63.76 \\
\textbf{REINFORCE+LoRA} &
  76.90{\scriptsize$\pm$0.91} &
  13.88{\scriptsize$\pm$0.10} &
  \textbf{57.47{\scriptsize$\pm$0.55}} &
  \textbf{53.97{\scriptsize$\pm$1.21}} &
  60.32{\scriptsize$\pm$0.40} &
  52.63{\scriptsize$\pm$1.11} &
  73.87{\scriptsize$\pm$0.35} &
  47.26{\scriptsize$\pm$0.38} &
  58.20{\scriptsize$\pm$1.86} &
  \textbf{86.80{\scriptsize$\pm$0.25}} &
  73.54{\scriptsize$\pm$0.25} &
  52.35{\scriptsize$\pm$0.09} &
  58.93 \\
\textbf{CASCADE (Ours)} &
  \textbf{84.54{\scriptsize$\pm$0.65}} &
  \textbf{63.98{\scriptsize$\pm$0.32}} &
  56.79{\scriptsize$\pm$0.14} &
  52.25{\scriptsize$\pm$0.48} &
  \textbf{62.44{\scriptsize$\pm$0.26}} &
  \textbf{60.26{\scriptsize$\pm$0.13}} &
  \textbf{84.66{\scriptsize$\pm$0.18}} &
  \textbf{47.72{\scriptsize$\pm$0.35}} &
  \textbf{70.76{\scriptsize$\pm$0.13}} &
  86.15{\scriptsize$\pm$0.21} &
  \textbf{75.41{\scriptsize$\pm$0.46}} &
  \textbf{55.18{\scriptsize$\pm$0.61}} &
  \textbf{66.68} \\ \midrule
\multicolumn{14}{l}{\cellcolor[HTML]{EFEFEF}\texttt{gemini-2.0-flash}} \\ \midrule
\textbf{Zero-shot} &
  69.67{\scriptsize$\pm$0.39} &
  N/A &
  N/A &
  N/A &
  59.61{\scriptsize$\pm$0.07} &
  39.32{\scriptsize$\pm$0.06} &
  73.61{\scriptsize$\pm$0.24} &
  12.92{\scriptsize$\pm$0.30} &
  28.06{\scriptsize$\pm$0.11} &
  83.16{\scriptsize$\pm$0.16} &
  80.66{\scriptsize$\pm$0.14} &
  62.19{\scriptsize$\pm$0.53} &
  56.58 \\
\textbf{ICRL} &
  73.79{\scriptsize$\pm$0.07} &
  \multicolumn{1}{c}{N/A} &
  \multicolumn{1}{c}{N/A} &
  \multicolumn{1}{c}{N/A} &
  64.98{\scriptsize$\pm$0.24} &
  55.77{\scriptsize$\pm$0.21} &
  74.10{\scriptsize$\pm$0.08} &
  30.07{\scriptsize$\pm$0.35} &
  28.50{\scriptsize$\pm$0.16} &
  84.70{\scriptsize$\pm$0.04} &
  80.62{\scriptsize$\pm$0.10} &
  62.32{\scriptsize$\pm$0.29} &
  61.65 \\
\textbf{ICRLPlus} &
  73.52{\scriptsize$\pm$0.04} &
  \multicolumn{1}{c}{N/A} &
  \multicolumn{1}{c}{N/A} &
  \multicolumn{1}{c}{N/A} &
  66.00{\scriptsize$\pm$0.13} &
  56.55{\scriptsize$\pm$0.11} &
  73.92{\scriptsize$\pm$0.04} &
  31.82{\scriptsize$\pm$0.18} &
  30.09{\scriptsize$\pm$0.17} &
  84.60{\scriptsize$\pm$0.04} &
  80.61{\scriptsize$\pm$0.20} &
  61.99{\scriptsize$\pm$0.16} &
  62.12 \\
\textbf{NP-CBR} &
  86.91{\scriptsize$\pm$0.50} &
  N/A &
  N/A &
  N/A &
  66.33{\scriptsize$\pm$0.20} &
  \textbf{62.05{\scriptsize$\pm$0.10}} &
  81.94{\scriptsize$\pm$0.29} &
  37.64{\scriptsize$\pm$0.42} &
  69.02{\scriptsize$\pm$0.18} &
  85.11{\scriptsize$\pm$0.13} &
  82.85{\scriptsize$\pm$0.46} &
  64.30{\scriptsize$\pm$0.28} &
  70.68 \\
\textbf{CASCADE (Ours)} &
  \textbf{90.29{\scriptsize$\pm$0.33}} &
  N/A &
  N/A &
  N/A &
  \textbf{67.13{\scriptsize$\pm$0.39}} &
  61.99{\scriptsize$\pm$0.24} &
  \textbf{83.33{\scriptsize$\pm$0.13}} &
  \textbf{40.79{\scriptsize$\pm$0.62}} &
  \textbf{75.96{\scriptsize$\pm$0.56}} &
  \textbf{86.29{\scriptsize$\pm$0.17}} &
  \textbf{83.05{\scriptsize$\pm$0.16}} &
  \textbf{64.42{\scriptsize$\pm$0.59}} &
  \textbf{72.58} \\ \bottomrule
\end{tabular}
}
\label{tab:overall}
\end{sidewaystable}

\begin{table}[th]
\caption{\textbf{Ablation studies of CASCADE in 12 single-turn tasks in DTLBench.} All the results are based on \texttt{Qwen3-32B}. We report the mean and standard deviation over five random seeds. The best performance for each task is highlighted in bold.}
\resizebox{\linewidth}{!}{
\begin{tabular}{@{}l
>{\columncolor[HTML]{EFEFEF}}c cccccc@{}}
\toprule
\multicolumn{1}{c}{} &
  \textbf{\begin{tabular}[c]{@{}c@{}}CASCADE\\ (Ours)\end{tabular}} &
  \textbf{Zero-shot} &
  \textbf{\begin{tabular}[c]{@{}c@{}}CASCADE w/o\\ \scriptsize{Neural-LinLogUCB}\end{tabular}} &
  \textbf{\begin{tabular}[c]{@{}c@{}}CASCADE w/o \\ Exploration\end{tabular}} &
  \textbf{\begin{tabular}[c]{@{}c@{}}CASCADE w/\\ \small{LinLogUCB}\end{tabular}} &
  \textbf{\begin{tabular}[c]{@{}c@{}}CASCADE w/\\ \footnotesize{NeuralLogUCB}\end{tabular}} &
  \textbf{\begin{tabular}[c]{@{}c@{}}CASCADE w/\\ \footnotesize{NeuralLinUCB}\end{tabular}} \\ \midrule
\textbf{DDXPlus} &
  \textbf{84.54{\scriptsize$\pm$0.65}} &
  63.77{\scriptsize$\pm$0.11} &
  80.84{\scriptsize$\pm$0.50} &
  84.09{\scriptsize$\pm$0.46} &
  80.48{\scriptsize$\pm$0.81} &
  81.73{\scriptsize$\pm$0.28} &
  83.96{\scriptsize$\pm$0.36} \\
\textbf{MIMIC-IV-MR} &
  \textbf{63.98{\scriptsize$\pm$0.32}} &
  12.79{\scriptsize$\pm$0.33} &
  60.58{\scriptsize$\pm$0.20} &
  63.60{\scriptsize$\pm$0.33} &
  58.63{\scriptsize$\pm$0.59} &
  61.53{\scriptsize$\pm$0.48} &
  61.96{\scriptsize$\pm$1.11} \\
\textbf{MIMIC-IV-MSR} &
  52.25{\scriptsize$\pm$0.48} &
  48.56{\scriptsize$\pm$0.14} &
  51.83{\scriptsize$\pm$0.20} &
  50.79{\scriptsize$\pm$0.44} &
  51.80{\scriptsize$\pm$0.40} &
  \textbf{53.00{\scriptsize$\pm$0.21}} &
  51.95{\scriptsize$\pm$0.55} \\
\textbf{MIMIC-IV-TLP} &
  56.79{\scriptsize$\pm$0.14} &
  51.36{\scriptsize$\pm$0.08} &
  56.29{\scriptsize$\pm$0.07} &
  57.09{\scriptsize$\pm$0.18} &
  55.54{\scriptsize$\pm$0.51} &
  \textbf{57.28{\scriptsize$\pm$0.57}} &
  56.01{\scriptsize$\pm$0.72} \\
\textbf{MUD} &
  \textbf{62.44{\scriptsize$\pm$0.26}} &
  25.53{\scriptsize$\pm$0.15} &
  60.93{\scriptsize$\pm$0.28} &
  62.09{\scriptsize$\pm$0.49} &
  61.41{\scriptsize$\pm$0.48} &
  61.76{\scriptsize$\pm$0.44} &
  60.83{\scriptsize$\pm$0.95} \\
\textbf{CMDL} &
  60.26{\scriptsize$\pm$0.13} &
  39.12{\scriptsize$\pm$0.16} &
  59.56{\scriptsize$\pm$0.07} &
  60.18{\scriptsize$\pm$0.40} &
  58.30{\scriptsize$\pm$0.46} &
  \textbf{60.93{\scriptsize$\pm$0.34}} &
  59.13{\scriptsize$\pm$0.28} \\
\textbf{Banking77} &
  \textbf{84.66{\scriptsize$\pm$0.18}} &
  66.66{\scriptsize$\pm$0.03} &
  82.77{\scriptsize$\pm$0.06} &
  84.26{\scriptsize$\pm$0.19} &
  79.76{\scriptsize$\pm$0.28} &
  82.40{\scriptsize$\pm$0.41} &
  82.44{\scriptsize$\pm$0.30} \\
\textbf{SEntFiN} &
  47.72{\scriptsize$\pm$0.35} &
  34.30{\scriptsize$\pm$0.09} &
  45.35{\scriptsize$\pm$0.11} &
  \textbf{48.37{\scriptsize$\pm$0.30}} &
  41.21{\scriptsize$\pm$0.27} &
  43.92{\scriptsize$\pm$0.86} &
  45.49{\scriptsize$\pm$0.86} \\
\textbf{RCA} &
  \textbf{70.76{\scriptsize$\pm$0.13}} &
  45.36{\scriptsize$\pm$0.05} &
  65.08{\scriptsize$\pm$0.16} &
  70.08{\scriptsize$\pm$0.03} &
  65.97{\scriptsize$\pm$0.67} &
  65.19{\scriptsize$\pm$0.24} &
  68.31{\scriptsize$\pm$0.81} \\
\textbf{LFD} &
  \textbf{86.15{\scriptsize$\pm$0.21}} &
  84.05{\scriptsize$\pm$0.05} &
  85.39{\scriptsize$\pm$0.03} &
  85.93{\scriptsize$\pm$0.22} &
  85.67{\scriptsize$\pm$0.08} &
  85.47{\scriptsize$\pm$0.27} &
  85.63{\scriptsize$\pm$0.32} \\
\textbf{SPIDER} &
  75.41{\scriptsize$\pm$0.46} &
  63.84{\scriptsize$\pm$0.39} &
  66.87{\scriptsize$\pm$0.25} &
  74.74{\scriptsize$\pm$0.81} &
  74.71{\scriptsize$\pm$0.29} &
  75.28{\scriptsize$\pm$0.31} &
  \textbf{75.62{\scriptsize$\pm$0.57}} \\
\textbf{BIRD} &
  55.18{\scriptsize$\pm$0.61} &
  44.65{\scriptsize$\pm$0.42} &
  49.60{\scriptsize$\pm$0.36} &
  55.25{\scriptsize$\pm$0.68} &
  53.95{\scriptsize$\pm$0.39} &
  \textbf{56.21{\scriptsize$\pm$0.20}} &
  54.55{\scriptsize$\pm$0.55} \\ \midrule
\textbf{Average} &
  \textbf{66.68} &
  48.33 &
  63.76 &
  66.37 &
  63.95 &
  65.39 &
  65.49 \\ \bottomrule
\end{tabular}
}
\label{tab:ablation}
\end{table}
\subsection{Ablation Studies}
Here, we perform the ablation studies of CASCADE in 12 single-turn tasks in DTLBench and aggregate the results in Table~\ref{tab:ablation}. Specifically, we consider the following ablation variants:
\begin{itemize}
    \item \textbf{Zero-shot}. We include this baseline as the ablation variant of CASCADE that does not perform deployment-time learning.
    \item \textbf{CASCADE w/o Neural-LinLogUCB}. This is exactly the non-parametric baseline NP-CBR, which does not utilise the reward feedback to adapt the retriever policy.
    \item \textbf{CASCADE w/o Exploration}. This ablation variant only performs exploitation during retrieval by setting the exploration coefficient $\alpha$ to 0.
    \item \textbf{CASCADE w/ LinLogUCB}. This ablation variant replaces the proposed Neural-LinLogUCB by LinLogUCB \cite{linlogucb}, which merely learns a linear weight based on the representation derived from the pretrained reranker model. We tune the exploration coefficient for each task.
    \item \textbf{CASCADE w/ NeuralLogUCB}. This ablation variant replaces the proposed Neural-LinLogUCB by NeuralLogUCB \cite{neural-logucb}, which learns a neural network based on the representation derived from the pretrained reranker model. We tune the exploration coefficient for each task.
    \item \textbf{CASCADE w/ Neural-LinUCB}. This ablation variant replaces the proposed Neural-LinLogUCB by Neural-LinUCB \cite{neural-linucb}, which similarly decouples representation learning and uncertainty estimation, but learns the reward model with the regression model instead of the logistic model. We tune the exploration coefficient for each task.
\end{itemize}
We can derive the following findings from Table~\ref{tab:ablation}. First, LLM agents can benefit from deployment-time learning by incorporating the reward feedback to adaptively adjust the behaviours. As an evidence, Zero-shot consistently shows the poorest performance across all tasks, demonstrating the importance of environmental feedback. Second, CBR agents can further benefit from the reward feedback by adapting the retriever policy via contextual bandit, which brings an additional 2.92 average absolute improvement. Third, the retriever policy in CBR agents must carefully balance exploration and exploitation. Over-reliance on the learned reward model without exploring uncertain cases results in an average absolute performance drop of 0.31.

Finally, selecting appropriate assumptions for the reward function in contextual bandits is crucial for CBR agents. On the one hand, a purely linear assumption, combined with a pretrained representation model, struggles to capture the complex interactions between query pairs. As a result, LinLogUCB consistently yields the lowest average performance among the contextual bandits. While incorporating neural networks into the reward model provides an average improvement of 1.44 over LinLogUCB, NeuralLogUCB remains heavily dependent on pretrained representations and still fail to generalise well across all tasks. In contrast, Neural-LinLogUCB decouples representation learning and uncertainty estimation, enabling end-to-end optimisation of the reward function. This leads to the best average performance of 66.68 across all tasks. On the other hand, established findings from the information retrieval community show that classification-based ranking objectives are typically more appropriate than regression-based ones. Motivated by this insight, CASCADE formulates CBR as a contextual bandit with binary feedback, achieving an additional average improvement of 1.19 over NeuralLinUCB.

\subsection{Extension of CASCADE to Multiple Cases}
Previous studies have shown that increasing the number of retrieved cases may bring further performance improvement to the CBR framework \cite{memento}. We now extend our analysis of CASCADE to the multiple-case setting. Although the proposed Neural-LinLogUCB algorithm in CASCADE is originally designed for the single-case setting, it can be naturally generalised to multiple cases by following the combinatorial neural bandit framework \cite{cnb}, which selects the top-$k$ cases with the highest UCB scores. Fig.~\ref{fig:multicase}a illustrates the performance of NP-CBR, REINFORCE+LoRA, and CASCADE, as the number of retrieved cases varies. We can find that CASCADE consistently outperforms NP-CBR across most settings, further demonstrating its effectiveness. Moreover, the results indicate that CASCADE’s performance exhibits a steady improvement as the number of retrieved cases increases. Notably, REINFORCE+LoRA surpasses CASCADE in the single-case setting for 3 out of 12 tasks (MIMIC-IV-MSR, MIMIC-IV-TLP, and LFD); however, this gap is closed once the number of retrieved cases increases to 4. This finding underscores the potential of memory-based learning mechanisms to outperform gradient-based ones through appropriate context engineering.

Moreover, we evaluate CASCADE on a representative subset of single-turn tasks to investigate how its performance scales with model size and the number of retrieved cases. As illustrated in Fig.~\ref{fig:multicase}b, CASCADE generally benefits from retrieving more cases and operating with larger models. However, the upper bound of such performance gain is inherently constrained by the foundational capabilities of the LLMs. Therefore, scaling up the number of retrieved cases alone does not ensure consistent performance gains. Furthermore, we present the results of CASCADE's absolute improvement over the best-performing non-parametric baseline NP-CBR in Fig.~\ref{fig:multicase}c. We observe that CASCADE outperforms NP-CBR in almost all settings, further demonstrating its empirical superiority.

\begin{figure}[thpb]
    \centering
    \includegraphics[width=\linewidth]{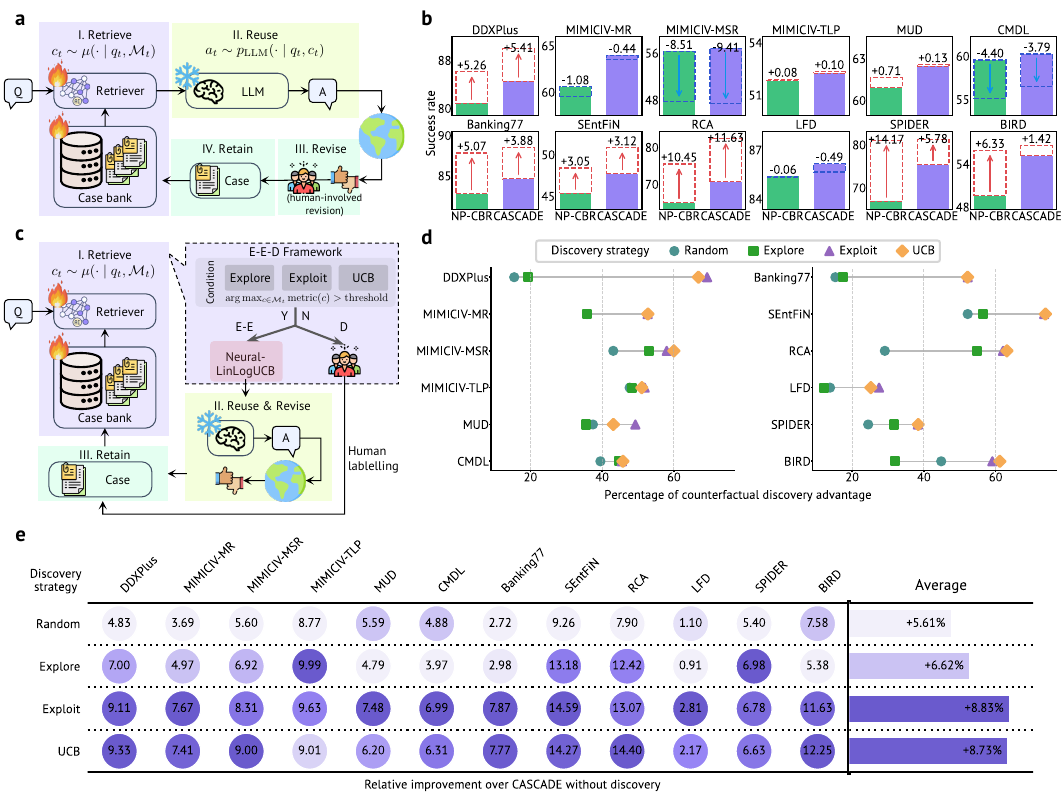}
    \caption{\textbf{Results for human-in-the-loop settings.} All the results are obtained using \texttt{Qwen3-32B} and reported based on five different random seeds. \textbf{a}, The CBR framework with human-involved revision. \textbf{b}, Success rates of CASCADE and NP-CBR on 12 single-turn tasks, comparing performance with and without human-involved revision. \textbf{c}, The CASCADE method under the exploration-exploitation-discovery (E-E-D) retrieval framework. If the maximum metric value in the case bank exceeds a predefined threshold, CASCADE performs Neural-LinLogUCB to balance exploration and exploitation; otherwise, it queries human experts for labelling. \textbf{d}, Percentage of counterfactual discovery advantage, i.e., the proportion of discovery steps yielding higher counterfactual rewards than exploration-exploitation actions, across four strategies. \textbf{e}, Relative performance gains over standard CASCADE when the discovery is driven by each of the four strategies.}
    \label{fig:human}
\end{figure}

\subsection{Human-in-the-loop Results}
We have shown that CASCADE can serve as a good deployment-time learner merely based on the binary feedback provided by the environment. In this subsection, we introduce two human-in-the-loop settings to demonstrate how CASCADE can benefit from human interactions to derive better performance.

\noindent\textbf{Human-involved revision.} In many real-world application scenarios, human experts participate in the Revise step to ensure the correctness of generated solutions (Fig.~\ref{fig:human}a). For example, in our preliminary applied work on automated software testing \cite{re4}, human test engineers are responsible for revising the generated test scripts by LLMs. Under this setting, a new case can always be safely retained at each timestep; thus, the coverage gap would be closed faster due to $p_{\min}=1$. We simulate this setting by retaining the case with the ground-truth label in each timestep, and make comparison between CASCADE and the best-performing non-parametric baseline NP-CBR in Fig.~\ref{fig:human}b. The results show that both CASCADE and NP-CBR benefit from human-involved revision in most settings. Meanwhile, performance deterioration can be observed in MIMIC-MR, MIMIC-MSR, CMDL, and LFD, where the retriever policies of NP-CBR and CASCADE struggle to manage the growing and diverse case bank within limited online timesteps. Importantly, CASCADE consistently outperforms NP-CBR across all settings, highlighting its superior adaptability and efficient learning of the retriever policy.

\noindent\textbf{CASCADE under E-E-D retrieval framework.} A notable advantage of CASCADE lies in its contextual bandit modelling of the retriever policy, which supports a human-in-the-loop setting where agents selectively seek human expert assistance within a limited budget. This is similar to the exploration-exploitation-discovery (E-E-D) framework in infinitely many-armed bandit literature \cite{inf-mab, bandit-book}, where exploitation means selecting an apparently relevant case, exploration means selecting an uncertain case, and discovery means seeking for a novel latent case when no sufficiently relevant cases are believed to exist in the current case bank. To simulate this setting, we provide the agent with a fixed budget of 10\% discovery steps across all timesteps, in which it can directly access the ground-truth label (Fig.~\ref{fig:human}c). For CASCADE, we introduce three discovery metrics: \textbf{(1) Explore} refers to the exploration term in Eq.~(\ref{eq:ucb}), i.e., $\lVert f(x_{t,c};\bomega_{t-1}) \rVert_{\bA_{t-1}^{-1}}$, which reflects the uncertainty of the case. \textbf{(2) Exploit} refers to the exploitation term in Eq.~(\ref{eq:ucb}), i.e., $\langle \btheta_{t-1}, f(x_{t,c};\bomega_{t-1}) \rangle$, which reflects the learned utility of the case. \textbf{(3) UCB} refers to the UCB score defined in Eq.~(\ref{eq:ucb}), which reflects the upper confidence bound of the expected utility of the case. A discovery step is triggered when the selected metric’s value falls below a predefined threshold. Since the ranges of these metrics evolve over time, we maintain a first-in-first-out queue of length 16, and dynamically set the threshold to the 10th percentile of the queued values. As a baseline, we introduce the \textbf{Random} strategy, which randomly determines whether to perform discovery with a probability of 10\%.

To assess these discovery strategies, we first report the percentage of counterfactual discovery advantage, i.e., the proportion of discovery steps yielding higher counterfactual rewards than exploration-exploitation actions, in Fig.~\ref{fig:human}d. As expected, the results indicate that Exploit and UCB are the best-performing discovery metrics, which consistently outperform the others in all the settings. Specifically, the Exploit and UCB strategies trigger discovery whenever no relevant cases exist in the case bank based on the learned utility function and the most optimistic estimate of expected utility, respectively. To further evaluate the benefits of the discovery mechanism, we report the relative improvement of each discovery strategy over standard CASCADE in Fig.~\ref{fig:human}e. Notably, Exploit and UCB achieve average relative performance gains of 8.83\% and 8.73\%, respectively, whereas Random provides only a 5.61\% improvement. This demonstrates that the contextual bandit modelling in CASCADE provides an effective discovery mechanism.
\section{Coverage Gap Analysis}
In this section, we provide the complete theoretical analyses on the coverage gap of CASCADE. To simplify the analyses, we make the following necessary assumptions:

\begin{assumption}
\label{assumption:lipschitz}
    There exists a metric space $(\mathcal{Q},d_\mathcal{Q})$ and a constant $L_{\mathcal{Q}}>0$, such that for any query $q \in \mathcal{Q}$ and any case $c=(q',a',r'=1)$, we have
    \[
    1-\bar{\mathcal{R}}(q,c) \le L_{\mathcal{Q}}\cdot d_\mathcal{Q}(q,q').
    \]
\end{assumption}
Assumption~\ref{assumption:lipschitz} formalises the principle that similar problems share similar solutions, ensuring that small perturbations in the query would not lead to large variations in the expected utility for LLMs.

\begin{assumption}
\label{assumption:non-zero}
For query $q_t$, and any case $c \in \{(q,a,r=1)|q\in\mathcal{Q},a\in\mathcal{A}\} \cup \emptyset$, there exists a positive constant $p_{\min} > 0$ such that 
\[
\bar{\mathcal{R}}(q_t,c) = \mathbb{E}_{a_t\sim p_{\text{LLM}}(\cdot|q_t,c)}[\mathcal{R}(q_t,a_t)] \ge p_{\min}.
\]
\end{assumption}
Assumption~\ref{assumption:non-zero} requires that LLMs maintain a strictly positive probability of generating a correct solution for any query, and that conditioning on any case, even an irrelevant one, does not reduce this probability to zero. We emphasise that this assumption is mainly introduced as a technical simplification to facilitate a clean theoretical analysis.

\begin{assumption}
\label{assumption:iid}
The queries $\{q_t\}_{t=1}^T$ are independently and identically distributed (i.i.d.) draws from a distribution $\mathcal{P}$ over $\mathcal{Q}$. There exists constants $d_0 > 0$, $C_0>0$, and a diameter $D\in(0,\infty)$ such that $\forall q \in\mathcal{Q}$ and $0< \varepsilon \le D$,
\[
\mathcal{P}(B(q,\varepsilon)) \ge C_0 \varepsilon^{d_0},
\]
where $B(q,\varepsilon)=\{q'\in\mathcal{Q}\mid d_\mathcal{Q}(q,q')\le\varepsilon\}$.
\end{assumption}

Assumption \ref{assumption:iid} models the query stream as i.i.d. draws from a distribution with a mild density condition. This is standard in non-parametric and bandit analyses and prevents degeneracy by ensuring that every metric ball has non-negligible probability mass. This assumption guarantees that similar queries recur with sufficient frequency, allowing the case bank to gradually cover the query space and making sub-linear coverage gap achievable.

Now, we can present the coverage gap bound as below.

\begin{theorem}
\label{theorem:coverage-gap-bound}
Suppose Assumptions~\ref{assumption:lipschitz},~\ref{assumption:non-zero},~\ref{assumption:iid} hold. For any $\delta\in(0,1)$, define
$t_0:=
\left\lceil
\frac{8}{p_{\min}}
\log\Big(\frac{32}{\delta\,p_{\min}}\Big)+1
\right\rceil.
$
Then, with probability at least $1-\delta$, the accumulated coverage gap satisfies
\[
R_T^\Delta
\le t_0 + L_{\mathcal{Q}}\Big(\frac{2\log\big(\frac{4T}{\delta}\big)}{C_0 p_{\min}}\Big)^{1/d_0} \sum_{t=t_0+1}^T (t-1)^{-1/d_0}.
\]
\end{theorem}
\begin{proof}
We first relate $\Delta_t$ to a nearest-neighbour distance in the case bank $\mathcal{M}_t$. Denote the distance between query $q_t$ and the nearest query in the case bank $\mathcal{M}_t$ as $d_t=\min_{(q,a,1)\in\mathcal{M}_t} d_{\mathcal{Q}}(q_t,q)$. According to Assumption 1, we have
\begin{equation}
\label{eq:lipschitz}
\Delta_t=\mathcal{R}(q_t,a_t^\star)-\bar{\mathcal{R}}(q_t,c_t^\star) \le L_{\mathcal{Q}}d_t.
\end{equation}

Next, we derive the high-probability lower bound on the number of the cases in the case bank, i.e., $N_t=|\mathcal{M}_t|$. By Assumption 2, at each timestep $t$, we have $\mathbb{P}(r_t=1\mid\mathcal{F}_{t-1}) = \bar{\mathcal{R}}(q_t,c_t) \ge p_{\min}$. Let $U_t \sim \mathrm{Unif}(0,1)$, realise $r_t=\mathbb{I}\{U_t \le p_t\}$ with $p_t \ge p_{\min}$, and $z_t=\mathbb{I}\{U_t \le p_{\min}\}$; thus, we have $z_t\le r_t$. Define $Z_t=\sum_{s=1}^{t-1} z_s\sim\mathrm{Binomial}(t-1,p_{\min})$; thus, we have $N_t=\sum_{s=1}^{t-1} r_s \ge Z_t$.
By applying Chernoff bound to $Z_t$, we have
\[
\mathbb{P}\left(N_t \le \frac{p_{\min}(t-1)}{2}\right)
\le
\mathbb{P}\left(Z_t \le \frac{p_{\min}(t-1)}{2}\right)
\le
\exp\left(-\frac{p_{\min}(t-1)}{8}\right).
\]
Then, we can derive the following union bound over $t \ge t_0$:
\[
\mathbb{P}\left(\exists t \ge t_0: N_t \le \frac{p_{\min}(t-1)}{2}\right) \le \sum_{s=t_0}^{\infty} e^{-\frac{p_{\min}(s-1)}{8}} = \frac{e^{-p_{\min}(t_0-1)/8}}{1-e^{-p_{\min}/8}} \le \frac{16}{p_{\min}} e^{-p_{\min}(t_0-1)/8} \le \frac{\delta}{2},
\]
where the second inequality is because $\forall x\in(0,1),1-e^{-x}\ge x/2$. Thus, with probability at least $1-\delta/2$, we have 
\begin{equation}
\label{eq:nt-bound}
N_t\ge \frac{p_{\min}(t-1)}{2}, \quad \forall t\ge t_0.
\end{equation}

Next, we bound the minimum distance $d_t$ in each timestep. Without loss of generality, we consider timestep $t$ with $N_t\ge1$. Then, according to Assumption 3, for any $\varepsilon \in (0,D]$, we have
\[
\mathbb{P}(d_t>\varepsilon) \le \left(1-\mathcal{P}\left(B\left(q_t,\varepsilon\right)\right)\right)^{N_t} \le (1-C_0\varepsilon^{d_0})^{N_t} \le \exp(-C_0N_t\varepsilon^{d_0}),
\]
where the last inequality is because $\forall x\in(0,1),1-x\le e^{-x}$. After choosing $\varepsilon_t=(\log(2T/\delta)/C_0N_{t})^{1/d_0}$, we have
\[
\mathbb{P}(d_t>\varepsilon_t) \le \frac{\delta}{2T}.
\]

A union bound over $t\in[T]$ implies that with probability at least $1-\delta/2$,
\begin{equation}
\label{eq:dt-bound}
d_t \le \left(\frac{\log\big(\frac{4T}{\delta}\big)}{C_0N_{t}}\right)^{1/d_0},
\qquad
\forall t\in[T]\text{ such that }N_t\ge 1.
\end{equation}

Now, we can derive the upper bound for accumulated coverage gap. By taking the intersection of events in Eq.~(\ref{eq:nt-bound}) and Eq.~(\ref{eq:dt-bound}), with probability at least $1-\delta$, we have
\begin{align}
    R_T^\Delta & = \sum_{t=1}^T [\mathcal{R}(q_t,a_t^\star)-\bar{\mathcal{R}}(q_t,c_t^\star)] \nonumber \\
    & = \sum_{t=1}^{t_0} [\mathcal{R}(q_t,a_t^\star)-\bar{\mathcal{R}}(q_t,c_t^\star)] + \sum_{t=t_0}^T [\mathcal{R}(q_t,a_t^\star)-\bar{\mathcal{R}}(q_t,c_t^\star)] \nonumber \\
    & \le t_0 + \sum_{t=t_0+1}^T L_{\mathcal{Q}}d_t \nonumber \\
    & \le t_0 + \sum_{t=t_0+1}^T L_{\mathcal{Q}}\left(\frac{\log\big(\frac{2T}{\delta}\big)}{C_0N_t}\right)^{1/d_0} \nonumber \\
    & \le t_0 + \sum_{t=t_0+1}^T L_{\mathcal{Q}}\left(\frac{2\log\big(\frac{2T}{\delta}\big)}{C_0p_{\min}(t-1)}\right)^{1/d_0} \nonumber \\
    & = t_0 + L_{\mathcal{Q}}\Big(\frac{2\log\big(\frac{2T}{\delta}\big)}{C_0 p_{\min}}\Big)^{1/d_0} \sum_{t=t_0+1}^T (t-1)^{-1/d_0}, \nonumber
\end{align}
where the first inequality is due to $\Delta_t\le1$ and Eq.~(\ref{eq:lipschitz}), the second inequality is due to Eq.~(\ref{eq:dt-bound}), and the third inequality is due to Eq.~(\ref{eq:nt-bound}).
\end{proof}

\begin{remark}
From Theorem~\ref{theorem:coverage-gap-bound}, the coverage gap exhibits different scaling behaviours depending on the intrinsic dimension parameter $d_0$. If $d_0>1$, we can derive $\tilde{\mathcal{O}}(T^{\frac{d_0-1}{d_0}})$ coverage gap bound; if $0<d_0\le 1$, we can derive $\tilde{\mathcal{O}}(1)$ coverage gap bound.
\end{remark}
\section{Neural-Linear Logistic Bandit: Neural-LinLogUCB}
In this section, we formally introduce the proposed neural logistic bandit algorithm in CASCADE and then analyse its regret. Without loss of generality, we present theoretical analyses based on the standard neural contextual bandit with binary feedback (NCBF) setting \cite{neural-logucb, nlb}. For clarity, we make a slight abuse of notation during the derivation of the regret bound for the Neural-LinLogUCB algorithm. This section can be regarded as self-contained and independent from the rest of the paper.
\subsection{Preliminary}

\noindent\textbf{Contextual Bandit with Binary Feedback.} We consider the stochastic $K$-armed contextual bandit problem with binary feedback \cite{glm,po-glm,io-glm,neural-logucb,nlb}, where the total number of timesteps $T$ is known. At each timestep $t \in [T]$, the agent observes $K$ feature vectors $\{\bx_{t,k} \in \mathbb{R}^d \mid k \in [K]\}$ from the corresponding contextual information $\{x_{t,k}\mid k\in[K]\}$. Then, the agent selects an action $c_t$ and receives a reward $r_{t,c_t}\in\{0,1\}$. We assume that the binary reward follows a Bernoulli distribution, where the probability of $r_{t,c_t}=1$ is given by
\begin{equation}
    \mathbb{P}\{r_{t,c_t}=1\mid \bx_{t,c_t}\} = \sigma(h(\bx_{t,c_t})),
\end{equation}
where $h:\mathbb{R}^d\rightarrow\mathbb{R}$ is an unknown latent reward function, and $\sigma:\mathbb{R}\rightarrow [0,1]$ is the sigmoid function, i.e., $\sigma(x)=1/(1+e^{-x})$. We denote $L_\sigma$ and $\kappa_\sigma$ as the Lipschitz constant and the strong monotonicity constant of the sigmoid function, i,e., $\kappa_\sigma \le \dot\sigma(\cdot) \le L_\sigma$. Note that we have $L_\sigma\le 1/4$ for the sigmoid function. 

The objective is to maximise the accumulated rewards, or equivalently minimising the following \textit{pseudo regret} (or \textit{regret} for short):
\begin{equation}
    R_T = \mathbb{E} \left[\sum_{t=1}^T (r_{t,c_t^\star} - r_{t,c_t})\right]=\sum_{t=1}^T\left[ \sigma(h(\bx_t^\star)) - \sigma(h(\bx_{t,c_t}))\right], 
\end{equation}
where $c_t^\star=\arg\max_{c\in [K]} \mathbb{E}[r_{t,c}]$ denotes the optimal action at timestep $t$ that maximises the expected reward, and $\bx_t^\star=\bx_{t,c_t^\star}$ denotes the feature vector of the optimal action at timestep $t$.

\noindent\textbf{Deep Neural Networks.} We follow previous works on neural contextual bandits \cite{neural-ucb, neural-linucb, neural-logucb} to utilise the fully connected network with the ReLU activation function as the network architecture. Specifically, the neural network $f(\bx;\bomega)$ with depth $L\ge2$ is defined as:
\begin{equation}
  f(\bx;\bomega) = \sqrt{m}\phi(\bW_L\phi(\bW_{L-1}\cdots\phi(\bW_1\bx))), 
\end{equation}
where $\phi(x)=\max\{0,x\}$ is the ReLU activation function, $\bW_1\in\mathbb{R}^{m\times d}$, $\bW_i\in\mathbb{R}^{m\times m},2\le i\le L-1$, $\bW_L\in\mathbb{R}^{d\times m}$, and $\bomega=[\text{vec}(\bW_1)^\top, \cdots, \text{vec}(\bW_L)^\top]^\top \in \mathbb{R}^p$ with $p=2md+m^2(L-2)$. We denote the gradient of the neural network by $\bg(\bx;\bomega)=\nabla_{\bomega} f(\bx;\bomega) \in \mathbb{R}^{d\times p}$.

\subsection{Neural-Linear Logistic Bandit Algorithm}
\noindent\textbf{Modelling Reward Function with Deep Representation Model and Shallow Linear Head.}
Although neural logistic bandit algorithms have shown powerful capabilities in approximating the complex reward function \cite{neural-ucb, neural-logucb, nlb}, they are inefficient at estimating uncertainty due to the large number of parameters involved in the entire network. To this end, we follow previous works \cite{neural-lin1, neural-lin2, neural-linucb} to decouple representation learning and uncertainty estimation, and propose the neural linear logistic upper confidence bound (Neural-LinLogUCB) algorithm. In particular, we model the reward function with a deep neural network model $f(\cdot;\bomega)$ for representation learning and a shallow linear head $\btheta \in \mathbb{R}^d$ for uncertainty estimation, i.e., $h(\bx) = \langle \btheta^\star, f(\bx;\bomega^\star) \rangle$, where $\btheta^\star$ and $\bomega^\star$ denote the unknown optimal parameters for linear head and neural network respectively. This decoupling enables the algorithm to benefit from both strong approximation capabilities of deep neural networks for representation learning and efficient exploration from the shallow linear head.

\noindent\textbf{Initialisation.} We follow the standard initialisation scheme from neural contextual bandits \cite{neural-ucb, neural-linucb, neural-logucb, nlb}, where each entry of $\bomega$ follows an appropriate Gaussian distribution. Specifically, for any layer $1\le l \le L-1$, we set $\bW_l=\begin{bmatrix} \bW & \bo \\ \bo & \bW \end{bmatrix}$, where the parameters in $\bW$ is independently sampled from the Gaussian distribution $\mathcal{N}(0,4/m)$; for the last layer $L$, we set $\bW_L=\begin{bmatrix} \bW^\top & -\bW^\top \end{bmatrix}$, where the parameters in $\bW$ is independently sampled from the Gaussian distribution $\mathcal{N}(0,2/m)$. For the linear head $\btheta$, we initialise its each entry by independently sampling from the Gaussian distribution $\mathcal{N}(0,1/d)$.

\noindent\textbf{Learning Algorithm.} In timestep $t$, the agent observes a set of feature vectors $\mathcal{X}_t=\{\bx_{t,1},...,\bx_{t,K}\}$. Then, it selects the action that maximises the following upper confidence bound:
\begin{equation}
    c_t=\arg\max_{k \in [K]} \left\{
    \langle \btheta_{t-1}, f(\bx_{t,k};\bomega_{t-1}) \rangle
    +
    \alpha_t \lVert f(\bx_{t,k};\bomega_{t-1}) \rVert_{\bA_{t-1}^{-1}}
    \right\},
\end{equation}
where $\{\alpha_t>0\}_{t=1}^T$ are exploration coefficients, and the matrix $\bA_t$ is defined as
\begin{equation}
\label{eq:theory:A}
    \bA_t = \lambda \bI + \sum_{s=1}^t f(\bx_{s,c_s};\bomega_{s-1})f(\bx_{s,c_s};\bomega_{s-1})^\top,
\end{equation}
where $\lambda >0$ denotes the hyper-parameter of the regularization strength in linear head estimation. After selecting the action $c_t$, the agent will receive a stochastic binary reward $r_t$. 

Then, we utilise the current history data $\{(\bx_{s,c_s}, r_s)\}_{s=1}^{t}$ to learn the neural network and linear head. First, we estimate the linear head by solving the the following logistic regression problem with L2 regularisation:
\begin{equation}
\label{eq:theory:logistic-regression}
    \min_{\btheta\in\mathbb{R}^d}\sum_{s=1}^{t} \left[
    -r_s \log\sigma\left(\btheta^\top f(\bx_{s,c_s};\bomega_{s-1})\right) - (1-r_s) \log\left(1-\sigma\left(\btheta^\top f(\bx_{s,c_s};\bomega_{s-1})\right)\right)
    \right] + \frac{\lambda}{2} \lVert \btheta \rVert^2_2.
\end{equation}

Then, we train the neural network by minimising the following cross-entropy loss function:
\begin{equation}
\label{eq:theory:loss}
    \mathcal{L}_t(\bomega) = \frac{1}{m}\sum_{s=1}^t \left[
    -r_s \log\sigma\left(\btheta_{s-1}^\top f(\bx_{s,c_s};\bomega)\right) - (1-r_s) \log\left(1-\sigma\left(\btheta_{s-1}^\top f(\bx_{s,c_s};\bomega)\right)\right)
    \right] + \frac{\lambda}{2} \lVert \bomega - \bomega_0 \rVert_2^2.
\end{equation}

We summarise the theoretical version of the algorithm in Algorithm \ref{alg:bandit}.

\begin{algorithm*}[tbhp]
\caption{Neural-LinLogUCB algorithm (theoretical version).}
\begin{algorithmic}[1]
\label{alg:bandit}
\STATE \textbf{Input:} Number of rounds $T$, exploration parameters $\{\alpha_t>0\}_{t\in[T]}$, regularisation parameter $\lambda$, number of total timesteps $T$.
\STATE \textbf{Initialisation:} Initialise the parameters of neural network $\bomega_0$ and the linear head $\btheta_0$. Initialise the matrix $\bA_0=\lambda\bI$.
\FOR{$t = 1,\dots,T$}
\STATE Observe the feature vectors $\mathcal{X}_t=\{\bx_{t,1},\cdots,\bx_{t,K}\}$
\STATE Choose arm $c_t = \arg\max_{k\in[K]} \btheta^\top_{t-1}f(\bx_{t,k};\bomega_{t-1}) + \alpha_t \lVert f(\bx_{t,k};\bomega_{t-1})\rVert_{\bA_{t-1}^{-1}}$
\STATE Observe the reward $r_t$
\STATE Update $\btheta_t$ by solving the logistic regression problem in Eq.~(\ref{eq:theory:logistic-regression})
\STATE Train neural networks via gradient descent by minimising the loss function in Eq.~(\ref{eq:theory:loss})
\STATE Update the matrix $\bA_t$ as defined in Eq.~(\ref{eq:theory:A})
\ENDFOR
\end{algorithmic}
\end{algorithm*}

\subsection{Regret Analysis}
Now, we analyse the upper regret bound of the proposed algorithm Neural-LogUCB. Our regret analysis follows existing neural contextual bandit works \cite{neural-ucb,neural-linucb, neural-logucb} built upon the neural tangent kernel (NTK) matrix \cite{ntk}.

\begin{definition}
    Let $\{\bx_i\}_{i=1}^{TK}$ be a set of feature vectors. Define
    \[
    \tilde{\bSigma}_{i,j}^{(1)}=\bSigma_{i,j}^{(1)}=\langle\bx_i,\bx_j\rangle,
    \]
    \[
    \bLambda_{i,j}^{(l)}=\begin{bmatrix} \bSigma_{i,i}^{(l)} & \bSigma_{i,j}^{(l)} \\ \bSigma_{j,i}^{(l)} & \bSigma_{j,j}^{(l)} \end{bmatrix},
    \]
    \[
    \bSigma_{i,j}^{(l+1)}=2\mathbb{E}_{u,v\sim\mathcal{N}(0,\bLambda_{i,j}^{(l)})}[\phi(u)\phi(v)],
    \]  
    \[
    \tilde{\bSigma}_{i,j}^{(l+1)}=2\tilde{\bSigma}_{i,j}^{(l)}\mathbb{E}_{u,v\sim\mathcal{N}(0,\bLambda_{i,j}^{(l)})}[\dot\phi(u)\dot\phi(v)] + \bSigma_{i,j}^{(l)}.
    \]
    Then, the neural tangent kernel matrix on this set of feature vectors is defined as 
    \[
    \bH=(\tilde{\bSigma}^{(L)}+\bSigma^{(L)})/2.
    \]
\end{definition}

Now, we introduce several necessary standard assumptions in neural contextual bandits \cite{neural-ucb, neural-linucb, neural-logucb, nlb}.

\begin{assumption}
\label{assumption:feature-vector}
    $\forall t\in[T]$, $\forall k\in[K]$, we assume that $\lVert \bx_{t,k}\rVert_2=1$ and each entry satisfies that $[\bx_{t,k}]_j = [\bx_{t,k}]_{j+d/2}$.
\end{assumption}
The first assumption is quite mild: the first part is only introduced for convenience, while the second part can be easily satisfied if we construct a new feature vector $\bx^\prime=(\bx^\top,\bx^\top)^\top/\sqrt{2}$. With this assumption, we can derive that the initialisation scheme above leads to $f(\bx_{t,k};\bomega_0)=\bo$ for all $t\in[T], k\in[K]$.

\begin{assumption}
\label{assumption:gradient-lipschitz}
    There exists a constant $\ell_{\mathrm{Lip}} > 0$ such that for all $\bx, \bx^\prime \in \{\bx_i\}_{i=1}^{TK}$,
    \[
        \lVert \nabla_{\bomega}f(\bx;\bomega_0)-\nabla_{\bomega}f(\bx^{\prime};\bomega_0)\rVert \le \ell_{\text{Lip}}\lVert \bx-\bx^\prime\rVert_2.
    \]
\end{assumption}
This assumption ensures the stability of the parameter gradients with respect to input variations, which can be satisfied when $TK$ feature vectors lie in a certain subspace of $\mathbb{R}^d$.

\begin{assumption}
\label{assumption:kappa}
The strong monotonicity constant is positive, i.e., $\kappa_\sigma>0$.
\end{assumption}

Finally, we introduce the assumption on the NTK matrix as below.
\begin{assumption}
\label{assumption:NTK}
    The NTK matrix $\bH$ is positive definite, i.e., there exists a constant $\lambda_0>0$, such that $\bH\succeq\lambda_0\bI$.
\end{assumption}
This assumption is also widely introduced by existing works, which requires the NTK matrix to be non-singular.

Now, we can present the regret bound of the proposed algorithm Neural-LinLogUCB.
\begin{theorem}
\label{theorem:neural-log-ucb-bound}
Suppose Assumptions~\ref{assumption:feature-vector},~\ref{assumption:gradient-lipschitz},~\ref{assumption:kappa},~\ref{assumption:NTK} hold. Assume that $\|\btheta^\star\|_2\le M$ for some positive constant $M>0$. For any $\delta\in(0,1)$, choose $\alpha_t$ in Neural-LinLogUCB as
\[
\alpha_t = \frac{1}{\kappa_\sigma} \Big(\nu\sqrt{2(d\log(1+Lt^2/\lambda) +\log(1/\delta))} + \sqrt{\lambda}M\Big).
\]
For $m\ge poly(L, d, 1/\kappa_\sigma,L_\sigma, \log(1/\delta), \log(TK))$, with probability at least $1-\delta$, we have
\begin{align*}
    R_T & \le \frac{C_1}{\kappa_\sigma}\sqrt{Td\log(1+LT^2/\lambda^2)}  \Big(\nu\sqrt{d\log(1+LT^2/\lambda) +\log(1/\delta)} + \sqrt{\lambda}M\Big) \nonumber \\
    & \quad + \frac{C_2d}{\kappa_\sigma}\sqrt{\log(1/\delta)L} \|\bh-\tilde{\bh}\|_{\bH^{-1}} \nonumber \\
    & \quad \cdot \Big(m^{-1/6}T^{5/3}\sqrt{d\log m}\lambda^{-2/3}L^3 + \ell_{\text{Lip}}m^{-1/2}T^{3/2}\lambda^{-1/2}\Big),
\end{align*}
where $C_1, C_2>0$ are absolute constants, $\bh=\big(h(\bx_1), h(\bx_2), \cdots, h(\bx_{TK})\big)^\top \in \mathbb{R}^{TK}$ and $\tilde{\bh}=\big(\btheta_0^\top f(\bx_1;\bomega_0), \cdots, \btheta_{T-1}^\top f(\bx_{TK};\bomega_{T-1})\big)^\top \in \mathbb{R}^{TK}$.
\end{theorem}

\begin{remark}
    With the same conditions of Theorem~\ref{theorem:neural-log-ucb-bound}, let $B$ be a constant such that $\|\bh-\tilde{\bh}\|_{\bH^{-1}}\le B$; if we choose a sufficiently wide neural network such that $m\ge T^7$, the regret of Neural-LinLogUCB is $R_T=\tilde{\mathcal{O}}(\frac{B}{\kappa_\sigma}\sqrt{T})$.
\end{remark}

\noindent\textbf{Comparison with Other Neural Contextual Bandits.} The regret bound of the proposed Neural-LinLogUCB is worse than that of Neural-LinUCB \cite{neural-linucb} $\mathcal{O}(B\sqrt{T})$ with the additional dependency of $\kappa_\sigma$. The reason for this gap is analogous to why NCBF-UCB \cite{neural-logucb} achieves a worse regret bound than Neural-UCB \cite{neural-ucb}: both arise from the fundamental difference between binary feedback and continuous feedback settings.

\noindent\textbf{Practical Implementation of Neural-LinLogUCB in CASCADE.} Due to the complexity of the natural-language context space, we adapt the theoretical Neural-LinLogUCB framework in CASCADE and introduce several modifications for practical implementation.

Firstly, to derive the feature vectors $\bx$ from the context information $x$ (i.e., the concatenation of the query and the case), one may apply a pretrained embedding model to convert the discrete natural language into continuous vector representation. However, prior work in the information retrieval community \cite{cedr} has shown that learning complex interactions via MLPs on top of such frozen embeddings can be challenging, particularly in online learning settings. Thus, we replace the MLPs in Neural-LinLogUCB by a Transformer-based encoder model that is initialised from \texttt{gte-reranker-modernbert-base} \cite{modern-bert} to benefit from its complex interaction modelling capabilities and good pretrained representations. Meanwhile, we initialise the linear head $\btheta$ with the last layer of \texttt{gte-reranker-modernbert-base}, ensuring that the initial behaviour of the reward model aligns with that of the pretrained reranker model.

Secondly, Neural-LinLogUCB requires updating both the linear head and the deep neural network at every timestep, which can be computationally expensive. To mitigate this overhead, we apply stochastic gradient descent to update the linear head at each step, serving as a computationally efficient approximation to solving the full logistic regression problem. For the encoder, we further reduce the cost by performing only a single-step mini-batch gradient descent update every $H$ timesteps, which maintains efficiency while still allowing the model to adapt over time.

These modifications enable Neural-LinLogUCB a practical contextual bandit solution for learning the retriever model in the case-based reasoning framework.

\subsection{Proof of the Main Results}
In this subsection, we provide the proof of the regret bound for Neural-LinLogUCB, which basically leverages the techniques from previous neural contextual bandits \cite{neural-ucb, neural-linucb, neural-logucb}. Firstly, we present several useful lemmas from them.
\begin{lemma}[Lemma C.1 from Xu et al. \cite{neural-linucb}]
\label{lemma:linearlisation}
    There exists $\bomega^\star \in \mathbb{R}^p$ such that $\sqrt{m}\lVert\bomega^\star-\bomega^{(0)}\rVert_2 \le \sqrt{(\bh-\tilde{\bh})^\top\bH^{-1}(\bh-\tilde{\bh})} \le B$ and it holds that 
    \[
        h(\bx_{t,k})=\btheta^{\star\top}f(\bx_{t,k};\bomega_{t-1}) + \btheta_0^\top\bg(\bx_{t,k};\bomega_0)(\bomega^\star-\bomega_0),
    \]
    for all $k\in[K]$ and $t\in[T]$.
\end{lemma}

Lemma~\ref{lemma:linearlisation} implies that the reward function at points $\bx$ can be approximated by a linear function around the initial parameter $\bomega_0$.

The next lemma show the upper bound on the distance between $\bomega_t$ and $\bomega_0$.

\begin{lemma}[Lemma 3 from Verma et al. \cite{neural-logucb}]
\label{lemma:neural-parameter-bound}
We have that $\lVert \bomega_t-\bomega_0 \rVert_2 \le \sqrt{t/(m\lambda)}, \forall t \in [T]$.
\end{lemma}

With Lemma~\ref{lemma:neural-parameter-bound}, we can derive the following lemmas regarding the gradients of the neural networks.
\begin{lemma}[Lemma 4 from Verma et al. \cite{neural-logucb}, Lemma B.5 and Lemma B.6 from Zhou et al. \cite{neural-ucb}]
\label{lemma:neural-gradient-bound}
Let $\tau\triangleq 2\sqrt{t/(m\lambda)}$. There exists constants $C_1, C_2>0$ such that with probability of at least $1-\delta$,
\[
\lVert\bg(\bx;\bomega_t)\rVert_F \le C_1\sqrt{mLd},
\]
\[
\lVert\bg(\bx;\bomega_0) - \bg(\bx;\bomega_t)\rVert_F \le C_2\sqrt{md\log m}\tau^{1/3}L^{7/2} = C_2m^{1/3}\sqrt{d\log m} \Big(\frac{t}{\lambda}\Big)^{1/6}L^{7/2},
\]
for any $\bx \in \{\bx_{i}\}_{i=1}^{TK}$.
\end{lemma}

Furthermore, we have the following lemma to bound the distance between the output of neural network and its linearisation.
\begin{lemma}[Lemma C.3 from Xu et al.\cite{neural-linucb}, Lemma 5 from Verma et al. \cite{neural-logucb}, Lemma B.4 from Zhou et al. \cite{neural-ucb}]
\label{lemma:neural-linearilisation-bound}
Let $\tau\triangleq 2\sqrt{t/(m\lambda)}$. There exists constants $C_3>0$ such that with probability of at least $1-\delta$,
\[
\lVert f(\bx;\bomega_t) - \bg(\bx;\bomega_0)(\bomega_t-\bomega_0) \rVert_2 \le C_3\tau^{4/3}L^3\sqrt{dm\log m}=C_3m^{-1/6}\sqrt{d\log m}\Big(\frac{t}{\lambda}\Big)^{2/3}L^3,
\]
for any $\bx \in \{\bx_{i}\}_{i=1}^{TK}$.
\end{lemma}

In addition, we can derive the following lemma regarding the bound of the learned representation vectors.
\begin{lemma}
\label{lemma:representation-bound}
There exists constants $C_3>0$ such that with probability of at least $1-\delta$,
\[
\lVert f(\bx;\bomega_t)\rVert_2 \le C_3\sqrt{Ldt/\lambda},
\]
for any $\bx \in \{\bx_{i}\}_{i=1}^{TK}$.
\end{lemma}
\begin{proof}
\begin{align}
\nonumber
    \lVert f(\bx;\bomega_t)\rVert_2 &= \lVert f(\bx;\bomega_t)- \bg(\bx;\bomega_0)(\bomega_t-\bomega_0) + \bg(\bx;\bomega_0)(\bomega_t-\bomega_0)\rVert_2 \\
\nonumber
    & \le \lVert f(\bx;\bomega_t)- \bg(\bx;\bomega_0)(\bomega_t-\bomega_0)\rVert_2 + \lVert \bg(\bx;\bomega_0)(\bomega_t-\bomega_0)\rVert_2 \\
\nonumber
    & \le \lVert f(\bx;\bomega_t)- \bg(\bx;\bomega_0)(\bomega_t-\bomega_0)\rVert_2 + \lVert\bg(\bx;\bomega_0)\rVert_F \cdot \lVert\bomega_t-\bomega_0\rVert_2 \\
\nonumber
    & \le C_3m^{-1/6}\sqrt{d\log m}\Big(\frac{t}{\lambda}\Big)^{2/3}L^3 + C_1\sqrt{mLd}\sqrt{t/(m\lambda)} \\
\nonumber
    & \le C_3\sqrt{Ldt/\lambda},
\end{align}
where the first inequality is due to triangle inequality, the second inequality is due to Cauchy-Schwarz inequality, the third inequality comes from Lemma~\ref{lemma:neural-linearilisation-bound}, Lemma~\ref{lemma:neural-gradient-bound} and Lemma~\ref{lemma:neural-parameter-bound}, and the last inequality can be achieved by a large enough $m$.
\end{proof}

Next, we introduce several lemmas related to the matrix $\bA_t$.

\begin{lemma}[Lemma C.6 in Xu et al. \cite{neural-linucb}, Lemma 10 and Lemma 11 from Abbasi-Yadkori et al. \cite{linucb}]
\label{lemma:feature-matrix-bound}
Let $\{\bx_t\}_{t=1}^{\infty}$ be a sequence in $\mathbb{R}^d$ and $\lambda>0$. Suppose $\|\bx_t\|_2\le G$ and $\lambda \ge \max\{1,G^2\}$ for some $G>0$, Let $\bA_t=\lambda\bI+\sum_{s=1}^t\bx_t\bx_t^\top$. Then, we have
\[
\det(\bA_t)\le(\lambda+tG^2/d)^d,
\]
\[
\sum_{t=1}^T \|\bx_t\|_{\bA_{t-1}^{-1}}^2 \le 2\log\frac{\det{\bA_T}}{\det \lambda\bI} \le 2d\log(1+TG^2/(\lambda d)).
\]
\end{lemma}

\begin{lemma}[Theorem 1 in Abbasi-Yadkori et al. \cite{lin-ucb}]
\label{lemma:self-normalized-martingales}
Let $\{\xi_t\}_{t=1}^{\infty}$ be a real-valued stochastic process and $\{\bx_t\}_{t=1}^{\infty}$ be a stochastic process in $\mathbb{R}^d$. Let $\mathcal{F}_t=\sigma(\bx_1,\cdots,\bx_{t+1},\xi_1,\cdots,\xi_{t})$ be a $\sigma$-algebra, such that $\bx_t$ and $\xi_t$ are $\mathcal{F}_{t-1}$-measurable. Let $\bA_t=\lambda\bI +\sum_{s=1}^t\bx_t\bx_t^\top$ for some constant $\lambda>0$ and $\bs_t=\sum_{t=1}^T\xi_t\bx_t$. If $\xi_t$ is $\nu$-sub-Gaussian conditioned on $\mathcal{F}_{t-1}$, then for any $\eta\in(0,1)$, with probability at least $1-\delta$, we have
\[
\|\bs_t\|_{\bA_t^{-1}}^2 \le 2\nu^2\log\frac{\det(\bA_t)^{1/2}\det(\lambda\bI)^{-1/2}}{\delta}.
\]
\end{lemma}

\begin{lemma}[Lemma C.5 from Verma et al. \cite{neural-linucb}]
    \label{lemma:bounded-vector-inverse-A-bound}
    Assume that $\bA_t=\lambda\bI + \sum_{s=1}^t \bphi_s\bphi_s^\top$, where $\bphi \in \mathbb{R}^d$ and for $\lambda,G>0$, $\|\bphi_s\|_2 \le G, \forall s \in [t]$. Let $\{\zeta_t\}_{t=1,\cdots}$ be a real-value sequence such that $|\zeta_t|\le U$ for some constant $U>0$. Then, we have
    \[
    \Big\|\bA_{t}^{-1} \sum_{s=1}^t \bphi_s\zeta_s\Big\|_2 
    \le 2Ud, \forall t=1,2,\cdots
    \]
\end{lemma}

Now, we start the proof by deriving the confidence bound of the estimated $\btheta_t$. Different from LinUCB \cite{lin-ucb} to directly bound $\|\btheta_t-\btheta^\star\|_{\bA_t}$, we need to consider the bias caused by representation learning via the deep neural networks. Otherwise, we would derive the linear regret bound.

Specifically, let $\bz_s= f(\bx_{s,c_s};\bomega_{s-1})\in\mathbb{R}^d$ be the representation vectors at timestep $s$, and we have $\bA_t=\lambda\bI+\sum_{s=1}^t \bz_s\bz_s^\top$. By Lemma~\ref{lemma:linearlisation}, the underlying reward function can be rewritten as
\[
    h(\bx_{t,k})=\btheta^{\star\top}f(\bx_{t,k};\bomega_{t-1}) + \btheta_0^\top\bg(\bx_{t,k};\bomega_0)(\bomega^\star-\bomega_0).
\]
For brevity, let $b_s=b(\bx_{s,c_s})=\btheta_0^\top\bg(\bx_{s,c_s};\bomega_0)(\bomega^\star-\bomega_0)$ denote the bias induced by the mismatch between the learned and true representation parameters, and $\delta_s=\sigma\big(\btheta^{\star\top}\bz_s+b_s\big)-\sigma\big(\btheta^{\star\top}\bz_s\big)$ as the corresponding bias after applying the sigmoid function.

According to Algorithm~\ref{alg:bandit}, we solve the regularised logistic regression loss at each timestep $t$ as:
\[
\ell_t(\btheta) = \sum_{s=1}^{t} \left[
    -r_s \log\sigma\left(\btheta^\top \bz_s\right) - (1-r_s) \log\left(1-\sigma\left(\btheta^\top \bz_s\right)\right)
    \right] + \frac{\lambda}{2} \lVert \btheta \rVert^2_2.
\]
We derive the estimator of $\btheta_t$ as the minimiser of the loss function above, i.e., $\btheta_t=\arg\min_{\btheta \in \mathbb{R}^d} \ell_t(\btheta)$, where the optimality condition is $\nabla_{\btheta} \ell_t(\btheta)=\bo$, i.e.,
\begin{equation*}
\label{eq:theta-t}
\lambda \btheta_t + \sum_{s=1}^t \Big( \sigma(\btheta_t^\top\bz_s)-r_s\Big)\bz_s = \bo.
\end{equation*}

Let $\epsilon_s$ be the observation noise in timestep $s$, i.e., 
\[
r_s=\sigma(\btheta^{\star\top}\bz_s+b_s)+\epsilon_s=\sigma(\btheta^{\star\top}\bz_s)+\delta_s+\epsilon_s.
\]
Let $\lambda'\in(0,1)$, and let $\xi_{t,s}=\lambda'\btheta_t^\top\bz_s + (1-\lambda')\btheta^{\star\top}\bz_s$. Then, according to the mean-value theorem, we have for some $\lambda'$ that
\[
\sigma(\btheta_t^\top\bz_s)-\sigma(\btheta^{\star\top}\bz_s)
=\dot\sigma(\xi_{t,s})\bz_s^\top(\btheta_t-\btheta^\star),
\]
where $\kappa_\sigma\le\dot\sigma(\xi_{t,s})\le L_\sigma$. This allows us to show that
\begin{equation}
\label{eq:theta-bias-relationship}
  \big(\lambda\bI+\widehat{\bA}_t\big)(\btheta_t-\btheta^\star) = \sum_{s=1}^t \epsilon_s \bz_s + \sum_{s=1}^t \delta_s \bz_s - \lambda\btheta^\star,  
\end{equation}

where the matrix $\widehat{\bA}_t$ is defined as:
\[
\widehat{\bA}_t = \sum_{s=1}^t \dot\sigma(\xi_{t,s})\bz_s\bz_s^\top.
\]
Then, we define the bias term as $\bb_t=(\lambda\bI+\widehat{\bA}_t)^{-1}\sum_{s=1}^t \delta_s\bz_s$, and introduce the following bias-corrected confidence bound.

\begin{lemma}
\label{lemma:bias-corrected-confidence} 
For any $\delta\in(0,1)$, with probability at least $1-\delta$, the distance between the estimated $\btheta_t$ and $\btheta^\star$ can be bounded as follows:
\begin{equation}
\label{eq:main-confidence}
\big\|\btheta_t-\btheta^\star-\bb_t\big\|_{\bA_t}
\le
\frac{1}{\kappa_\sigma} \Big(\nu\sqrt{2(d\log(1+Lt^2/\lambda) +\log(1/\delta))} + \sqrt{\lambda}M\Big).
\end{equation}
\end{lemma}
\begin{proof}
We first show that $\lambda\bI+\widehat{\bA}_t$ is spectrally sandwiched between $\kappa_\sigma\bA_t$ and $\bA_t$.
Specifically, by the definition of strong monotonicity constant of sigmoid function, i.e., $\kappa_\sigma\le\dot\sigma(\cdot)\le1$, we have
\begin{equation*}
\lambda\bI+\widehat{\bA}_t \succeq\lambda\bI+\kappa_\sigma\sum_{s=1}^t \bz_s\bz_s^\top \succeq \kappa_\sigma\big(\lambda\bI+\sum_{s=1}^t \bz_s\bz_s^\top\big)=\kappa_\sigma \bA_t.
\end{equation*}
Also, since $\widehat{\bA}_t \preceq \sum_{s=1}^t\bz_s\bz_s^\top = \bA_t - \lambda\bI$, we have
\begin{equation*}
\lambda\bI+\widehat{\bA}_t \preceq \lambda\bI + (\bA_t - \lambda\bI) = \bA_t.
\end{equation*}
Combining both, we can derive that
\begin{equation}
\label{eq:A-hat-A}
\kappa_\sigma \bA_t \preceq \lambda\bI+\widehat{\bA}_t \preceq \bA_t.
\end{equation}

Then, we use the definition of $\bb_t$ to rewrite Eq.~(\ref{eq:theta-bias-relationship}) as:
\begin{equation}
\label{eq:delta-decomp}
\btheta_t - \btheta^\star - \bb_t = (\lambda\bI+\widehat{\bA}_t)^{-1}\Big(\sum_{s=1}^t \epsilon_s\bz_s-\lambda\btheta^\star\Big).
\end{equation}

Now, we can bound the norm as below:
\begin{align}
    & \|\btheta_t - \btheta^\star - \bb_t\|_{\bA_t} \nonumber \\ 
    &= \|(\lambda\bI+\widehat{\bA}_t)^{-1}\Big(\sum_{s=1}^t \epsilon_s\bz_s-\lambda\btheta^\star\Big)\|_{\bA_t} \nonumber \\
    & = \sqrt{\Big(\sum_{s=1}^t \epsilon_s\bz_s-\lambda\btheta^\star\Big)^\top (\lambda\bI+\widehat{\bA}_t)^{-1/2}(\lambda\bI+\widehat{\bA}_t)^{-1/2}\bA_t(\lambda\bI+\widehat{\bA}_t)^{-1/2}(\lambda\bI+\widehat{\bA}_t)^{-1/2}\Big(\sum_{s=1}^t \epsilon_s\bz_s-\lambda\btheta^\star\Big)} \nonumber \\
    & \le \sqrt{\frac{1}{\kappa_\sigma}\Big(\sum_{s=1}^t \epsilon_s\bz_s-\lambda\btheta^\star\Big)^\top (\lambda\bI+\widehat{\bA}_t)^{-1}\Big(\sum_{s=1}^t \epsilon_s\bz_s-\lambda\btheta^\star\Big)} \nonumber \\
    & = \frac{1}{\sqrt{\kappa_\sigma}} \Big\|\sum_{s=1}^t \epsilon_s\bz_s-\lambda\btheta^\star\Big\|_{(\lambda\bI+\widehat{\bA}_t)^{-1}} \nonumber \\
    & \le \frac{1}{\sqrt{\kappa_\sigma}} \Big( \Big\|\sum_{s=1}^t \epsilon_s\bz_s\Big\|_{(\lambda\bI+\widehat{\bA}_t)^{-1}} + \Big\|\lambda\btheta^\star\Big\|_{(\lambda\bI+\widehat{\bA}_t)^{-1}} \Big) \nonumber \\ 
    & \le \frac{1}{\sqrt{\kappa_\sigma}} \Big( \frac{1}{\sqrt{\kappa_\sigma}}\Big\|\sum_{s=1}^t \epsilon_s\bz_s\Big\|_{\bA_t^{-1}} + \Big\|\lambda\btheta^\star\Big\|_{(\lambda\bI+\widehat{\bA}_t)^{-1}} \Big) \nonumber \\
    & \le \frac{1}{\kappa_\sigma} \Big\|\sum_{s=1}^t \epsilon_s\bz_s\Big\|_{\bA_t^{-1}} +\sqrt{\frac{\lambda}{\kappa_\sigma}} \|\btheta^\star\|_2, \label{eq:confidence-bound-1}
\end{align}
where the first inequality is due to $\lambda\bI+\widehat{\bA}_t \preceq \bA_t$ in Eq.~(\ref{eq:A-hat-A}), the second inequality is due to the triangle inequality, the third inequality is due to $\kappa_\sigma \bA_t \preceq \lambda\bI+\widehat{\bA}_t$ in Eq.~(\ref{eq:A-hat-A}), the fourth inequality is due to $\lambda\bI+\widehat{\bA}_t \succeq \lambda\bI$. 

Now, we can apply Lemma~\ref{lemma:self-normalized-martingales} to bound the first term, and the fact that $\|\btheta^\star\|_2\le M$ to bound the second term in Eq.~(\ref{eq:confidence-bound-1}). As such, we further have
\begin{align*}
\|\btheta_t - \btheta^\star - \bb_t\|_{\bA_t} 
& \le \frac{1}{\kappa_\sigma} \Big(
    \nu\sqrt{2\log\frac{\det(\bA_t)^{1/2}\det(\lambda\bI)^{-1/2}}{\delta}} + \sqrt{\lambda}M
    \Big) \\
& \le \frac{1}{\kappa_\sigma} \Big(
    \nu\sqrt{2(d\log(1+Lt^2/\lambda) +\log(1/\delta))} + \sqrt{\lambda}M
    \Big),
\end{align*}
where we apply Lemma~\ref{lemma:representation-bound} and Lemma~\ref{lemma:feature-matrix-bound} for the last inequality.
\end{proof}

Now we are ready to prove the regret bound of Algorithm~\ref{alg:bandit}.

\begin{proof}[Proof of Theorem~\ref{theorem:neural-log-ucb-bound}]
We first apply Lemma~\ref{lemma:linearlisation} to decompose the instantaneous latent reward gap into different parts and bound them individually. Specifically, there exists $\bomega^\star \in \mathbb{R}^p$ such that
\begin{align}
    & h(\bx_{t,c_t^\star}) - h(\bx_{t,c_t}) \nonumber \\
    &= \btheta_0^\top[\bg(\bx_{t,c_t^\star};\bomega_0) - \bg(\bx_{t,c_t};\bomega_0)](\bomega^\star - \bomega_0) + \btheta^{\star\top}[f(\bx_{t,c_t^\star};\bomega_{t-1}) - f(\bx_{t,c_t};\bomega_{t-1})] \nonumber \\
    &= \btheta_0^\top[\bg(\bx_{t,c_t^\star};\bomega_0) - \bg(\bx_{t,c_t};\bomega_0)](\bomega^\star - \bomega_0) \nonumber \\
    & \quad + \btheta^{\top}_{t-1}[f(\bx_{t,c_t^\star};\bomega_{t-1}) - f(\bx_{t,c_t};\bomega_{t-1})] \nonumber \\
    & \quad - (\btheta_{t-1}-\btheta^\star)^\top[f(\bx_{t,c_t^\star};\bomega_{t-1}) - f(\bx_{t,c_t};\bomega_{t-1})]. \label{eq:utility-gap-decomposition}
\end{align}
Now, we upper bound these three terms. We bound the first term as below:
\begin{align}
    & \btheta_0^\top[\bg(\bx_{t,c_t^\star};\bomega_0) - \bg(\bx_{t,c_t};\bomega_0)](\bomega^\star - \bomega_0) \nonumber \\
    & \le \|\btheta_0\|_2 \cdot \|\bg(\bx_{t,c_t^\star};\bomega_0) - \bg(\bx_{t,c_t};\bomega_0)\|_F \cdot \|\bomega^\star - \bomega_0\|_2 \nonumber \\
    & \le \ell_{\text{Lip}} \|\btheta_0\|_2 \cdot \|\bx_{t,c_t^\star} - \bx_{t,c_t}\|_2 \cdot \|\bomega^\star - \bomega_0\|_2 \nonumber \\
    & \le \ell_{\text{Lip}} \cdot 2(2+\sqrt{d^{-1}\log(1/\delta)}) \cdot 2 \cdot \sqrt{(\bh-\tilde{\bh})^\top\bH^{-1}(\bh-\tilde{\bh})/m} \nonumber \\
    & \le C_4\ell_{\text{Lip}}m^{-1/2}\sqrt{d^{-1}\log(1/\delta)(\bh-\tilde{\bh})^\top\bH^{-1}(\bh-\tilde{\bh})}, \label{eq:ugd-1}
\end{align}
where $C_4>0$ is the absolute constant, the first inequality is due to Cauchy-Schwarz inequality, the second inequality is due to Assumption~\ref{assumption:gradient-lipschitz}. For the third inequality, we bound $\|\btheta_0\|$ based on its initialisation scheme, bound $\|\bx_{t,c_t^\star} - \bx_{t,c_t}\|_2$ due to Assumption~\ref{assumption:feature-vector}, and bound $\|\bomega^\star - \bomega_0\|_2$ due to Lemma~\ref{lemma:linearlisation}. 

For the second term, we use the definition of the upper confidence bound in Algorithm~\ref{alg:bandit}, i.e., $\btheta^{\top}_{t-1}f(\bx_{t,c_t^\star};\bomega_{t-1}) +\alpha_t \|f(\bx_{t,c_t^\star}; \bomega_{t-1})\|_{\bA_{t-1}^{-1}}  \le \btheta^{\top}_{t-1}f(\bx_{t,c_t};\bomega_{t-1}) +\alpha_t \|f(\bx_{t,c_t}; \bomega_{t-1})\|_{\bA_{t-1}^{-1}}$, to derive the following upper bound:
\begin{align}
    \btheta^{\top}_{t-1}[f(\bx_{t,c_t^\star};\bomega_{t-1}) - f(\bx_{t,c_t};\bomega_{t-1})] \le \alpha_t \|f(\bx_{t,c_t}; \bomega_{t-1})\|_{\bA_{t-1}^{-1}} - \alpha_t \|f(\bx_{t,c_t^\star}; \bomega_{t-1})\|_{\bA_{t-1}^{-1}}, \label{eq:ugd-2}
\end{align}
where $\alpha_t$ is the upper bound derived from Lemma~\ref{lemma:bias-corrected-confidence}.

The last term can be bounded as below.
\begin{align}
    & - (\btheta_{t-1}-\btheta^\star)^\top[f(\bx_{t,c_t^\star};\bomega_{t-1}) - f(\bx_{t,c_t};\bomega_{t-1})] \nonumber \\
    & = - (\btheta_{t-1}-\btheta^\star-\bb_{t-1})^\top[f(\bx_{t,c_t^\star};\bomega_{t-1}) - f(\bx_{t,c_t};\bomega_{t-1})] \nonumber \\ 
    & \quad - \bb_{t-1}^\top [f(\bx_{t,c_t^\star};\bomega_{t-1}) - f(\bx_{t,c_t};\bomega_{t-1})] \nonumber \\
    & \le \|\btheta_{t-1}-\btheta^\star-\bb_{t-1}\|_{\bA_{t-1}} \cdot \|f(\bx_{t,c_t^\star};\bomega_{t-1})\|_{\bA_{t-1}^{-1}} \nonumber \\
    & \quad + \|\btheta_{t-1}-\btheta^\star-\bb_{t-1}\|_{\bA_{t-1}} \cdot \|f(\bx_{t,c_t};\bomega_{t-1})\|_{\bA_{t-1}^{-1}} \nonumber \\
    & \quad + \|\bb_{t-1}\|_2 \cdot \|f(\bx_{t,c_t^\star};\bomega_{t-1}) - f(\bx_{t,c_t};\bomega_{t-1})\|_2 \nonumber \\
    & \le \alpha_t \|f(\bx_{t,c_t^\star};\bomega_{t-1})\|_{\bA_{t-1}^{-1}} + \alpha_t \|f(\bx_{t,c_t};\bomega_{t-1})\|_{\bA_{t-1}^{-1}} \nonumber \\
    & \quad + \|\bb_{t-1}\|_2 \cdot \|f(\bx_{t,c_t^\star};\bomega_{t-1}) - f(\bx_{t,c_t};\bomega_{t-1})\|_2, \label{eq:ugd-3}
\end{align}
where the first inequality is due to the Cauchy-Schwarz inequality, and the second inequality is due to Lemma~\ref{lemma:bias-corrected-confidence}.

Summarising Eq.~(\ref{eq:ugd-2}) and Eq.~(\ref{eq:ugd-3}), we can derive that
\begin{align}
    & \btheta^{\top}_{t-1}[f(\bx_{t,c_t^\star};\bomega_{t-1}) - f(\bx_{t,c_t};\bomega_{t-1})] - (\btheta_{t-1}-\btheta^\star)^\top[f(\bx_{t,c_t^\star};\bomega_{t-1}) - f(\bx_{t,c_t};\bomega_{t-1})] \nonumber \\
    & \le 2\alpha_t\|f(\bx_{t,c_t};\bomega_{t-1})\|_{\bA_{t-1}^{-1}} + \|\bb_{t-1}\|_2 \cdot \|f(\bx_{t,c_t^\star};\bomega_{t-1}) - f(\bx_{t,c_t};\bomega_{t-1})\|_2. \label{eq:ugd-2+3}
\end{align}

Now, we aim to further bound the second term in Eq.~(\ref{eq:ugd-2+3}). We first bound $\|\bb_{t-1}\|_2$. To apply Lemma~\ref{lemma:bounded-vector-inverse-A-bound}, we bound the $|\delta_s|$ as below:
\begin{align}
    |\delta_s|
    & = |\sigma(\btheta^{\star\top}\bz_s + b_s) - \sigma(\btheta^{\star\top}\bz_s)| \nonumber \\
    & \le L_\sigma|b_s| \nonumber \\
    & = L_\sigma|\btheta_0^\top \bg(\bx_{s,c_s};\bomega_0)(\bomega^\star-\bomega_0)| \nonumber \\
    & \le L_\sigma \|\btheta_0\|_2 \cdot \|\bg(\bx_{s,c_s};\bomega_0)\|_F \cdot \|\bomega^\star-\bomega_0\|_2 \nonumber \\
    & \le L_\sigma \cdot 2(2+\sqrt{d^{-1}\log(1/\delta)}) \cdot  C_1\sqrt{mLd} \cdot \sqrt{(\bh-\tilde{\bh})^\top\bH^{-1}(\bh-\tilde{\bh})/m} \nonumber \\
    & \le C_5 L_\sigma \sqrt{\log(1/\delta)L(\bh-\tilde{\bh})^\top\bH^{-1}(\bh-\tilde{\bh})},
\end{align}
where $C_5>0$ is an absolute constant, the first inequality is due to the definition of Lipschitz constant $L_\sigma$, the second inequality is due to triangle inequality, the third inequality is due to the initialisation scheme of $\btheta_0$, Lemma~\ref{lemma:linearlisation} and Lemma~\ref{lemma:neural-gradient-bound}. Then, we can derive that
\begin{align}
    \|\bb_{t-1}\|_2 
    & = \|(\lambda\bI+\widehat{\bA}_{t-1})^{-1} \big(\sum_{s=1}^{t-1} \delta_s\bz_s\big)\|_2 \nonumber \\
    & \le \frac{1}{\kappa_\sigma}\|\bA_{t-1}^{-1}\big(\sum_{s=1}^{t-1} \delta_s\bz_s\big)\|_2 \nonumber \\
    & \le \frac{2C_5L_\sigma d}{\kappa_\sigma}\sqrt{\log(1/\delta)L(\bh-\tilde{\bh})^\top\bH^{-1}(\bh-\tilde{\bh})}, \label{eq:ugd-2+3-2-1}
\end{align}
where the first inequality is due to the fact that $\lambda\bI+\widehat{\bA}_{t-1} \succeq \kappa_\sigma \bA_{t-1}$, the second inequality is due to Lemma~\ref{lemma:bounded-vector-inverse-A-bound}.

Next, we bound $\|f(\bx_{t,c_t^\star};\bomega_{t-1}) - f(\bx_{t,c_t};\bomega_{t-1})\|_2$ via Lemma~\ref{lemma:neural-linearilisation-bound}:
\begin{align}
    & \|f(\bx_{t,c_t^\star};\bomega_{t-1}) - f(\bx_{t,c_t};\bomega_{t-1})\|_2 \nonumber \\
    & = \Big\|f(\bx_{t,c_t^\star};\bomega_{t-1})-\bg(\bx_{t,c_t^\star};\bomega_0)(\bomega_{t-1}-\bomega_0) -f(\bx_{t,c_t};\bomega_{t-1})+\bg(\bx_{t,c_t};\bomega_0)(\bomega_{t-1}-\bomega_0) \nonumber \\
    & \quad + (\bg(\bx_{t,c_t};\bomega_0)-\bg(\bx_{t,c_t^\star};\bomega_0))(\bomega_{t-1}-\bomega_0)\Big\|_2 \nonumber \\
    & \le \|f(\bx_{t,c_t^\star};\bomega_{t-1})-\bg(\bx_{t,c_t^\star};\bomega_0)(\bomega_{t-1}-\bomega_0)\|_2 + \|f(\bx_{t,c_t};\bomega_{t-1})-\bg(\bx_{t,c_t};\bomega_0)(\bomega_{t-1}-\bomega_0)\|_2 \nonumber \\
    & \quad + \|\bg(\bx_{t,c_t};\bomega_0)-\bg(\bx_{t,c_t^\star};\bomega_0)\|_F \cdot \|\bomega_{t-1}-\bomega_0\|_2 \nonumber \\
    & \le 2C_3m^{-1/6}\sqrt{d\log m}\Big(\frac{t}{\lambda}\Big)^{2/3}L^3 + 2\ell_{\text{Lip}}\sqrt{t/(m\lambda)}, \label{eq:ugd-2+3-2-2}
\end{align}
where the first inequality is due to triangle inequality, the second inequality is due to Lemma~\ref{lemma:neural-linearilisation-bound}, Assumption~\ref{assumption:gradient-lipschitz}, and Lemma~\ref{lemma:neural-parameter-bound}.

Plugging Eq.~(\ref{eq:ugd-2+3-2-1}) and Eq.~(\ref{eq:ugd-2+3-2-2}) back into Eq.~(\ref{eq:ugd-2+3}) yields
\begin{align}
    & \btheta^{\top}_{t-1}[f(\bx_{t,c_t^\star};\bomega_{t-1}) - f(\bx_{t,c_t};\bomega_{t-1})] - (\btheta_{t-1}-\btheta^\star)^\top[f(\bx_{t,c_t^\star};\bomega_{t-1}) - f(\bx_{t,c_t};\bomega_{t-1})] \nonumber \\
    & \le 2\alpha_t\|f(\bx_{t,c_t};\bomega_{t-1})\|_{\bA_{t-1}^{-1}} + \frac{2C_5L_\sigma d}{\kappa_\sigma}\sqrt{\log(1/\delta)L(\bh-\tilde{\bh})^\top\bH^{-1}(\bh-\tilde{\bh})} \nonumber \\
    & \quad \cdot\Big(2C_3m^{-1/6}\sqrt{d\log m}\Big(\frac{t}{\lambda}\Big)^{2/3}L^3 + 2\ell_{\text{Lip}}\sqrt{t/(m\lambda)}\Big).
\end{align}

Now, we can derive the upper bound of the instantaneous latent reward gap as:
\begin{align}
    & h(\bx_{t,c_t^\star}) - h(\bx_{t,c_t}) \nonumber \\
    & \le 2\alpha_t\|f(\bx_{t,c_t};\bomega_{t-1})\|_{\bA_{t-1}^{-1}} + C_4\ell_{\text{Lip}}m^{-1/2}\sqrt{d^{-1}\log(1/\delta)(\bh-\tilde{\bh})^\top\bH^{-1}(\bh-\tilde{\bh})} \nonumber \\
    & \quad + \frac{2C_5L_\sigma d}{\kappa_\sigma}\sqrt{\log(1/\delta)L(\bh-\tilde{\bh})^\top\bH^{-1}(\bh-\tilde{\bh})} \nonumber \\
    & \quad \cdot\Big(2C_3m^{-1/6}\sqrt{d\log m}\Big(\frac{t}{\lambda}\Big)^{2/3}L^3 + 2\ell_{\text{Lip}}\sqrt{t/(m\lambda)}\Big).
\end{align}

Therefore, the regret of Neural-LinLogUCB can be bounded as:
\begin{align}
    R_T 
    & = \sum_{t=1}^T \big( \sigma(h(\bx_{t,c_t^\star})) - \sigma(h(\bx_{t,c_t})) \big) \nonumber \\
    & \le \sum_{t=1}^T L_\sigma|h(\bx_{t,c_t^\star}) - h(\bx_{t,c_t})| \nonumber \\
    & \le \sqrt{T\max_{t\in[T]}\alpha_t^2\sum_{t=1}^T\|f(\bx_{t,c_t};\bomega_{t-1})\|_{\bA_{t-1}^{-1}}^2} + C_4\ell_{\text{Lip}}m^{-1/2}T\sqrt{d^{-1}\log(1/\delta)} \|\bh-\tilde{\bh}\|_{\bH^{-1}} \nonumber \\
    & \quad + \frac{2C_5L_\sigma d}{\kappa_\sigma}\sqrt{\log(1/\delta)L} \|\bh-\tilde{\bh}\|_{\bH^{-1}} \nonumber \\
    & \quad \cdot \Big(2C_3m^{-1/6}T^{5/3}\sqrt{d\log m}\lambda^{-2/3}L^3 + 2\ell_{\text{Lip}}m^{-1/2}T^{3/2}\lambda^{-1/2}\Big) \nonumber \\
    & \le \frac{C_6}{\kappa_\sigma}\sqrt{Td\log(1+LT^2/\lambda^2)}  \Big(\nu\sqrt{d\log(1+LT^2/\lambda) +\log(1/\delta)} + \sqrt{\lambda}M\Big) \nonumber \\
    & \quad + \frac{C_7d}{\kappa_\sigma}\sqrt{\log(1/\delta)L} \|\bh-\tilde{\bh}\|_{\bH^{-1}} \nonumber \\
    & \quad \cdot \Big(m^{-1/6}T^{5/3}\sqrt{d\log m}\lambda^{-2/3}L^3 + \ell_{\text{Lip}}m^{-1/2}T^{3/2}\lambda^{-1/2}\Big), \nonumber \\
\end{align}
where the first inequality is due to the definition of Lipschitz constant, the second inequality is due to Cauchy's inequality, the third inequality comes from the fact that $\max_{t\in[T]}\alpha_t=\alpha_T$, Lemma~\ref{lemma:representation-bound} and Lemma~\ref{lemma:feature-matrix-bound}, $C_6,C_7>0$ are absolute constants independent of other parameters. 
\end{proof}

\clearpage
\begin{hyphenrules}{nohyphenation}
\setlength{\bibsep}{.5ex plus .8ex}
\addcontentsline{toc}{section}{References for Supplementary Notes}
\putbib
\end{hyphenrules}
\end{bibunit}

\end{document}